\documentclass[twoside]{article}

\usepackage[accepted]{aistats2024}
\usepackage[round]{natbib}
\renewcommand{\bibname}{References}

\usepackage{subfigure}
\usepackage{algorithm}
\usepackage{algorithmic}
\usepackage{amsthm}
\usepackage{amssymb}
\usepackage{amsfonts}
\usepackage{hyperref}
\usepackage{graphicx}
\usepackage{xcolor}
\usepackage{soul}
\usepackage[shortlabels]{enumitem}

\floatstyle{ruled}
\newfloat{algorithm}{tbp}{loa}
\providecommand{\algorithmname}{Algorithm}
\floatname{algorithm}{\protect\algorithmname}

\theoremstyle{plain}
\newtheorem{defn}{Definition}[section]
\theoremstyle{plain}
\newtheorem{thm}{Theorem}[section]
\theoremstyle{remark}
\newtheorem{rem}[thm]{Remark}
\theoremstyle{plain}
\newtheorem{lem}[thm]{Lemma}
\theoremstyle{plain}
\newtheorem{prop}[thm]{Proposition}

\begin{document}



\global\long\def\Real{\mathbb{R}}%
\global\long\def\Natural{\mathbb{N}}%
\global\long\def\Expectation{\mathbb{E}}%
\global\long\def\Probability{\mathbb{P}}%
\global\long\def\Var{\mathbb{V}}%
\global\long\def\CE#1#2{\Expectation\left[\left.#1\right|#2\right]}%
\global\long\def\CP#1#2{\Probability\left(\left.#1\right|#2\right)}%
\global\long\def\abs#1{\left|#1\right|}%
\global\long\def\norm#1{\left\Vert #1\right\Vert }%
\global\long\def\Indicator#1{\mathbb{I}\left(#1\right)}%
\global\long\def\Maxeigen#1{\lambda_{\max}\left(#1\right)}%
\global\long\def\Mineigen#1{\lambda_{\min}\left(#1\right)}%
\global\long\def\Trace#1{\text{Tr}\left(#1\right)}%

\global\long\def\mW{\mathcal{W}}%
\global\long\def\mT{\mathcal{T}}%
\global\long\def\mC{\mathcal{C}}%
\global\long\def\mH{\mathcal{H}}%
\global\long\def\mF{\mathcal{F}}%
\global\long\def\mM{\mathcal{M}}%
\global\long\def\mE{\mathcal{E}}%
\global\long\def\mD{\mathcal{D}}%
\global\long\def\mB{\mathcal{B}}%
\global\long\def\mX{\mathcal{X}}%
\global\long\def\mS{\mathcal{S}}%
\global\long\def\Parameter#1{\theta_{#1}^{\star}}%
\global\long\def\mA{\mathcal{A}}%
\global\long\def\Feature#1#2{\phi(#1,#2)}%
\global\long\def\Weight#1#2{w_{#2}^{#1}}%
\global\long\def\PseudoAction#1#2{\tilde{a}{}_{#1}^{(#2)}}%
\global\long\def\Impute#1#2{\widehat{w}_{#1}^{\text{Im}(#2)}}%
\global\long\def\Barw#1#2{\bar{w}_{#1}^{(#2)}}%

\global\long\def\TransitDistribution#1#2{\mathbb{P}(\cdot|#1,#2)}%
\global\long\def\State#1#2{x_{#1}^{(#2)}}%
\global\long\def\Action#1#2{a_{#1}^{(#2)}}%
\global\long\def\Reward#1#2{r\left(#1,#2\right)}%
\global\long\def\Sign#1{\text{sign}\left(#1\right)}%

\global\long\def\Pseudoreward#1#2#3#4{\tilde{Q}_{#1,#2}\left(#3,#4\right)}%
\global\long\def\Policy#1#2#3{\pi_{#1}\left(#2,#3\right)}%
\global\long\def\Value#1#2#3{V_{#1}^{#2}(#3)}%
\global\long\def\IPW#1#2{\widehat{w}_{#1}^{IPW(#2)}}%
\global\long\def\Estimator#1#2{\widehat{w}_{#1}^{(#2)}}%
\global\long\def\Newprob#1#2#3{\tilde{\pi}_{#1,#2}^{(#3)}}%
\global\long\def\AV#1#2#3#4{Q_{#1}^{#2}(#3,#4)}%
\global\long\def\AVhat#1#2#3#4{\widehat{Q}_{#1}^{#2}(#3,#4)}%
\global\long\def\algo{\texttt{RDRLVI}}

\twocolumn[

\aistatstitle{A Doubly Robust Approach to Sparse Reinforcement Learning}

\aistatsauthor{ Wonyoung Kim \And Garud Iyengar \And  Assaf Zeevi}

\aistatsaddress{ Columbia University \And  Columbia University \And Columbia University } ]

\begin{abstract}
We propose a new regret minimization algorithm for episodic sparse linear Markov decision process (SMDP) where the state-transition distribution is a linear function of observed features. 
The only previously known algorithm for SMDP requires the knowledge of the sparsity parameter and oracle access to an unknown policy.
We overcome these limitations by combining the doubly robust method that allows one to use feature vectors of \emph{all} actions with a novel analysis technique that enables the algorithm to use data from all periods in all episodes.  
The regret of the proposed algorithm is $\tilde{O}(\sigma^{-1}_{\min}
s_{\star} H \sqrt{N})$, where $\sigma_{\min}$ denotes the restrictive
the minimum eigenvalue of the average Gram matrix of feature vectors,
$s_\star$ is the sparsity parameter, $H$ is the length of an episode, and $N$ is the number of rounds.  
We provide a lower regret bound that matches the upper bound up to logarithmic factors on a newly identified subclass of SMDPs.
Our numerical experiments support our theoretical results and demonstrate the superior performance of our algorithm. 
\end{abstract} 

%----------------------------
% 1. Introduction
%----------------------------

\section{INTRODUCTION}
The goal of reinforcement learning (RL) is to maximize the cumulative expected reward while simultaneously learning the unknown transition structure of the underlying Markov decision process~(MDP).  
RL has been applied to robotics~\citep{kober2013reinforcement}, human-level game plays~\citep{mnih2013playing,silver2016mastering}, dialogue systems~\citep{li2016deep}, among others~\citep{barto2017some}.   
Modern RL applications have exponentially large, possibly infinite state space, and therefore tabular RL  \citep{auer2008near,osband2016generalization,azar2017minimax,dann2017unifying,jin2018q,strehl2006pac} is intractable, and value function approximation is essential.  

RL with deep neural networks-based value function approximation had empirical success in a variety of settings with high dimensional state and action spaces~\citep{mnih2013playing,mnih2016asynchronous,schulman2015trust}.  
However, providing theoretical guarantees for these methods has been
challenging because, in the high dimensional setting, most states are not visited even once during a set of learning episode~\citep{sutton2018reinforcement,szepesvari2022algorithms}  
Consequently, there was an effort to establish convergence results when the true (unknown) value function is assumed to be a linear function of $d$ features  \citep{hu2022nearly,zhou2021nearly,he2021logarithmic,he2021uniform}.  
Extensions to the setting where the value function is within a prescribed distance from a linear function
\citep{cai2020provably,jin2023provably,ayoub2020model,zanette2020frequentist}, and to settings the state transition, and therefore, the value function, is a sparse linear function of the features, i.e., a linear function of $s_{\star} \ll d$ features  \citep{jiang2017contextual,sun2019model,agarwal2020flambe,hao2021online}.
The latter class of problems includes low-rank MDPs and sparse linear Markov decision processes~(SMDPs).  
The SMDP approach provides flexibility over linear function approximation since one is now allowed to consider a much larger set $d \gg 1$ of features and select only the $s_{\star} \ll d$ informative features, and we contribute to the literature of online SMDP.

Most of the prior work using sparse linear approximation is for offline RL, and the extension to online RL has remained challenging.  
\citet{hao2021online} proposed the first online algorithm for RL with sparse linear approximation with a regret bound that is logarithmic in the number of features~$d$. 
However, their algorithm requires knowledge of the number of informative features~$s_{\star}$, and needs oracle access to an exploratory policy, and, unfortunately, identifying such a policy is as hard as designing the learning algorithm.   
Therefore, the problem of designing a practical online algorithm for RL with sparse linear function approximation remains open.

Our main contributions are as follows.
\begin{enumerate}[(a)]
\item 
Online RL with bandit feedback is hard because only samples for the $Q$-value function of the selected actions are observed. 
We propose a novel algorithm that leverages a technique called doubly
robust (DR) estimation to impute values for the $Q$-values for unselected actions (Section~\ref{sec:proposed_method}).
The estimator proposed by \citet{hao2021online} requires oracle access to an exploratory policy in order to guarantee that $Q$-value estimates converge to the true $Q$-values. 
In contrast, our estimator converges to the optimal $Q$-value function without the oracle access to an exploratory policy, which is possible using features of \emph{all} actions. 
\item
We develop a new analysis technique that carefully accounts for the dependence between $Q$-value estimates for different periods $h$, which allows us to use data from all $H$ periods and all $n$ episodes to estimate the $Q$-value function of a time-homogeneous MDP (Section~\ref{sec:tail_analysis}).
In contrast, previous methods partition the episodes into $H$ groups and use the $h$-th group to estimate the $Q$-value function for period $h$.  
Thus, our estimation method increases the number of effective samples from $n/H$ to $n$. 

\item 
We leverage our estimation to propose new algorithm~\algo\ for homogeneous SMDPs whose regret is $\tilde{O}(\sigma_{\min}^{-1}s_{\star}H\sqrt{N})$ 
(Theorem~\ref{thm:upper_bound}), where $\sigma_{\min}$ is the restrictive minimum eigenvalue defined in Definition~\ref{def:RME} and $s_{\star}$ is the sparsity parameter defined in Definition~\ref{defn:SMDP}. 
\algo\ does not require knowledge of the number of informative features $s_{\star}$ and does not need oracle access to an exploratory policy, and yet the regret bound that is logarithm in the number of features~$d$.

\item 
We provide a novel lower bound on the regret for SMDPs
(Theorem~\ref{thm:lower_bound}). 
We show that the lower bound critically depends on $\sigma_{\min}$.  
For SMDP instances with $\sigma_{\min}^2 \geq s^{\star}/d$ we show that the regret of our proposed algorithm is tight to within logarithmic factors; whereas when $\sigma_{\min}^2 \leq s^{\star}/d$ there is a gap that needs to addressed.  
This result is an improvement and an extension of the lower bound results for sparse linear bandits. 

\item 
The results of our numerical experiments demonstrate the superior
performance of the proposed algorithm over the previously known
algorithms.  
The results empirically verify the dependence of regret on
$\sigma_{\min}$, and that the regret is almost independent of the 
dimension of the feature vector~$d$.
\end{enumerate}

%------------------------
% 2. Related Works
%------------------------
\section{RELATED WORK}
Function approximation MDP is introduced by \citet{sutton1988learning,tsitsiklis1996analysis} and
\citet{bradtke1996linear}.
For inhomogeneous episodic MDP, \citet{hu2022nearly} and
\citet{zhou2021nearly} proved an $\tilde{O}(dH^{3/2}\sqrt{N})$ regret bound with a nearly matching lower bound when the optimal value function is assumed to be a linear function of the features, and \citet{he2021logarithmic} established a logarithmic regret bound when there is a positive sub-optimality gap.  
\citet{jin2023provably,ayoub2020model} and \citet{zanette2020frequentist} established a regret bound when the true value function is within a prescribed distance from a linear function. 
For offline RL, \citet{jiang2017contextual} and \citet{sun2019model}
considered a larger class of MDPs that have, respectively, low Bellman rank and witness rank.  
\citet{agarwal2020flambe} introduced the low-rank MDP setting where the algorithm chooses a low-dimensional feature from a certain function class.
\citet{kolter2009regularization,geist2011l1} and \citet{painter2012greedy} studied the feature selection in offline RL using $\ell_1$ regularization. 
Finite sample guarantees for offline RL were established by
\citet{ghavamzadeh2011finite,geist2012dantzig} and \citet{hao2021sparse}.

Online SMDP reduces to contextual linear bandits with sparse parameters when the episode length $H=1$. 
\citet{abbasi2012online} proposed an algorithm that achieves an $\tilde{O}(\sqrt{s_{\star}dN})$ regret bound and matches a lower bound established in \citet{lattimore2020bandit}.
\citet{hao2020high} and \citet{jang2022popart} proposed an algorithm with a regret upper bound that does not have $\sqrt{d}$ and depends only on the minimum eigenvalue of the Gram matrix of contexts.
\citep{oh2021sparsity} and \citep{kim2019doubly,bastani2020online} use results in high dimensional statistics~\citep{buhlmann2011statistics, van2009conditions} to establish regret bounds that depend on the restrictive minimum eigenvalue and the compatibility condition, respectively.
Even though the fact that (restrictive) minimum eigenvalue is the critical parameter determining the upper and lower bound of regret for linear bandits is known, extending these results to the SMDP is nontrivial and remains open.

Applying the DR method \citep{bang2005doubly,fleiss2013statistical}  to bandit literature was introduced by~\citet{kim2019doubly} and~\citet{dimakopoulou2019balanced}.  
A line of linear bandit literature applied the DR method to design an algorithm with improved regret bound~\citet{kim23improved,kim2023double} and near-optimal regret bound~\citet{kim2021doubly,kim2023squeeze} 
However, all preceding works are limited to (sparse) linear bandits, and extending these results to online sparse linear RL is nontrivial.

%------------------------
% 3. SMDP
%------------------------

\section{SPARSE LINEAR MARKOV DECISION PROCESS}

In this section, we present the problem formulation of the SMDP and a lower bound that depends only on the restrictive minimum eigenvalue for the regret. 

\subsection{Problem Formulation}
Let MDP$(\mX,\mathcal{A}, H,\Probability,r)$ denote an 
episodic homogeneous MDP where $\mathcal{\mX}$ and $\mathcal{A}$ are the sets of possible states and actions, $H\in\Natural$ is the length of each episode, $\Probability$ is the state transition probability measure, and $r$ is the reward function.  
We allow the cardinality of the state space $\mX$ to be 
infinite but require $\mathcal{A}$ to have finite cardinality $|\mA|$.  
For $x\in\mX$ and $a\in\mA$, the probability measure $\TransitDistribution{x}{a}$ denotes over the state in the next time if action $a$ is taken in state $x$, and $r:\mX\times\mathcal{A}\to[0,1]$ is the deterministic reward function. 
For simplicity of exposition, we assume that the reward function is known; all our results hold when the reward is unknown. 

An agent interacts with this episodic MDP as follows. 
At the beginning of each episode, an initial state $x_{1}$ is sampled from the unknown distribution $\Probability_0$. 
Then, in each period $h\in[H]$, the agent observes the state $x_{h}\in\mX$, picks an action $a_{h}\in\mathcal{A}$, and receives a reward $\Reward{x_{h}}{a_{h}}$.   
The MDP evolves into a new state $x_{h+1}$ drawn from the probability measure $\Probability\left(\cdot|x,a\right)$. 
The episode terminates after $H$ interactions, i.e. when $x_{H+1}$ is observed. 
Note that the agent does take an action at $x_{H+1}$ and hence receives no reward.

We focus on the sparse linear Markov decision process (SMDP) defined as follows.
\begin{defn}[Sparse linear MDP, \citet{hao2021online}]
\label{defn:SMDP}
The MDP$(\mathcal{\mX},\mathcal{A},H,\Probability,r)$ is $s_{\star}$-sparse to a (known) feature map $\phi:\mathcal{\mX}\times\mathcal{A}\to[-1,1]^{d}$ if there exists an (unknown) function $\psi = (\psi_1(x), \ldots, \psi_{d}(x)):\mX\to\Real^{d}$ and an (unknown) set $\mathcal{I}_{\star}\subseteq[d]$ with $|\mathcal{I}_{\star}|:=s_{\star}\ll d$ such that $\psi_{i}(x)=0$ for all $x\in\mX$ and $i \not \in \mathcal{I}_{\star}$, and
\begin{align*}
\Probability\left(X_{h+1}=x|X_{h}=x^{\prime},a_{h}=a^{\prime}\right)
&=\Feature{x^{\prime}}{a^{\prime}}^{\top}\psi(x), \\  
&=\sum_{i\in \mathcal{I}_{\star}} \phi_i({x^{\prime}},{a^{\prime}}) \psi_i(x)
\end{align*}
for all $h\in[H]$ and $(x^{\prime},a^{\prime})\in\mX\times\mA$. 
We denote a sparse MDP by $\text{SMDP}\left(\mX,\mA, H,\phi,\mathbf{\psi},r\right)$. 
\end{defn}

A policy $\pi:=(\pi_{1},\ldots,\pi_{H})$ where $\pi_h: \mX 
\to\Delta_{\mA}$, $h \in [H]$, is a function from the state $\mX$
to the set $\Delta_{\mA}$ of probability distributions over $\mA$.  
Let 
\[
\Value h{\pi}x:=\Expectation^{\pi}\left[\left.\sum_{h\prime=h}^{H}\Reward{x_{h^{\prime}}}{a_{h^{\prime}}}\right|x_{h}=x\right],\quad\forall x\in\mX,
\]
denote the expected reward of policy $\pi$ over periods $h, \ldots, H$ when the state in period $h\in[H]$ is $x$.
For $(x,a)\in\mX\times\mA$, define the $Q$-value function
\[
\AV h{\pi}xa:=\Reward xa+
\mathbb{E}_{x^{\prime}\sim\CP{\cdot}{x_{h}=x,a_{h}=a}} 
\left[V_{h+1}^{\pi}(x^{\prime})\right],
\]
which is the expected value of cumulative rewards over $[h,H]$ when the agent takes action $a \in \mA$ in period~$h$, and follows policy $\pi$ thereafter.  
Since $\abs{\mA}$ and $H$ are both finite, there always exists an optimal policy $\pi^{\star}$ that achieves the optimal value $\Value h{\star}x=\sup_{\pi}\Value h{\pi}x$ for all $x\in\mathcal{S}$ and $h\in[H]$ (see e.g. \cite{puterman2014markov}).
Let $\widehat{A}$ denote an algorithm that takes as input $(\mX, \mA, H, \phi, r)$ ($\psi$ is \emph{not} known to $\widehat{A}$) and a sequence of episodes and computes a sequence of policies $\widehat{\pi}^{(1)},\ldots,\widehat{\pi}^{(N)}$.
The total regret $\text{R}(N,\widehat{A})$ of $\widehat{A}$ over $N$ episodes
\[
\text{R}(N,\widehat{A}):=\sum_{n=1}^{N}\left[\Value 1{\star}{\State 1n}-\Value 1{\widehat{\pi}{}^{(n)}}{\State 1n}\right],
\]
where $\Value 1{\star}{\State 1n}-\Value1{\widehat{\pi}^{(n)}}{\State 1n}$ denotes the regret over episode $n \in [N]$.  

For $f:\mX \to \Real$, let $[\Probability f
](x,a):=\Expectation_{x^{\prime}\sim\TransitDistribution
  xa}f(x^{\prime})$.
Then the $Q$-value function $\AV h{\pi}xa$ and the value function $\Value h{\pi}x$ of the policy $\pi$ is given by the Bellman equation: For all $(x,a)\in\mathcal{S}\times\mathcal{A}$,
\begin{equation}
\begin{split}
\AV h{\pi}xa&=\Reward xa+[\Probability V_{h+1}^{\pi}](x,a),
\\
\Value h{\pi}x&=\Expectation_{a\sim\pi_{h}(x)}\left[\AV h{\pi}xa\right], \quad
\Value{H+1}{\pi}x=0.
\end{split}
\label{eq:Bellman_eq}
\end{equation}
The optimal $Q$-value function $\AV h{\star}xa$ and the optimal value $\Value h{\star}x$ is given by the Bellman equations:
\begin{equation}
\begin{split}
\AV h{\star}xa&=\Reward xa+\left[\Probability V_{h+1}^{\star}\right](x,a),\\
\Value h{\star}x&=\max_{a\in\mathcal{A}}\AV h{\star}xa,
\quad
\Value{H+1}{\star}x=0.
\end{split}
\label{eq:Bellman_opt}
\end{equation}
The Bellman equation~\eqref{eq:Bellman_opt} implies that the optimal policy is the greedy policy with respect to the optimal $Q$-value function $\{Q_{h}^{\star}\}_{h\in[H]}$.
Thus, to identify the optimal policy $\pi^{\star}$, it suffices to
estimate the optimal $Q$-value functions. 

We will extensively use a result that, for SMDPs, the $Q$-value function is linear in the feature map $\phi$.
\begin{prop}[Sparse linearity of the expected value function]
\label{prop:linearity}
For an SMDP$(\mathcal{\mX},\mathcal{A},H,\phi,\psi,r)$ and for any policy $\pi$, there exists a set of vectors $\{\Weight{\pi}h\in\Real^{d}: h\in[H]\}$ 
such that
\[
[\Probability V_{h}^{\pi}](x,a)=\Feature xa^{\top}w_{h}^{\pi}.
\]
for all $(x,a)\in\mX\times\mA$, and the $i$-th entry of  $w_{h}^{\pi} = 0$ for all $h \in [H]$ and $i\notin\mathcal{I}_{\star}$. 
\end{prop}

\subsection{A Regret Lower Bound}
Proposition~\ref{prop:linearity} implies that the learning task for SMDPs reduces to estimating weights $\{w_{h}^{\pi}\}_{h\in[H]}$.
The following restricted minimum eigenvalue is critical in this estimation task.
\begin{defn}[Restrictive minimum eigenvalue]
\label{def:RME}
The restricted minimum eigenvalue (RME) for a positive semi-definite matrix $M\in\Real^{d\times d}$ as a function of a sparsity parameter $s\in[d]$ is defined as follows:
\[
\sigma_{\min}^{2}\!(M,s):=\!\min_{\mathcal{I}\subset[d],|\mathcal{I}|\le s}\Big\{ \frac{\beta^{\top}M\beta}{\norm{\beta_{\mS}}_{2}^{2}}:\norm{\beta_{\mathcal{I}^{c}}}_{1}\!\le\!3\norm{\beta_{\mathcal{I}}}_{1}\!\!\Big\}.
\]
\end{defn}
Let
$\Sigma(\pi):=\Expectation^{\pi}\big[\frac{1}{H}\sum_{h=1}^{H}\Feature{\State
  hn}{\Action hn}\Feature{\State hn}{\Action hn}^{\top}\big]$ denote the
expected Gram matrix over the states and actions from a policy $\pi$.  
Let $\pi^{U}$ denote the policy that chooses actions uniformly over the
set $\mA$ for all $h\in[H]$ and $x\in\mX$, and let $\Sigma^U :=\Sigma(\pi^U)$ denote the expected Gram matrix of the uniform policy
$\pi^U$. 
Let $\mS_{\text{E}}$ (resp. $\mS_{\text{H}}$) denote a collection of SMDP
instances of which satisfy $\sigma_{\min}(\Sigma^{U},s_{\star})\ge 
s_{\star}/d$ (resp. $\sigma_{\min}(\Sigma^{U},s_{\star}) <
s_{\star}/d$). 
We show that the lower bound on the regret of any algorithm on an SMDP
instance depends on whether the instance is in  $\mS_{\text{E}}$ or
$\mS_{\text{H}}$.

%----------------------------------------
% Figure 1. Regret bound
%----------------------------------------
\begin{figure}[t]
\centering
\subfigure[Upper and lower regret bound when $ds_{\star} > 25N$.]{{\label{fig:regret_high_d}\includegraphics[width=0.45\textwidth]{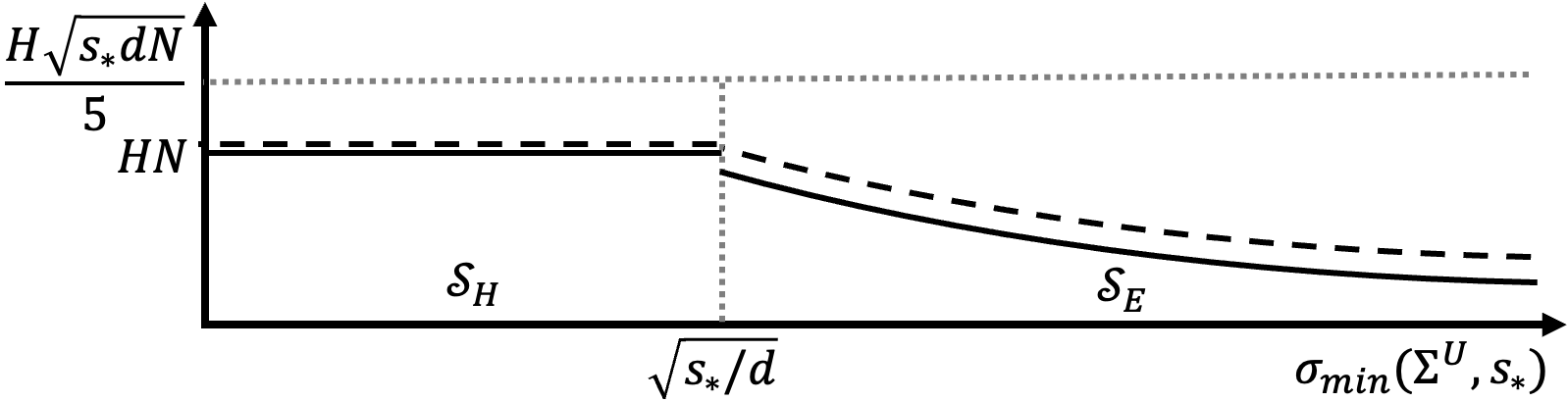}
}}
\subfigure[Upper and lower regret bound when $ds_{\star} \le 25N$.]{{\label{fig:regret_low_d}\includegraphics[width=0.45\textwidth]{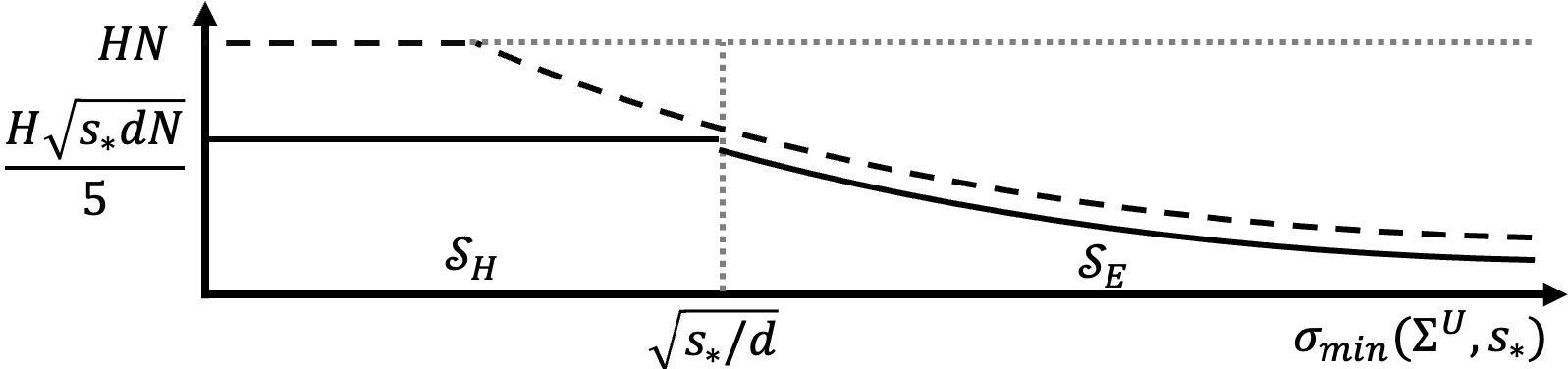}
}}
\caption{\label{fig:regret} Illustration of the regret lower bound (solid) and upper bound (dashed) proved in this paper (ignoring logarithmic terms). 
Our upper bound matches the lower bound for SMDPs in $\mS_{\text{E}}$ or when $ds_{\star} > 25N$.
}
\vspace{.1in}
\end{figure}

\begin{thm}[SMDP Regret Lower Bound.]
\label{thm:lower_bound}
Suppose $N \ge s_{\star} \ge 5$ and $d\ge s_{\star}^{2}$. 
Then for any algorithm $\widehat{A}$,
\begin{equation}
\sup_{\text{SMDP} \in \mS_{\text{E}}} \!\!\!\Expectation[\mathrm{R}(N,\!\widehat{A})]
  \!\ge\!\min\!\bigg\{\!\frac{Hs_{\star}\sqrt{N}}{20\sigma_{\min}(\Sigma^{U}\!,s_{\star})},
     \:HN\!\bigg\},
\label{eq:lower_bound}
\end{equation}
and
\begin{equation}
  \sup_{\text{SMDP} \in \mS_{\text{H}}}\!\!\!
  \Expectation[\mathrm{R}(N,\!\widehat{A})]  \ge \min
\bigg\{\!\frac{H\sqrt{s_{\star}dN}}{5},\;HN\!\bigg\}.
     \label{eq:lower_bound_2}
\end{equation}
\end{thm}
The $HN$ term in the bounds is from the rewards bounded by $1$. 
For instances in $\mS_{\text{E}}$, i.e. when the RME
$\sigma_{\min}(\Sigma^U,s^\ast) \geq \sqrt{s_{\star}/d}$ the lower
bound~\eqref{eq:lower_bound} does not depend on dimension~$d$, which is because the RME provides enough variability on the $s_{\star}$ non-zero entries of the feature vectors to estimate the non-zero entries of $\{w_{h}^{\star}\}_{h\in[H]}$.  
As the RME increases, the lower bound decreases because it is easier to identify $\mathcal{I}_{\star}$.  
However, when RME $\sigma_{\min}(\Sigma^U,s_{\star}) \leq
\sqrt{s_{\star}/d}$, the features do not have sufficient variability on $\mathcal{I}_{\star}$ and any algorithm must estimate all $d$ entries of the weight $\{w_{h}^{\star}\}_{h\in[H]}$, resulting in $d$ appearing in the regret bound~\eqref{eq:lower_bound_2}. 

We compare the regret upper and lower bounds (ignoring the trivial $HN$ term and logarithmic terms) in Figure~\ref{fig:regret}. 
The regret of our proposed algorithm is tight within a logarithmic term
on $\mS_{\text{E}}$ or when $ds_{\star}/25 > N > \Omega(\log^{7}d)$, i.e. $d$ is large compared to $N$; however, there is a gap for SMDPs in $\mS_{\text{H}}$ when $d \le 25N/s_{\star}$.

The bound~\eqref{eq:lower_bound_2} generalizes  $\Omega(\sqrt{s_{\star}dN})$ lower bound for sparse linear bandits established by~\citet{lattimore2020bandit} to SMDPs. 
When $d \ge N$, \citet{hao2020high} proved $\Omega(\lambda_{\min}^{-1/3}s_{\star}^{1/3}N^{2/3})$ regret bound and \citet{jang2022popart} improved the result to $\Omega(\lambda_{\min}^{-1/3}s_{\star}^{2/3}N^{2/3})$, where
$\lambda_{\min}$ is the (unrestricted) minimum eigenvalue. 
\citet{hao2021online} proved an $\Omega(dH)$ regret bound when $d\ge N$ for SMDPs.
However, the impact of the restrictive minimum eigenvalue on the regret lower bound for general $d$ and $N$ has not been discovered.
Theorem~\ref{thm:lower_bound} establishes a lower bound for SMDPs with any $d$ and $N$, and sparse linear bandits as a special case.
Furthermore, we identify $\mS_{\text{E}}$ and a novel lower bound where the regret depends on the RME instead of $d$. 
When $\lambda_{\min}=\sigma_{\min}(\Sigma^{U},s_{\star})$, the lower bound~\eqref{eq:lower_bound} is an improvement over the best known lower bound $\Omega(\lambda_{\min}^{-1/3}s_{\star}^{2/3}N^{2/3})$ for sparse linear bandits. 
This is because $N^{2/3} \le  d^{\frac{1}{6}} \sqrt{N} \le \lambda_{\min}^{1/3}s_{\star}^{1/3} \sqrt{N}$ is implied by the assumption that $d\ge N$ and $\sigma_{\min}(\Sigma^{U},s_{\star})\ge\sqrt{s_{\star}/d}$.
%The proof of the lower bound is in Appendix \ref{subsec:lower_bound_proof}. 

%------------------------------
% 4. Proposed Method
%------------------------------
\section{PROPOSED METHOD}
\label{sec:proposed_method}

We propose a novel estimator and an algorithm that uses features from all actions, periods, and episodes.

\subsection{Randomized Doubly Robust \textit{Q}-Value Function}
\label{sec:estimator}
The Bellman equation \eqref{eq:Bellman_opt} and
Proposition~\ref{prop:linearity} implies,
\begin{align*}
\Value h{\star}{x^{\prime}} = & \max_{a^{\prime}\in\mA}\left\{ \Reward{x^{\prime}}{a^{\prime}}+\left[\Probability V_{h+1}^{\star}\right](x^{\prime},a^{\prime})\right\} \\
= & \max_{a^{\prime}\in\mA}\left\{ \Reward{x^{\prime}}{a^{\prime}}+\Feature{x^{\prime}}{a^{\prime}}^{\top}\Weight{\star}{h+1}\right\}.
\end{align*}
Thus, it follows that 
\begin{align*}
&[\Probability V_{h}^{\star}](x,a)=\Feature
                 xa^{\top}\Weight{\star}h\\
  &=\Expectation_{x^{\prime}\sim\Probability(\cdot|x,a)}\Big[\max_{a^{\prime}\in\mA}\{ 
  \Reward{x^{\prime}}{a^{\prime}}+\Feature{x^{\prime}}{a^{\prime}}^{\top}
    \Weight{\star}{h+1}\}  
  \Big].
\end{align*}
Then,
\[
\max_{a^{\prime}\in\mA}\widehat{Q}_{\Weight{\star}{h+1}}(x^{\prime},a^{\prime}):=\max_{a^{\prime}\in\mA}\{
\Reward{x^{\prime}}{a^{\prime}}\!+\!\Feature{x^{\prime}}{a^{\prime}}^{\top}\Weight{\star}{h+1}\}  
\]
is an unbiased estimator for $\Feature xa^{\top}\Weight{\star}h$ for any $(x,a)\in\mX\times\mA$ when $x^{\prime}$ is generated from the distribution $\Probability(\cdot|x,a)$. 
Since $\Weight{\star}{h+1}$ is unknown, we need an estimate
$w_{h+1}^{(n)}$ for $w_{h+1}^{\star}$ using data from $n$ episodes and $H$ periods.

For period $k\in[H]$ and action $a\in\mA$, let $X_{k+1}^{(\tau)}(a)$ denote a random sample of the state according to $\Probability(\cdot|\State k{\tau},a)$. 
Note that $X_{k+1}^{(\tau)}(\Action{k}{\tau}) = \State{k+1}{\tau}$.
Let $\Pi_{[0,H]}(x):=\min\{\max\{x,0\},H\}$ is a projection function onto $[0,H]$.
Let
\begin{equation}
\widehat{Y}_{w_{h+1}^{(n)}}\!(\State k{\tau}\!\!,a)\!:=\!\Pi_{[0,H]}\!\Big(\max_{a^{\prime}\in\mA}\widehat{Q}_{w_{h+1}^{(n)}}\!\!(X_{k+1}^{(\tau)}(a),a^{\prime})\Big)
\label{eq:Y_hat} 
\end{equation}
denote an estimate for the $Q$-value function on $(\State k{\tau}\!,a)$ for $a\in\mA$.
For $\tau\in[n]$ and $k\in[H]$, we only observe $\State{k+1}{\tau}\!=\!X_{k+1}^{(\tau)}(\Action{k}{\tau})$, the estimate \eqref{eq:Y_hat} is observable only when $a=\Action {k}{\tau}$. 
Therefore, the conventional least square value iteration for (inhomogeneous) MDP estimates $\Weight{\star}h$ by minimizing the loss function,
\begin{equation}
\sum_{\tau=1}^{n}\Big\{ \widehat{Y}_{w_{h+1}^{(n)}}(\State{h}{\tau},\Action{h}{\tau})-w_{h}^{\top}\Feature{\State{h}{\tau}}{\Action{h}{\tau}}\Big\}^{2}. 
\label{eq:conventional_loss}
\end{equation}

\cite{hao2021online} equally divide $n$ episodes into $H$ partitions
$\{\mD_{h}\}_{h\in[H]}$ and estimate $\Weight{\star}h$ using the episodes in $\mD_h$, i.e., by minimizing the loss function, 
\begin{equation}
\sum_{\tau\in\mD_{h}}\sum_{k=1}^{H}\Big\{\widehat{Y}_{w_{h+1}^{(n)}}(\State{k}{\tau},
\Action{k}{\tau})-w_{h}^{\top}\Feature{\State
  k{\tau}}{\Action k{\tau}}\Big\}^{2}.
\label{eq:hao_loss}
\end{equation}
The loss function~\eqref{eq:hao_loss} sums up over $k\in[H]$, enabling the estimation procedure to use a Gram matrix that sums up over all periods $k\in[H]$ in each episode. 
However, in order to ensure that the estimate for $w_h^{(n)}$ is independent of $w_k^{(n)}$ for $k > h$, the estimator of $w_h^{(n)}$ can only use episodes in $D_h$, since $w_{h+1}^{(n)}$ in~\eqref{eq:hao_loss} is estimated with $(\State{k}{\tau},\Action{k}{\tau})_{k\in[H],\tau\in\mD_{h+1}}$.
In order to use all $n$ episodes in each estimation, the correlation between $w_{h+1}^{(n)}$ and $(\State{k}{\tau},\Action{k}{\tau})_{k\in[H],\tau\in\mD_{h+1}}$ need to be analyzed carefully. 

While the loss functions~\eqref{eq:conventional_loss} and~\eqref{eq:hao_loss} used in previous work only utilize selected actions $\Action{k}{\tau}$, we consider the estimated $Q$-value function $\widehat{Y}_{w_{h+1}^{(n)}}(\State{k}{\tau},a)$ for unselected actions $a\neq\Action{k}{\tau}$ as missing data and apply the DR method to develop a novel estimator that uses \emph{all} actions. 
Let $\PseudoAction k{\tau}$ denote a random variable sampled from the Uniform distribution on $\mA$, i.e., $\Probability(\PseudoAction k{\tau}=a)=|\mA|^{-1}$, independent of all other random variables.  
We define the pseudo-reward analogous to the DR method as follows:
\begin{equation}
\begin{split}
\tilde{Y}_{w,k}^{(\tau)}(a):=\frac{\mathbb{I}(\PseudoAction k{\tau}=a)}{|\mA|^{-1}}\widehat{Y}_{w}(\State k{\tau},a)\\
&\hspace*{-1.45in} \mbox{} +\bigg\{ 1-\frac{\mathbb{I}(\PseudoAction k{\tau}=a)}{|\mA|^{-1}}\bigg\}
\Feature{\State k{\tau}}a^{\top}\Impute hn, 
\label{eq:p_reward}
\end{split}
\end{equation}
where the imputation estimator
\begin{equation}
\begin{split}
\Impute hn =
\arg\min_{w_{h}}\bigg\{\lambda_{\text{Im}}^{(n)}\norm{w_{h}}_{1} + \\
&\hspace*{-1.85in} \sum_{\tau=1}^{n}\sum_{k=1}^{H}\!\Big(\widehat{Y}_{\Estimator{h+1}{n}}\!(\State
k{\tau},\Action k{\tau})-w_{h}^{\top}\Feature{\State k{\tau}}{\Action
  k{\tau}}\Big)^{2}\bigg\}. 
\end{split}
\label{eq:impute_loss}
\end{equation}
Taking expectation over $\PseudoAction k{\tau}$ on both sides of \eqref{eq:p_reward} gives $\Expectation[\tilde{Y}_{h,k}^{(\tau)}(a)]=\widehat{Y}_{h}(\State k{\tau},a)$. 
Thus, the pseudo-reward $\tilde{Y}_{h,k}^{(\tau)}(a)$ is unbiased for all $a\in\mA$.

Still, we observe $\widehat{Y}_{h}(\State k{\tau},a)$ only when $a=\Action k{\tau}$ and we resample $\PseudoAction k{\tau}$ until $\PseudoAction k{\tau} = \Action k{\tau}$. 
The resampling further randomizes the policy and connects it to the uniform policy. 
Let $\mM_{k}^{(\tau)}$ denote the event of obtaining the matching $\PseudoAction k{\tau}=\Action k{\tau}$ with certain number of resamples.
On the event $\mM_{k}^{(\tau)}$, we use unbiased pseudo-rewards
$\{\tilde{Y}_{h,k}^{(\tau)}(a)\}_{a\in\mA}$, otherwise we do not use the data.
Let $\Estimator{H+1}n=\mathbf{0}$ and we will construct our estimator $\Estimator hn$ recursively for $h=H,\ldots,2$ by minimizing
\begin{equation}
\begin{split}
\Estimator
hn=\arg\min_{w_{h}}\bigg\{\lambda_{\text{Est}}^{(n)}\norm{w_{h}}_{1} +\\
&\hspace*{-1.9in} \sum_{\tau=1}^{n}\sum_{k=1}^{H}\mathbb{I}(\mM_{k}^{(\tau)})\sum_{a\in\mA}\Big(\tilde{Y}_{\Estimator{h+1}{n},k}^{(\tau)}(a)-w_{h}^{\top}\Feature{\State 
  k{\tau}}a\Big)^{2}\bigg\}, 
\end{split}
\label{eq:DR_loss}
\end{equation}
where $\lambda_{\text{Est}}^{(n)}>0$ is another regularization parameter.
Although $\Estimator{h+1}{n}$ is correlated with $\State{k}{\tau}$, we develop a novel analysis technique to obtain finite sample guarantees (see Section~\ref{sec:tail_analysis} for details). 
Note that with a sufficiently large number of resamples, the event $\mM_{k}^{(\tau)}$ happens with high probability.
Then our estimator utilizes data that are not used by previous works in that (i) we use unbiased pseudo-rewards and feature vectors of all arms in $\mA$ and (ii) we use all data points in $\tau\in[n]$ instead of splitting them into independent partitions as in \eqref{eq:hao_loss}. 
These two novel contributions enable us to design a practical and optimal algorithm for SMDP.

\subsection{Proposed Algorithm}
Our proposed algorithm, Randomized Doubly Robust Lasso Value Iteration (\texttt{RDRLVI}), is described in Algorithm~\ref{alg:DRLSVI}. 
The \texttt{RDRLVI} samples $\Action{h}{n}$ as $\epsilon$-greedy algorithm with $\epsilon=1- (1-n^{-1/2})^{\frac{1}{H}}$ in order to induce exploration. 
Before taking the action $\Action{h}{n}$, \texttt{RDRLVI} resamples at most $M_{h}^{(n)}=\log(H(\tau+1)^{2}/\delta)/\log(1/(1-|\mA|^{-1}))$ times to ensure that the pseudo-action $\PseudoAction{h}{n} = \Action{h}{n}$. 
Since $\Probability(\PseudoAction k{\tau}=\Action
k{\tau})=\abs{\mA}^{-1}$, the matching event $\mM_{k}{(\tau)}$ occurs with probability at least $1-\delta
H^{-1}(\tau+1)^{-2}$.
In practice, resampling succeeds within a few trials; however, if there is no match after $M_h^{(\tau)}$ trials, the algorithm does not update the estimators.

The computational complexity of \texttt{RDRLVI} is higher than the previous algorithm because it needs to compute the imputation estimator and pseudo-rewards.  
However, this additional cost is compensated by the benefits: \texttt{RDRLVI} uses all samples in estimating value function resulting in a faster convergence rate (Theorem~\ref{thm:est_tail}) and a significantly superior regret bound (Theorem~\ref{thm:upper_bound}) -- all without requiring oracle access to an exploratory policy whose expected Gram matrix has positive RME, or the knowledge of $\sigma_{U}$ and $s_\star$.
This relaxation is possible since the algorithm collects features from \emph{all} actions in $\mA$ to compose a Gram matrix with a larger RME than a Gram matrix generated by an exploratory policy.

%------------------------------------
% Algorithm
%------------------------------------
\begin{algorithm}[t]
\caption{~Randomized Doubly~Robust~Lasso Value Iteration~\texttt{(RDRLVI)}}
\label{alg:DRLSVI}
\begin{algorithmic}
\STATE \textbf{INPUT}: Confidence parameter ($\delta>0$).
\STATE Initialize $\Estimator 10=\cdots=\Estimator H0=0$ and set
$\Estimator{H+1}n=0$. 
\FOR{Episode $n=1,\ldots,N$}
\STATE Receive the initial state $\State 1n$.
\STATE Set $\epsilon_n = 1-(1-n^{-1/2})^{\frac{1}{H}}$
\STATE Set $M_{h}^{(n)}=\ln(H(\tau+1)^{2}/\delta)/\ln(1/(1-|\mA|^{-1}))$
\FOR{period $h=1,\ldots,H$}
\WHILE{($\PseudoAction hn\neq\Action hn$) \AND ($\text{count} \leq M_{h}^{(n)})$}
\STATE Sample $\PseudoAction hn \sim \text{unif}(\mA)$ 
\STATE Select $\Action hn$ using $\epsilon_n$-greedy policy
\[
  \Action hn= \left\{
  \begin{array}{l}
    \underset{a\in\mA}{\mathrm{argmax}}\{\Pi_{[0,H]}\big(\widehat{Q}_{\Estimator{h+1}{n-1}}(\State   
    hn,a)\big)\\
    \hspace*{1in} \text{w.p.}\
1-\epsilon_n,\\
    \sim \text{unif}(\abs{\mA}-1)\\
    \hspace*{1in} \text{w.p.}\ \epsilon_n.
  \end{array}
  \right.
\]
\STATE $\text{count} = \text{count}+ 1$
\ENDWHILE
\STATE Play $\Action hn$
\ENDFOR
\FOR{period $h=H,\ldots,1$}
\STATE Update $\Impute hn$ by minimizing the loss \eqref{eq:impute_loss}.
\IF{$\PseudoAction hn\neq\Action hn$}
\STATE Set $\Estimator hn:=\Estimator h{n-1}$
\ELSE
\STATE Compute pseudo-rewards $\tilde{Y}_{\Impute{h}{n}\!\!,k}^{(\tau)}\!(a)$ in \eqref{eq:p_reward}.
\STATE Compute $\Estimator hn$ by minimizing the loss \eqref{eq:DR_loss}.
\ENDIF
\ENDFOR
\ENDFOR
\end{algorithmic}
\end{algorithm}

%--------------------------
% 5. Regret Analysis
%--------------------------
\section{REGRET ANALYSIS}
Next, we present our novel analysis to establish an upper bound for the regret of~\texttt{RDRLVI}.

\subsection{Analysis for Tail Inequality}
\label{sec:tail_analysis}
First, we bound the regret in terms of the $\ell_1$ error of the estimator $\Estimator{h}{n}$. 
\begin{lem}[Regret decomposition] 
\label{lem:regret_decomposition} 
Let $\widehat{A}_{\algo}$  denote Algorithm~\algo, and for each $n\in[N]$, define $\Estimator{H+1}{n}=\mathbf{0}$ and
\begin{equation}
\Barw hn:=\int_{\mX}\Pi_{[0,H]}\left(\max_{a^{\prime}\in\mA}\widehat{Q}_{\Estimator{h+1}n}(x,a^{\prime})\right)\psi(x)dx,
\label{eq:barw}
\end{equation}
Then, for any $N_1\in[N]$,
\[
\mathrm{R}(N,\!\widehat{A}_{\algo}) \!\le\! 2H(\sqrt{N}+N_1) + 2\!\!\sum_{n=N_1}^{N-1}\!\sum_{h=2}^{H}\|\Estimator h{n}\!-\Barw h{n}\!\|_{1}.
\]
\end{lem}
In the bound, the first term comes from the $\epsilon_n=1-(1-n^{-1/2})^{\frac{1}{H}}$-greedy policy in \texttt{RDRLVI} and the number of episodes $N_1$ required to obtain an effective $\ell_1$- error bound of the estimator $\Estimator{h}{n}$.

\begin{thm}[Tail inequality for the estimator]
\label{thm:est_tail} 
For any given $\delta\in(0,1)$, set $\lambda_{\mathrm{Im}}^{(n)}:=8H\sqrt{n\log\frac{2dHn^{2}}{\delta}}$ 
and $\lambda_{\mathrm{Est}}^{(n)}:=9|\mA|H\sqrt{n\log\frac{2dHn^{2}}{\delta}}$. 
Then, there exists an absolute constant $C$ such that for all $h\in[H]\setminus\{1\}$, and $n\ge C\sigma_{U}^{-4}s_{\star}^{4}H^{2}\log^{5}(dHn^{2}/\delta)\log^{2}(2d)$, 
\[
\norm{\Estimator {h}n-\Barw {h}n}_{1}\le\frac{8s_{\star}}{\sigma_{U}\sqrt{n}}\sqrt{\log\frac{dHn^{2}}{\delta}},
\]
with probability at least $1-12\delta$, where
$\Barw{h}{n}$ defined in~\eqref{eq:barw} and
$\sigma_{U}:=\sigma_{\min}(\Sigma^{{U}},s_{\star})$.
\end{thm}
Note that the episode length $H$ appears only as $\sqrt{\log(H)}$ in the convergence rate of our estimator.
This~is~a~significant improvement compared to the rate~$\tilde{O}(\sigma_U^{-1}s_{\star}(n/H)^{-\frac{1}{2}})$  of the estimator in~\citet{hao2021online} of which estimator only uses $n/H$ episodes in each period $h$.

The main challenge here is to obtain a bound for the residual, 
\begin{align*}
\eta_{\Estimator{h+1}{n},k}^{(\tau)}(a):=&\widehat{Y}_{\Estimator{h+1}{k}}(\State
k{\tau},a)\\
&-\big[\Probability\Pi_{[0,H]}\big(\max_{a^{\prime}\in\mA}
\widehat{Q}_{\Estimator{h+1}{n}}(x,a^{\prime})\big)\big](\State  k{\tau},a).
\end{align*}
Here, $\Estimator{h+1}{n}$ is correlated with $(\State{k}{\tau})_{k\in[H],\tau\in[n]}$ resulting in a bias. 
However, for sufficiently large $n$, the residual
$\eta_{\Estimator{h+1}{n},k}^{(\tau)}$ is close to
$\eta_{w_{h+1}^{\star},k}^{(\tau)}$ and the worst-case bias can be bounded. 
For $\rho>0$, define
$\mW_{h+1}(\rho):=\{w\in\Real^{d}:\|w-w_{h+1}^{\star}\|_1\le\rho\}$ and let
$\phi_{k}^{(\tau)}:=\Feature{\State{k}{\tau}}{\PseudoAction{k}{\tau}}$ and
$\eta_{\Estimator{h+1}{n},k}^{(\tau)}:=\eta_{\Estimator{h+1}{n},k}^{(\tau)}(\PseudoAction{k}{\tau})$. 
We decompose the residual vector as the worst case bound on the vicinity of $w_{h+1}^{\star}$ and on $w_{h+1}^{\star}$,
\begin{equation}
\begin{split}
&\norm{\sum_{\tau=1}^{n}\sum_{k=1}^{H}\eta_{\Estimator{h+1}n,k}^{(\tau)}\phi_{k}^{(\tau)}}_{\infty}\!\!\!\!\le\!\norm{\sum_{\tau=1}^{n}\sum_{k=1}^{H}\eta_{w_{h+1}^{\star},k}^{(\tau)}\phi_{k}^{(\tau)}}_{\infty}\!\!\!\!
\\
&+\!\!\!\sup_{w\in\mW_{h+1}(\rho)}\!\norm{\sum_{\tau=1}^{n}\sum_{k=1}^{H}\big(\eta_{w,k}^{(\tau)}-\eta_{w_{h+1}^{\star},k}^{(\tau)}\big)\phi_{k}^{(\tau)}}_{\infty}\!\!\!\!.
\end{split}
\label{eq:residual_decomposition}
\end{equation}
The following lemma bounds the worst case bound on $\mW_{h+1}(\rho)$.
\begin{lem}[Worst-case bound on the sum of residuals]
\label{lem:sup_bound}
Suppose $n^{3}\ge16e^{2}$ and let $\Action 1{\tau},\ldots,\Action H{\tau}$ denote the selected actions by policy $\pi^{(\tau)}$ and $\phi_{k}^{(\tau)}:=\Feature{\State{k}{\tau}}{\Action{k}{\tau}}$.
Then for any policy $\pi^{(\tau)}$,
\begin{align*}
&\sup_{w\in\mW_{h+1}(\rho)}\!\norm{\sum_{\tau=1}^{n}\sum_{k=1}^{H}\big\{
                 \eta_{w,k}^{(\tau)}(\Action
                 k{\tau})\!-\!\eta_{w_{h+1}^{\star},k}^{(\tau)}(\Action
                 k{\tau})\big\} \phi_{k}^{(\tau)}}_{\infty}\\ 
  &\le\rho\sqrt{2nH\log2d}\Big(8+\frac{256\sqrt{3}}{3}\log^{3/2}
    \frac{Hdn^{2}}{\delta}\Big),
\end{align*}
with probability at least $1-\delta/(Hn^{2})$.
\end{lem}
The proof Lemma~\ref{lem:sup_bound} involves nontrivial extensions using Rademacher complexity for sequential data developed by~\citet{rakhlin2015sequential}, and details are in Appendix~\ref{subsec:seq_rade}. 

Another challenge arises in bounding the first term of \eqref{eq:residual_decomposition}. 
By definition of the optimal value function, it follows that $\eta_{w_{h+1}^{\star},k}^{(\tau)}=\Value h{\star}{X_{k+1}^{(\tau)}(\PseudoAction k{\tau})}-[\Probability V_{h}^{\star}](\State k{\tau},\PseudoAction k{\tau})$. 
To bound the sum of the conditional variance of $V_h^{\star}$, previous inequalities (e.g., Lemma C.5 in \citet{jin2018q}) are not applicable because the actions are not from the optimal policy, and the summation is over the state $k\in[H]$ not the index of the value function $h\in[H]$. 
Hence, we develop the novel inequality in the following lemma. 

\begin{lem}[Bound on sum of variance of the optimal value functions] 
\label{lem:Var_bound}
Let $\Action 1{\tau},\ldots,\Action H{\tau}$ denote a sequence of actions selected by a policy $\pi^{(\tau)}$ and $\mH^{(\tau)}$ denote the sigma algebra generated by $\{\State{h^{\prime}}u,\Action{h^{\prime}}u\}_{u\in[\tau],h^{\prime}\in[H]}$
Then, for any policy $\pi^{(\tau)}$ and $h\in[H]$, the sum of the variance of the optimal value function is bounded by 
\[
\Expectation\bigg[\sum_{k=1}^{H}\big\{ \Value h{\star}{\State{k+1}{\tau}}-\left[\Probability V_{h}^{\star}\right](\State k{\tau}\!\!,\Action k{\tau})\big\}^{2}\bigg|
\mH^{(\tau-1)}\!\bigg] \!\le\! 10H^{2}\!.
\]
\end{lem}
With two novel lemmas, we bound the sum of residuals of $Q$-value function in~\eqref{eq:residual_decomposition} by $\tilde{O}(H\sqrt{n})$.
Lemma~\ref{lem:sup_bound} and Lemma~\ref{lem:Var_bound} can applied to handle correlation for a more general class of estimators.
We defer the detailed derivation in Appendix~\ref{subsec:tail_proof}.

\subsection{A Regret Bound of \texttt{RDRLVI}}
\label{subsec:regret_bound}
\begin{thm}[A regret bound of \texttt{RDRLVI}]
\label{thm:upper_bound}
Fix $\delta\in(0,1)$. 
Then, with probability at least $1-12\delta$, 
\begin{equation}
\begin{split}
&\mathrm{R}(N,\widehat{A}_{\algo})\!\le\!\min\!\bigg\{ HN,\;\;\frac{16s_{\star}H}{\sigma_{U}}\sqrt{N\log dHN^{2}}\\
&\hspace*{0.5in}\mbox{}+2H\Big(\sqrt{N}+\frac{CH^{2}s_{\star}^{4}}{\sigma_{U}^{4}}\log^{5}\frac{dHN^{2}}{\delta}\log^{2}\!2d\Big)\bigg\},
\end{split}
\label{eq:regret_bound}
\end{equation}
for all $N\ge\frac{Cs_{\star}^{4}H^{2}\log^{2}(2d)}{\sigma_{U}^{4}}\log^{5}\frac{2ds_{\star}^{4/5}H^{2/5}}{e\sigma^{4/5}\sqrt{\delta}}$,
where $C > 0$ is an absolute constant.
\end{thm}
The first $HN$ term represents the trivial bound resulting from the rewards bounded by $1$. 
Thus, the leading order term is $\tilde{O}(\sigma_{U}^{-1}s_{\star}H\sqrt{N})$.
For the SMDP such that $\sigma_{U}^2 \ge s_\star/d$, the upper regret bound matches the lower bound~\eqref{eq:lower_bound} up to logarithmic factors. 
As long as $\sigma_{U}$ does not change, our regret bound increases in the logarithmic of the ambient dimension $d$. 
In Section~\ref{subsec:dependency_exp}, we discuss how $\sigma_{U}$ and $d$ affect the regret \texttt{RDRLVI} in our numerical experiments.

With oracle access to an exploratory policy $\pi^{E}$ such that
$\sigma_{E}:=\sigma_{\min}(\Sigma^{\pi^{E}},s_{\star})$ is a positive constant independent of $d$ and $N$, \citet{hao2021online} established an $\tilde{O}(\sigma_{E}^{-\frac{2}{3}}H^{\frac{4}{3}}s_{\star}^{\frac{2}{3}}N^{\frac{2}{3}})$ regret bound for SMDP.
If $s_{\star}$ and $\sigma_{\min}$ are unknown, the regret bound increases
to $\tilde{O}(\sigma_{\min}^{-1}H^{7/3}s_{\star}^{5/3}N^{2/3})$.
Lemma~\ref{lem:RME_inequality}) establishes that the uniform policy
$\pi^{U}$ is also exploratory whenever the SMDP admits an
exploratory policy. 
Therefore, one can design an algorithm that uses $\pi^{U}$ as default choice for an exploratory policy; however, simply using
$\pi^{U}$ for pure exploration results in high regret.  
We employ $\pi^{U}$ to introduce the random pseudo-actions $\PseudoAction{k}{\tau}$ for the DR method and use features from all actions.  
This approach yields $\tilde{O}(\sigma_{U}^{-1}Hs_{\star}N^{\frac{1}{2}})$
regret bound, without the oracle access to $\pi^E$, $s_{\star}$ and
$\sigma_{E}$.

%-------------------------------
% Figure 2. regret vs. sigma
%-------------------------------
\begin{figure}[t]
\centering
\includegraphics[width=0.45\textwidth]{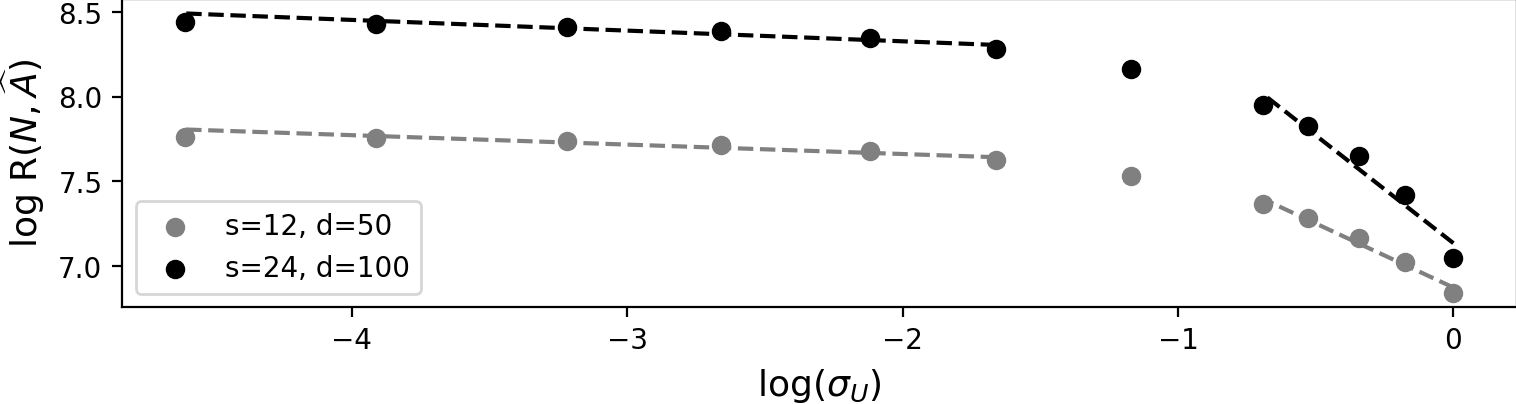}
\caption{\label{fig:regret_sigma} 
Logarithmic of cumulative regret of the proposed \texttt{RDRLVI} algorithm on RME $\sigma_{U}$. 
The dots are average regret based on ten experiments, and the slope of the right regression line is $-0.76$ ($s_\star=12$) and $-1.28$ ($s_\star=24$), respectively.
The slopes of the flat regression lines are both $-0.06$.
The figure supports our regret bound~\eqref{eq:regret_bound} which is proportional to $\sigma_{U}^{-1}$ and converges to $HN$ as $\sigma_{U}$ decreases to 0.
}
\end{figure}

%----------------------------------------
% Figure 3. Regret vs. d
%----------------------------------------
\begin{figure}[t]
%\vspace{.3in}
\centering
\includegraphics[width=0.45\textwidth]{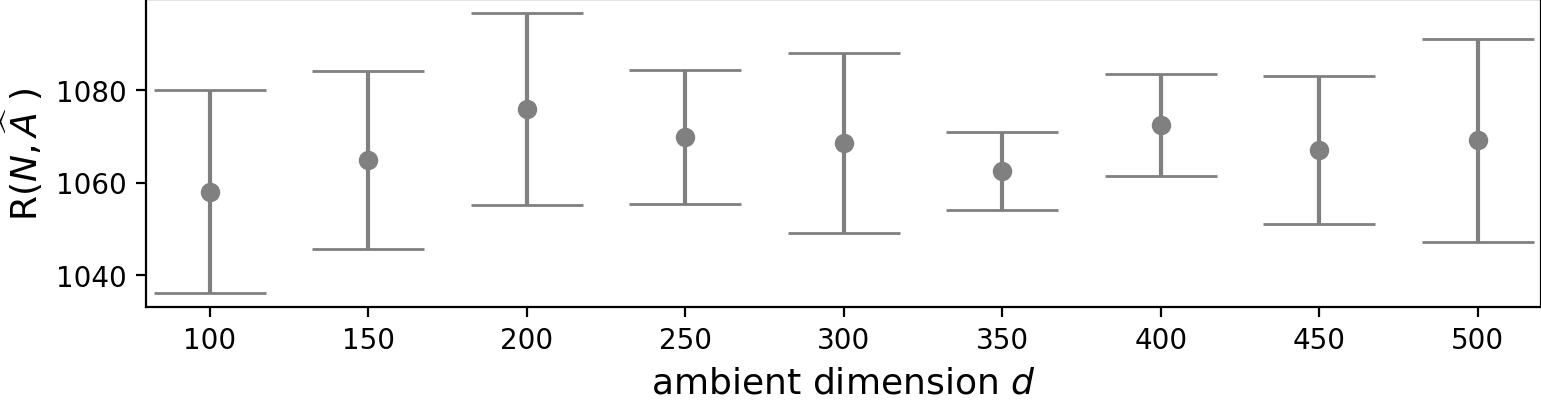}
\caption{\label{fig:8d} Cumulative regrets of the proposed \texttt{RDRLVI} algorithm on increasing ambient dimensions $d$ with $\sigma_{U}=1/6$.
The dots and error bars represent the average and standard deviation based on ten experiments.
As $d$ increases, the regret remains flat since the algorithm selects $s_{\star}$ features among $d$ features.}
\end{figure}

%----------------------------------------
% Figure 4. Comparison with Hao et al.
%----------------------------------------
\begin{figure}[t]
%\vspace{.3in}
\centering
\subfigure[Cumulative regret comparison]{{\label{fig:com_cumulative}\includegraphics[width=0.45\textwidth]{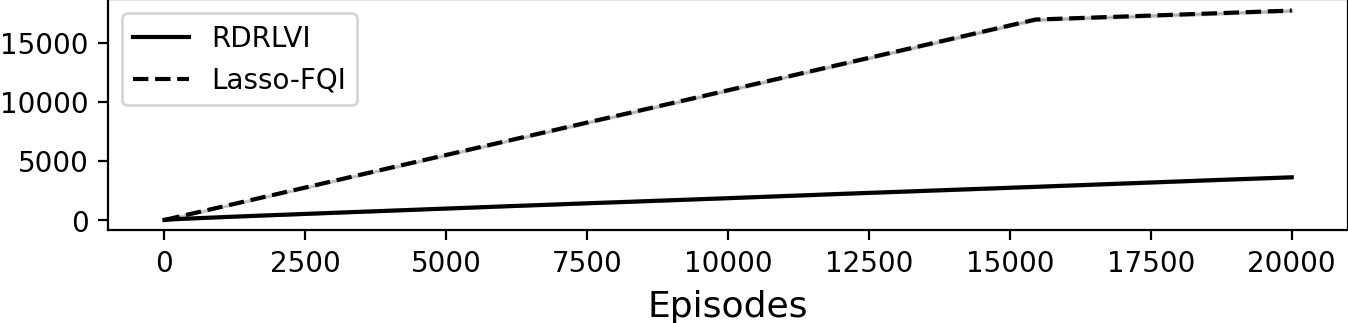}
}}
\subfigure[Average episodic regret comparison]{{\label{fig:com_instant}\includegraphics[width=0.45\textwidth]{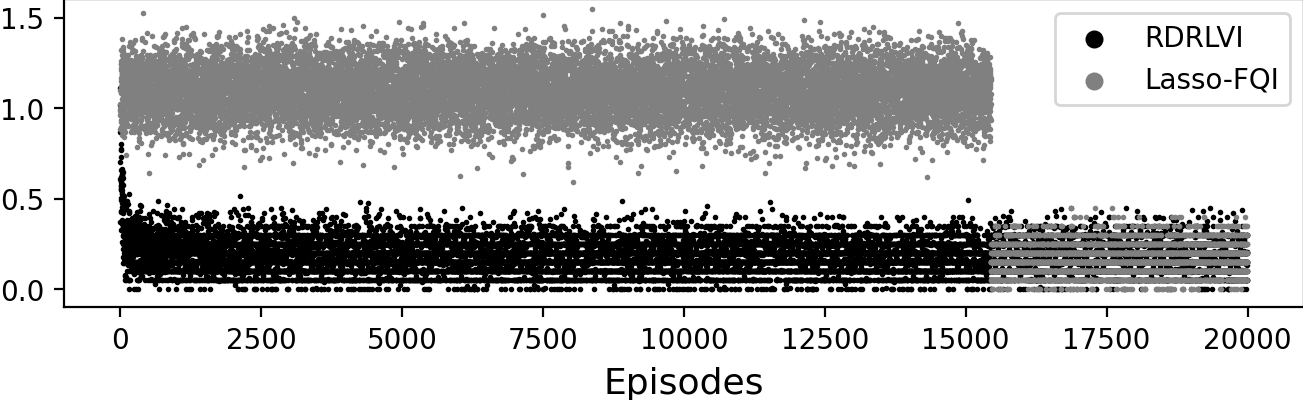}
}}
%\vspace{.1in}
\caption{\label{fig:comparison} Comparison of regrets of the proposed \texttt{RDRLVI} with \texttt{Lasso-FQI} \citep{hao2021online}.
The line and shade represent the average and standard deviation based on ten experiments.
The figures show that \texttt{RDRLVI} finds a low-regret policy while exploiting the reward.}
\end{figure}

%-------------------------------
% 6. Experiments
%-------------------------------

\section{EXPERIMENTS}

In this section, we discuss the results of numerical experiments that validate our theoretical results and superior performance of \algo.
We use an environment where the RME $\sigma_{U}$ is explicitly computable.
Details of the setting are in Appendix~\ref{sec:exp_setting}. 

\subsection{Empirical Analysis of the Regret Bound}
\label{subsec:dependency_exp}

Figure~\ref{fig:regret_sigma} shows the two-phase behavior of the
cumulative regret as the RME $\sigma_U$ decreases to $0$.  
We set $H=10$ and $N=500$ for $s_{\star}=12$ case and $N=1000$ for
$s_{\star}=24$ case. 
For sufficiently large value of $\sigma_U$, $\log \text{R}(N,\widehat{A})$ decreases linearly with $\log(\sigma_U)$ with the slope $-0.76$ for $s_{\star}=12$ ($-1.28$ for $s_{\star}=24$). 
The slope decreases for larger $s_{\star}$ because of the impact of the first term in \eqref{eq:regret_bound}. 
For sufficiently small $\sigma_U$, the regret reaches a plateau as it
converges to the trivial bound $HN$. 
These results validate our regret bound~\eqref{eq:regret_bound}.  

In Figure~\ref{fig:8d}, we plot the cumulative regret as a function of $d$ for a RME $\sigma_{U}=1/6$, $H=10$, $N=500$ and $s_{\star}=8$ (see Appendix~\ref{subsec:additional_exp} for results for other values of $s_{\star}$). 
The results show that when RME is sufficiently large, \algo~quickly identifies the $s_{\star}$ non-zero features, and the dimension
$d$ does not impact the regret of the algorithm.

\subsection{Comparison of \texttt{RDRLVI} and \texttt{Lasso-FQI}}
\label{subsec:com_exp}

We compare our \texttt{RDRLVI} with the Lasso fitted-Q-iteration algorithm (\texttt{Lasso-FQI}) proposed by \citet{hao2021online}. 
\texttt{Lasso-FQI} uses oracle access to the exploratory policy $\pi^{E}$, the size of the active entries $s_{\star}$ and RME $\sigma_{E}$. 
\citet{hao2021online} proposed that the number of episodes $N_1$ for exploration bet set to 
$N_1:=(2048s_{\star}^2H^4N^2\sigma_{E}^{-2}\log(2dH/\delta))^{1/3}$. However, $N_1$ involves worst-case bounds, and the algorithm may over-explore. 
Hence, we reduce the number of episodes used for exploration to  $N_1:=H^{4/3}N^{2/3}s_{\star}^{2/3}\sigma_{E}^{-1}$. 
Similarly, \texttt{RDRLVI} uses reduced the reduced value for
$\lambda_{\text{Im}}:=H\sqrt{n\log(2dH/\delta)}$ with $\delta=0.1$. 
 
Figure~\ref{fig:comparison} show cumulative and episodic regrets of the proposed \texttt{RDRLVI} and \texttt{Lasso-FQI} when $d=200$, $H=2$, $s_{\star}=24$, and $\sigma_U=\sigma_E=1$ (for results on other parameters, see Appendix~\ref{subsec:additional_exp}). 
Since \texttt{Lasso-FQI} chooses action according to the exploratory $\pi_E$ without using the estimated value function when $n \le N_1$, it causes high regret in most episodes. 
When $n\ge N_1$, \texttt{Lasso-FQI} finds the high-reward policy and takes greedy action until the end of the episodes. 
In contrast, the proposed $\texttt{RDRLVI}$ finds the low-regret policy while selecting the best action at each episode. 
We see that \texttt{RDRLVI} balances the trade-off between exploration and exploitation by using unbiased pseudo-rewards and features of all actions and possible states.

\renewcommand\bibname{References}
\bibliographystyle{plainnat}
\bibliography{ref}

\appendix
\onecolumn

\section{SUPPLEMENTARY MATERIALS FOR EXPERIMENTS}

\subsection{Experiment Setting}
\label{sec:exp_setting}

In this section, we present the setting used in our numerical experiment.
For given $d$, let $U_{1}^{(\tau)}\in[-1,1]^d$ denote a random variable whose entries are independent and have equal probability on $[-1,1]$ and we sample initial state $\State{1}{\tau}:=((U^{(\tau)})^\top, 1)^\top$.
For given $s_{\star}=4,8,12,\ldots$, we set $\mA:=[s_{\star}]$.
For each $a\in\mA$, the reward is
\[
\Reward{x}{a}:= \Indicator{x_{d+1}=1}(1-\frac{a-1}{s_{\star}}) + \Indicator{x_{d+1}=-1}\frac{a}{2s_{\star}},
\]
which heavily depends on the last entry of state $x_{d+1}$.
When $x_{d+1}=-1$, the maximum reward is $1/2$ and increasing in $a$.
In contrast, when $x_{d+1}=1$, the maximum reward is $1$ and decreasing in $a$.
Let $\nu:=a \text{ mod } (s_{\star}/4)$.
For any $\sigma>0$, we define feature,
\[
\Feature{x}{a}^\top:=\sigma(-x_{1},\ldots,-x_{\nu-1},x_{\nu},\ldots,x_{s_{\star}/2},-x_{s_{\star}/2+1},\ldots,-x_{3s_{\star}/4-\nu+1},x_{3s_{\star}/4-\nu+2}\ldots,x_{s_{\star}},x_{s_{\star}+1},\ldots,x_{d})
\]
The $\sigma>0$ will control the (restrictive) minimum eigenvalue $\sigma^{U}$.
To define transition distribution of states, let $x_{1:d}:=(x_1,\ldots,x_d)$, and
\begin{align*}
\psi(x_{1:d},1)^{\top}&:=2\sigma^{-1}s_{\star}^{-1}(x_{1}^{-1},\ldots,x_{s_{\star}/2}^{-1},0,\ldots,0),\\
\psi(x_{1:d},-1)^{\top}&:=2\sigma^{-1}s_{\star}^{-1}(0,\ldots,0,x^{-1}_{s_{\star}/2+1},\ldots,x^{-1}_{s_{\star}},0,\ldots,0).
\end{align*}
Now we obtain the transition probability,
\begin{align*}
\CP{(x_{1:d},1)}{(x_{1:d},1),a}&=\Feature{(x_{1:d},\pm1)}a^{\top}\psi(x_{1:d},1)=1-\frac{4(\nu-1)}{s_{\star}}\\
\CP{(x_{1:d},-1)}{(x_{1:d},1),a}&=\Feature{(x_{1:d},\pm1)}a^{\top}\psi(x_{1:d},-1)=\frac{4(\nu-1)}{s_{\star}}.
\end{align*}
Because $x^{(\tau)}_{1:d}=U_1^{(\tau)}$, we have $\sigma_U=\sigma/6$.
The optimal policy is to choose $a=1$, where the state stays $x_{d+1}=1$ and reward  $\Reward{x}{1}=1$. 
Therefore, the optimal policy gains $HN$ reward for $N$ episodes. 

\subsection{Additional Experiment Results}
\label{subsec:additional_exp}

In this section, we present additional numerical results demonstrating the superior performance of \algo.
In Figure~\ref{fig:d_appendix}, we plot the cumulative regret as a function of $d$ for a RME $\sigma_U=1/6$, $H=10$ and $N=1000$ with two different values of $s_{\star}=16, 24$.
The results show that, for other environments than in Section~\ref{subsec:dependency_exp}, \algo also quickly finds the $s_{\star}$ non-zero weights, and the dimension does not impact the regret.

%----------------------------------------
% Figure 3. Regret vs. d
%----------------------------------------
\begin{figure}[ht]
\vspace{.15in}
\centering
\subfigure[Regret changes with increasing $d$ when $s_{\star}=16$.]{{\label{fig:d16}\includegraphics[width=0.45\textwidth]{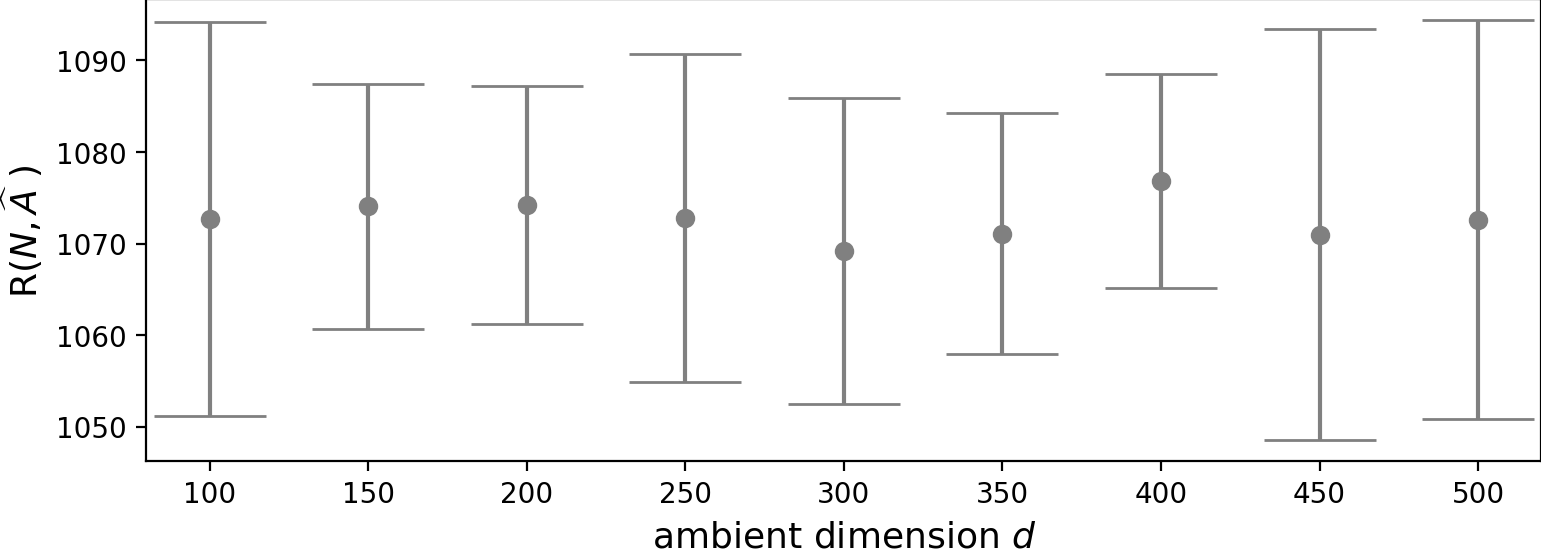}
}}
\subfigure[Regret changes with increasing $d$ when $s_{\star}=24$.]{{\label{fig:d24}\includegraphics[width=0.45\textwidth]{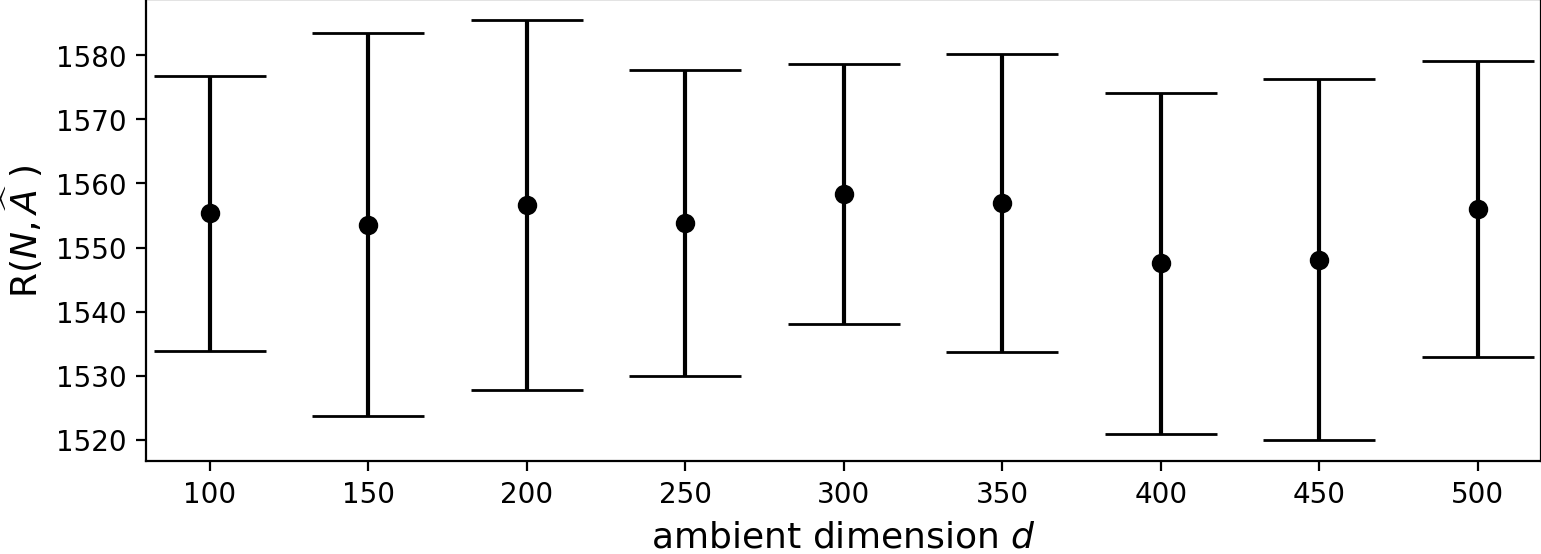}
}}
\vspace{.3in}
\caption{\label{fig:d_appendix} Cumulative regrets of the proposed \texttt{RDRLVI} algorithm on increasing ambient dimensions $d$ with $\sigma_{U}=1/6$, $H=10$, and $N=1000$.
The dots and error bars represent the average and standard deviation based on ten experiments.
As $d$ increases, the regret remains flat since the algorithm selects $s_{\star}$ features among $d$ features.}
\end{figure}

%----------------------------------------
% Figure 4. Comparison with Hao et al.
%----------------------------------------
\begin{figure}[ht]
%\vspace{.3in}
\centering
\subfigure[Cumulative regret comparison ($H=3$, $s_{\star}=8$).]{{\label{fig:com_H24}\includegraphics[width=0.45\textwidth]{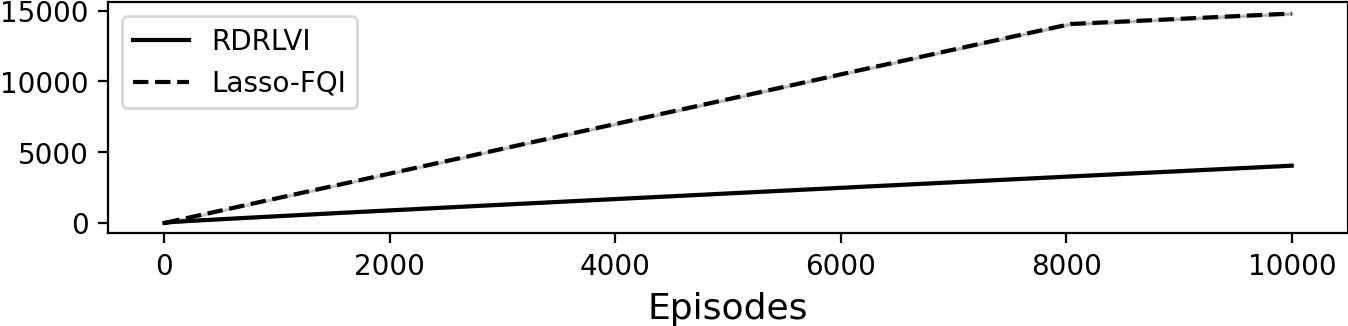}
}}
\subfigure[Cumulative regret comparison ($H=5$, $s_{\star}=4$).]{{\label{fig:com_H4}\includegraphics[width=0.45\textwidth]{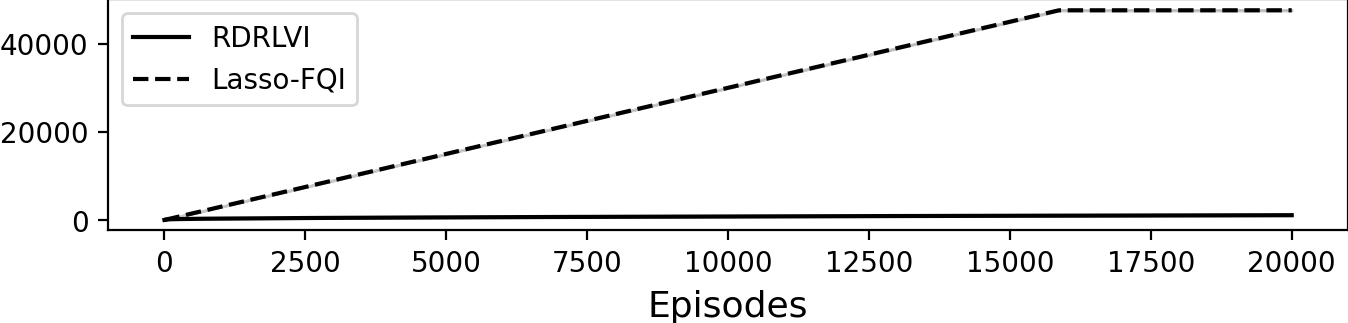}
}}
\subfigure[Episodic regret comparison ($H=3$, $s_{\star}=8$).]{{\label{fig:com_instant24}\includegraphics[width=0.45\textwidth]{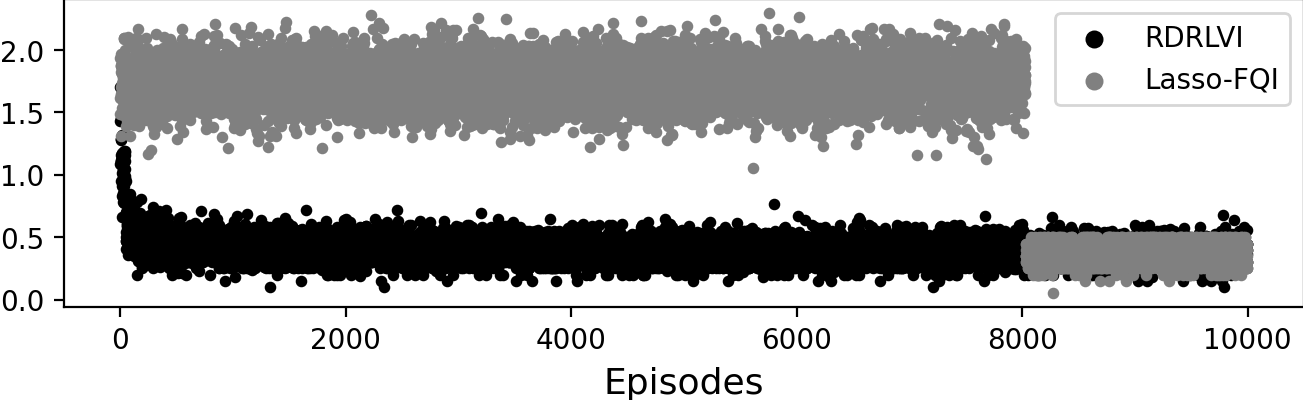}
}}
\subfigure[Episodic regret comparison ($H=5$, $s_{\star}=4$)]{{\label{fig:com_instant4}\includegraphics[width=0.45\textwidth]{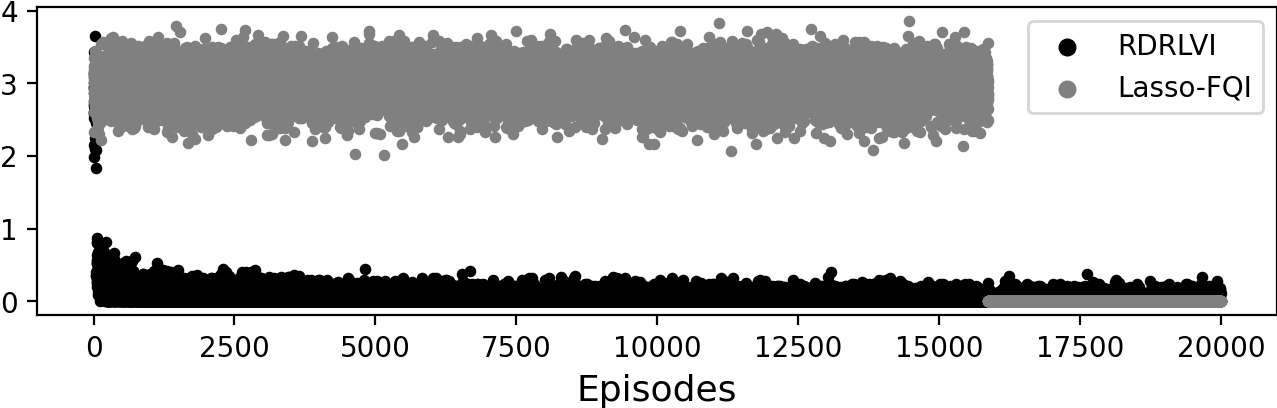}
}}
%\vspace{.1in}
\caption{\label{fig:comparison_appendix} Comparison of regrets of the proposed \texttt{RDRLVI} with \texttt{Lasso-FQI} \citep{hao2021online} when $d=200$ and $\sigma_{U}=1$.
The line and shade represent the average and standard deviation based on ten experiments.
The figures show that \texttt{RDRLVI} finds a low-regret policy while exploiting the reward and achieves lower regret than \texttt{Lasso-FQI}.}
\end{figure}

As in Section~\ref{subsec:com_exp}, we present additional results of comparing \texttt{LASSO-FQI} \citep{hao2021online} with our \algo.
Figure~\ref{fig:comparison_appendix} shows cumulative and episodic regrets of \texttt{LASSO-FQI} and our \algo.
The results show that, for other environments than in Section~\ref{subsec:com_exp}, \algo also finds the low-regret policy while selecting the best action at each episode.
We see that \algo balances the trade-off between exploration and exploitation by using unbiased pseudo-rewards and features of all actions and possible states.

\section{TECHNICAL LEMMAS}
In this section, we present technical lemmas used in our analysis.
We provide proofs after the novel lemmas.

\subsection{Comparison of an Exploratory Policy and the Uniform Policy}
\begin{lem}[Comparison of $\sigma_U$ and $\sigma_E$]
\label{lem:RME_inequality}
Let $\pi^{E}$ denote an exploratory policy such that $\sigma_{E}:=\sigma_{\min}(\Sigma(\pi^{E}),s_{\star})>0$ and $\pi^{E}(x,a)$ denote the probability of selecting an action $a\in\mA$ for the state $x\in\mX$.
Then,
\begin{equation}
\sigma_{U} \ge\big(\max_{(x,a)\in\mX\times\mA}\pi^{E}(x,a)\abs{\mA}\big)^{-H}\sigma_{E}.
\label{eq:eigen_inequality}
\end{equation}
\end{lem}
\begin{rem}
The inequality~\eqref{eq:eigen_inequality} does not involve $d$ and $N$.
While the cost of the worst-case of replacing the $\pi^{E}$ by $\pi^{U}$ involves $|\mA|$ and $H$, it does not involve $d$ or $N$.
We can redefine an exploratory policy $\pi^{E}$ close to the uniform policy which gives $\max_{(x,a)\in\mX\times\mA}\pi^{E}(x,a)=O(|\mA|^{-1})$.
\end{rem}

\begin{proof}
Let $\Phi(x,a):=\Feature xa\Feature xa^{\top}.$ Let $a_{1},\ldots,a_{H}$
denote a sequence of actions selected by the policy $\pi^{E}$. Note
that 
\begin{align*}
H\Sigma^{\pi^{E}}:= & \Expectation^{\pi^{E}}\left[\sum_{k=1}^{H}\Phi(x_{k},a_{k})\right]\\
= & \Expectation^{\pi^{E}}\left[\sum_{k=2}^{H}\Phi(x_{k},a_{k})\right]+\Expectation^{\pi^{E}}\left[\CE{\Phi(x_{1},a_{1})}{x_{1}}\right]\\
= & \Expectation^{\pi^{E}}\left[\sum_{k=2}^{H}\Phi(x_{k},a_{k})\right]+\Expectation^{\pi^{E}}\left[\CE{\sum_{a\in\mA}\pi^{E}(x_{1},a)\Phi(x_{1},a)}{x_{1}}\right]\\
= & \Expectation^{\pi^{E}}\left[\sum_{k=2}^{H}\Phi(x_{k},a_{k})\right]+\Expectation^{\pi^{E}}\left[\sum_{a\in\mA}\pi^{E}(x_{1},a)\Phi(x_{1},a)\right].
\end{align*}
where $\pi^{E}(x,a)$ is the probability that the policy $\pi^{E}$
selects an action $a$ when the state is $x$. Recursively, we obtain,
\[
H\Sigma^{\pi^{E}}=\sum_{k=1}^{H}\Expectation^{\pi^{E}}\left[\sum_{a\in\mA}\pi^{E}(x_{k},a)\Phi(x_{k},a)\right].
\]
Because $x_{1}$ is sampled from $\Probability_{0}$,
\[
\Expectation^{\pi^{E}}\left[\sum_{a\in\mA}\pi^{E}(x_{1},a)\Phi(x_{1},a)\right]=\int_{\mX}\sum_{a\in\mA}\pi^{E}(z_{1},a)\Phi(z_{1},a)d\Probability_{0}(z_{1}).
\]
 Using the SMDP setting, for each $k\ge2$, 
\begin{align*}
&\Expectation^{\pi^{E}}\left[\sum_{a\in\mA}\pi^{E}(x_{k},a)\Phi(x_{k},a)\right]\\
&= \Expectation^{\pi^{E}}\left[\CE{\sum_{a\in\mA}\pi^{E}(x_{k},a)\Phi(x_{k},a)}{x_{k-1},a_{k-1}}\right]\\
&= \Expectation^{\pi^{E}}\left[\int_{\mX}\sum_{a\in\mA}\pi^{E}(z_{k},a)\Phi(z_{k},a)\Feature{x_{k-1}}{a_{k-1}}^{\top}\psi(z_{k})dz_{k}\right]\\
&= \Expectation^{\pi^{E}}\left[\CE{\int_{\mX}\sum_{a\in\mA}\pi^{E}(z_{k},a)\Phi(z_{k},a)\Feature{x_{k-1}}{a_{k-1}}^{\top}\psi(z_{k})dz_{k}}{x_{k-1}}\right]\\
&= \Expectation^{\pi^{E}}\left[\int_{\mX}\sum_{u_{k}\in\mA}\sum_{u_{k-1}\in\mA}\pi^{E}(z_{k},u_{k})\Phi(z_{k},u_{k})\pi^{E}(x_{k-1},u_{k-1})\Feature{x_{k-1}}{u_{k-1}}^{\top}\psi(z_{k})dz_{k}\right].
\end{align*}
Applying the equality recursively, we obtain,
\begin{align*}
 & \Expectation^{\pi^{E}}\left[\sum_{a\in\mA}\pi^{E}(x_{k},a)\Phi(x_{k},a)\right]\\
 & =\int_{\mX^{k}}\sum_{j=1}^{k}\sum_{u_{j}\in\mA}\Phi(z_{k},u_{k})\left(\prod_{j=1}^{k}\pi^{E}(z_{j},u_{j})\right)\left(\prod_{j=2}^{k}\Feature{z_{j-1}}{u_{j-1}}^{\top}\psi(z_{j})\right)dz_{k}\cdots dz_{2}d\Probability(z_{1})\\
 & \preceq\left(\max_{(x,a)\in\mX\times\mA}\pi^{E}(x,a)\right)^{k}\int_{\mX^{k}}\sum_{j=1}^{k}\sum_{u_{j}\in\mA}\Phi(z_{k},u_{k})\left(\prod_{j=2}^{k}\Feature{z_{j-1}}{u_{j-1}}^{\top}\psi(z_{j})\right)dz_{k}\cdots dz_{2}d\Probability(z_{1})
\end{align*}
Because $\pi^{U}(x,a)=|\mA|^{-1}$,
\begin{align*}
 & \Expectation^{\pi^{E}}\left[\sum_{a\in\mA}\pi^{E}(x_{k},a)\Phi(x_{k},a)\right]\\
 & =\big(\abs{\mA}\max_{(x,a)\in\mX\times\mA}\pi^{E}(x,a)\big)^{k}\!\!\int_{\mX^{k}}\sum_{j=1}^{k}\sum_{u_{j}\in\mA}\Phi(z_{k},u_{k})\bigg(\prod_{j=1}^{k}\pi^{U}(z_{j},u_{j})\bigg)\bigg(\prod_{j=2}^{k}\Feature{z_{j-1}}{u_{j-1}}^{\top}\psi(z_{j})\bigg)dz_{k}\cdots dz_{2}d\Probability(z_{1})\\
 & \preceq\big(\abs{\mA}\max_{(x,a)\in\mX\times\mA}\pi^{E}(x,a)\big)^{k}\Expectation^{\pi^{U}}\left[\sum_{a\in\mA}\pi^{U}(x_{k},a)\Phi(x_{k},a)\right].
\end{align*}
Thus,
\begin{align*}
H\Sigma^{\pi^{E}}\preceq & \sum_{k=1}^{H}\left(\abs{\mA}\max_{(x,a)\in\mX\times\mA}\pi^{E}(x,a)\right)^{k}\Expectation^{\pi^{U}}\left[\sum_{a\in\mA}\pi^{U}(x_{k},a)\Phi(x_{k},a)\right]\\
\preceq & \left(\abs{\mA}\max_{(x,a)\in\mX\times\mA}\pi^{E}(x,a)\right)^{H}\sum_{k=1}^{H}\Expectation^{\pi^{U}}\left[\sum_{a\in\mA}\pi^{U}(x_{k},a)\Phi(x_{k},a)\right]\\
= & \left(\abs{\mA}\max_{(x,a)\in\mX\times\mA}\pi^{E}(x,a)\right)^{H}\Expectation^{\pi^{U}}\left[\sum_{k=1}^{H}\Phi(H_{k},a_{k})\right].
\end{align*}
This concludes the proof.
\end{proof}

\subsection{Lower Bound for the Restrictive Minimum Eigenvalue}
\begin{lem}
\label{lem:eigenvalue_bound} (Corollary 6.8 in \cite{buhlmann2011statistics})
Let $\Sigma_{0}$ and $\Sigma_{1}$ be two positive semi-definite
block diagonal matrices. Suppose that the restricted eigenvalue of
$\Sigma_{0}$ satisfies $\sigma_{\min}(\Sigma_{0},s)>0$ and $\|\Sigma_{1}-\Sigma_{0}\|_{\infty}\le\sigma_{\min}(\Sigma_{0},s)/(32s)$.
Then the restrictive eigenvalue of $\Sigma_{1}$ satisfies $\sigma_{\min}(\Sigma_{1},s)>\sigma_{\min}(\Sigma_{0},s)/2$. 
\end{lem}

\subsection{An Error bound for the Lasso Estimator}
\begin{lem}
\label{lem:lasso} (An $\ell_{1}$-error bound for Lasso
estimator) Let $\{x_{\tau}\}_{\tau\in[t]}$ denote the covariates
in $[-1,1]^{d}$ and $y_{\tau}=x_{\tau}^{\top}\bar{w}+e_{\tau}$ for
some $\bar{w}\in\Real^{d}$ and $e_{\tau}\in\Real$. For $\lambda>0$,
let 
\[
\widehat{w}_{t}=\arg\min_{w}\sum_{\tau=1}^{t}\left(y_{\tau}-x_{\tau}^{\top}w\right)^{2}+\lambda\norm w_{1}.
\]
Let $\bar{\mS}:=\{i\in[d]:\bar{w}(i)\neq0\}$ and $\Sigma_{t}:=\sum_{\tau=1}^{t}x_{\tau}x_{\tau}^{\top}$.
Suppose $\|\sum_{\tau=1}^{t}e_{\tau}x_{\tau}\|\le\frac{\lambda}{4},$
for some $\lambda>0$ and
$\|t^{-1}\Sigma_{t}-\bar{\Sigma}\|_{\infty}\le32|\bar{\mS}|^{-1
}\sigma_{\min}(\bar{\Sigma},\abs{\bar{\mS}})$
for some $\bar{\Sigma}\in\Real^{d\times d}$. Then the $\ell_{1}$-error
is bounded as
\[
\norm{\widehat{w}_{t}-\bar{w}}_{1}\le\frac{8\lambda\abs{\bar{\mS}}}{t\sigma_{\min}\left(\bar{\Sigma},\abs{\bar{\mS}}\right)}.
\]
\end{lem}

\begin{proof}
Let $X_{t}^{\top}:=(x_{1},\ldots,x_{t})\in[-1,1]^{d\times t}$ and
$\mathbf{e}_{t}^{\top}:=(e_{1},\ldots,e_{t})\in\Real^{t}$. We write
$X_{t}(j)$ and $\widehat{w}_{t}(j)$ as the $j$-th column of $X_{t}$
and $j$-th entry of $\widehat{w}_{t}$, respectively. By definition
of $\widehat{w}_{t}$,
\[
\norm{X_{t}\left(\bar{w}-\widehat{w}_{t}\right)+\mathbf{e}_{t}}_{2}^{2}+\lambda\norm{\widehat{w}_{t}}_{1}\le\norm{\mathbf{e}_{t}^{(j)}}_{2}^{2}+\lambda\norm{\bar{w}}_{1},
\]
which implies
\begin{align*}
\norm{X_{t}\left(\bar{w}-\widehat{w}_{t}\right)}_{2}^{2}+\lambda\norm{\widehat{w}_{t}}_{1}\le & 2\left(\widehat{w}_{t}-\bar{w}\right)^{\top}X_{t}^{\top}\mathbf{e}_{t}+\lambda\norm{\bar{w}}_{1}\\
\le & 2\norm{\widehat{w}_{t}-\bar{w}}_{1}\norm{X_{t}^{\top}\mathbf{e}_{t}}_{\infty}+\lambda\norm{\bar{w}}_{1}\\
\le & \frac{\lambda}{2}\norm{\widehat{w}_{t}-\bar{w}}_{1}+\lambda\norm{\bar{w}}_{1},
\end{align*}
where the last inequality uses the bound on $\lambda$. On the left
hand side, by triangle inequality,
\begin{align*}
\norm{\widehat{w}_{t}}_{1}= & \sum_{i\in\bar{\mS}}\abs{\widehat{w}_{t}(i)}+\sum_{i\in[d]\setminus\bar{\mS}}\abs{\widehat{w}_{t}(i)}\\
\ge & \sum_{i\in\bar{\mS}}\abs{\widehat{w}_{t}(i)}-\sum_{i\in\mS_{\star}}\abs{\widehat{w}_{t}(i)-\bar{w}(i)}+\sum_{i\in[d]\setminus\bar{\mS}}\abs{\bar{w}(i)}\\
= & \norm{\bar{w}}_{1}-\sum_{i\in\bar{\mS}}\abs{\widehat{w}_{t}(i)-\bar{w}(i)}+\sum_{i\in[d]\setminus\bar{\mS}}\abs{\widehat{w}_{t}(i)}
\end{align*}
and for the right hand side,
\[
\norm{\widehat{w}_{t}-\bar{w}}_{1}=\sum_{i\in\bar{\mS}}\abs{\widehat{w}_{t}(i)-\bar{w}(i)}+\sum_{i\in[d]\setminus\bar{\mS}}\abs{\widehat{w}_{t}(i)}.
\]
Plugging in both sides and rearranging the terms,
\begin{equation}
2\norm{X_{t}\left(\bar{w}-\widehat{w}_{t}\right)}_{2}^{2}+\lambda\sum_{i\in[d]\setminus\bar{\mS}}\abs{\widehat{w}_{t}(i)}\le3\lambda\sum_{i\in\bar{\mS}}\abs{\widehat{w}_{t}(i)-\bar{w}(i)}.\label{eq:w_basic}
\end{equation}
The inequality \eqref{eq:w_basic} implies $\sum_{i\in[d]\setminus\bar{\mS}}|\widehat{w}_{t}(i)-\bar{w}(i)|\le3\sum_{i\in\bar{\mS}}|\widehat{w}_{t}(i)-\bar{w}(i)|$
and
\begin{align*}
\norm{X_{t}\left(\bar{w}-\widehat{w}_{t}\right)}_{2}^{2}\ge & \sigma_{\min}\left(X_{t}^{\top}X_{t},\abs{\bar{\mS}}\right)\sum_{i\in\bar{\mS}}\abs{\widehat{w}_{t}(i)-\bar{w}(i)}^{2}\\
\ge & \frac{\sigma_{\min}\left(X_{t}^{\top}X_{t},\abs{\bar{\mS}}\right)}{\abs{\bar{\mS}}}\left(\sum_{i\in\bar{\mS}}\abs{\widehat{w}_{t}(i)-\bar{w}(i)}\right)^{2}\\
\ge & \frac{\sigma_{\min}\left(t\bar{\Sigma},\abs{\bar{\mS}}\right)}{2\abs{\bar{\mS}}}\left(\sum_{i\in\bar{\mS}}\abs{\widehat{w}_{t}(i)-\bar{w}(i)}\right)^{2},
\end{align*}
where the last inequality holds by assumption $\|t^{-1}\Sigma_{t}-\bar{\Sigma}\|_{\infty}\le32|\bar{\mS}|^{-1}\sigma_{\min}(\bar{\Sigma},\abs{\bar{\mS}})$
and Lemma \ref{lem:eigenvalue_bound}. Plugging in \eqref{eq:w_basic}
gives,
\begin{align*}
2\norm{X_{t}\left(\bar{w}-\widehat{w}_{t}\right)}_{2}^{2}+\lambda\left(\sum_{i\in[d]\setminus\bar{\mS}}\abs{\widehat{w}_{t}(i)}+\sum_{i\in\bar{\mS}}\abs{\widehat{w}_{t}(i)-\bar{w}(i)}\right)\le & 4\lambda\sqrt{\frac{2\abs{\bar{\mS}}}{\sigma_{\min}\left(t\bar{\Sigma},\abs{\bar{\mS}}\right)}}\norm{X_{t}\left(\bar{w}-\widehat{w}_{t}\right)}_{2}\\
\le & \frac{8\lambda^{2}\abs{\bar{\mS}}}{\sigma_{\min}\left(t\bar{\Sigma},\abs{\bar{\mS}}\right)}+\norm{X_{t}\left(\bar{w}-\widehat{w}_{t}\right)}_{2},
\end{align*}
where the last inequality uses $ab\le a^{2}/4+b^{2}$. 
Rearranging the terms,
\[
\norm{X_{t}\left(\bar{w}-\widehat{w}_{t}\right)}_{2}^{2}+\lambda\norm{\widehat{w}_{t}-\bar{w}}_{1}\le\frac{8\lambda^{2}\abs{\bar{\mS}}}{\sigma_{\min}\left(t\bar{\Sigma},\abs{\bar{\mS}}\right)},
\]
which proves the result. 
\end{proof}

\subsection{Sequential Rademacher Complexity for Martingales}
\label{subsec:seq_rade}
The following lemma connects the sum of martingale differences to the sequential Rademacher complexity.

\begin{lem}
\label{lem:prob_bound} (Lemma 4 in \citet{rakhlin2015sequential}.) Let $Z_{i}\in\mathcal{Z}$ denote a stochastic process adapted to filtration $\mH_{i}$ and $\mF$ a class of functions $f:\mathcal{Z}\to[-1,1]$. Let $\mathbf{z}:=(\mathbf{z}_{1},\ldots,\mathbf{z}_{n})$ denote a sequence of binary trees $\mathbf{z}_{i}:\{\pm1\}^{i-1}\to\mathcal{Z}$ and $\{\xi_{i}\}_{i\in[n]}$ denote independent Bernoulli random variables such that $\Probability(\xi_{i}=-1)=\Probability(\xi=1)=1/2$. 
Then for any $\alpha>0$,
\[
\Probability\left(\sup_{f\in\mF}\abs{\sum_{i=1}^{n}f(Z_{i})-\CE{Z_{i}}{\mH_{i-1}}}>\alpha\right)\le4\sup_{\mathbf{z}}\CP{\sup_{f\in\mF}\abs{\sum_{i=1}^{n}\xi_{i}f(\mathbf{z}_{i}(\xi_{1},\ldots,\xi_{i-1}))}>\frac{\alpha}{4}}{\mathbf{z}}.
\]
\end{lem}

We provide a novel lemma for a bound for sequential Rademacher complexity \citep{rakhlin2015sequential}.
The following lemma is a generalization of Lemma 6 in \citet{rakhlin2015sequential}.

\begin{lem}
\label{lem:entropy_bound} (Bound for sequential Rademacher complexity)
Let $\{\xi_{i}\}_{i\in[n]}$ denote independent Bernoulli random variables
such that $\Probability(\xi_{i}=-1)=\Probability(\xi=1)=1/2$ and
$\mathbf{z}:=(\mathbf{z}_{1},\ldots,\mathbf{z}_{n})$ denote a sequence
of binary trees $\mathbf{z}_{i}:\{\pm1\}^{i-1}\to\mathcal{Z}$. Let
$\mF$ denote a class of functions $f:\mathcal{Z}\to[-1,1]$. For
a fixed tree $\mathbf{z}$ and $\epsilon>0$, let $N(\epsilon,\mF,\|\cdot\|_{\infty,\mathbf{z}})$
denote a covering number of $\mF$ in the norm defined by $\|f\|_{\infty,\mathbf{z}}:=\max_{\{\xi_{i}\}_{i\in[n]}\in\{\pm1\}^{n}}|f(\mathbf{z}(\xi_{1},\ldots,\xi_{n})|$.
Then with probability at least $1-\delta$,

\[
\sup_{f\in\mF}\abs{\sum_{i=1}^{n}\xi_{i}f(\mathbf{z}_{i}(\xi_{1:i-1}))}\le2\inf_{\alpha>0}\left\{ n\alpha+2\int_{\alpha}^{1/2}\sqrt{3n\log\frac{N(\epsilon,\mF,\|\cdot\|_{\infty,\mathbf{z}})}{\sqrt{\delta}}}d\epsilon\right\}.
\]
\end{lem}

\begin{proof}
For given $\epsilon>0$, define $\epsilon_{j}=2^{-j}$. For a fixed
tree $\mathbf{z}$ of depth $n$, let $V_{j}$ be an $\epsilon_{j}$-cover
with respect to $\ell_{\infty}$-norm, $\|\cdot\|_{\infty,\mathbf{z}}$.
For any path $\xi:=\{\xi_{i}\}_{i\in[n]}\in\{\pm1\}^{n}$ and any
$f\in\mF$, let $\mathbf{v}^{(j)}(f,\xi)\in V_{j}$ denote a $\epsilon_{j}$-close
element of the cover in the $\|\cdot\|_{\infty,\mathbf{z}}$-norm.
Now for any $f\in\mF$ and $J\in\mathbb{N}$, 
\begin{align*}
 & \abs{\sum_{i=1}^{n}\xi_{i}f(\mathbf{z}_{i}(\xi_{1:i-1}))}\\
 & \le\abs{\sum_{i=1}^{n}\xi_{i}\left\{ f(\mathbf{z}_{i}(\xi_{1:i-1}))-\mathbf{v}_{i}^{(J)}(f,\xi_{1:i-1})\right\} }+\sum_{j=1}^{J}\abs{\sum_{i=1}^{n}\xi_{i}\left\{ \mathbf{v}_{i}^{(j)}(f,\xi_{1:i-1})-\mathbf{v}_{i}^{(j-1)}(f,\xi_{1:i-1})\right\} }\\
 & \le n\max_{i\in[n]}\abs{f(\mathbf{z}_{i}(\xi_{1:i-1}))-\mathbf{v}_{i}^{(J)}(f,\xi_{1:i-1})}+\sum_{j=1}^{J}\abs{\sum_{i=1}^{n}\xi_{i}\left\{ \mathbf{v}_{i}^{(j)}(f,\xi_{1:i-1})-\mathbf{v}_{i}^{(j-1)}(f,\xi_{1:i-1})\right\} }\\
 & \le n\max_{\xi^{\prime}\in\{\pm1\}^{n}}\max_{i\in[n]}\abs{f(\mathbf{z}_{i}(\xi_{1:i-1}^{\prime}))-\mathbf{v}_{i}^{(J)}(f,\xi_{1:i-1}^{\prime})}+\sum_{j=1}^{J}\abs{\sum_{i=1}^{n}\xi_{i}\left\{ \mathbf{v}_{i}^{(j)}(f,\xi_{1:i-1})-\mathbf{v}_{i}^{(j-1)}(f,\xi_{1:i-1})\right\} }\\
 & =n\|f(\mathbf{z})-\mathbf{v}^{(J)}(f)\|_{\infty,\mathbf{z}}+\sum_{j=1}^{J}\abs{\sum_{i=1}^{n}\xi_{i}\left\{ \mathbf{v}_{i}^{(j)}(f,\xi_{1:i-1})-\mathbf{v}_{i}^{(j-1)}(f,\xi_{1:i-1})\right\} }
\end{align*}
Because $\mathbf{v}^{(N)}(f,\xi)\in V_{N}$,
\[
\sup_{f\in\mF}\abs{\sum_{i=1}^{n}\xi_{i}f(\mathbf{z}_{i}(\xi_{1:i-1}))}\le n\epsilon_{J}+\sup_{f\in\mF}\left(\sum_{j=1}^{J}\abs{\sum_{i=1}^{n}\xi_{i}\left\{ \mathbf{v}_{i}^{(j)}(f,\xi_{1:i-1})-\mathbf{v}_{i}^{(j-1)}(f,\xi_{1:i-1})\right\} }\right).
\]
To bound the second term, consider all possible pairs of $\mathbf{v}^{r}\in V_{j-1}$
and $\mathbf{v}^{s}\in V_{j}$, for $r\in[|V_{j-1}|]$ and $s\in[|V_{j}|]$.
For each pair $(\mathbf{v}^{r},\mathbf{v}^{s})$, define a real-valued
tree $\mathbf{w}^{(j|r,s)}$by
\[
\mathbf{w}_{i}^{(j|r,s)}(\xi):=\begin{cases}
\mathbf{v}_{i}^{s}(\xi)-\mathbf{v}_{i}^{r}(\xi) & \text{if there exists }f\in\mF\text{ s.t. }\mathbf{v}^{s}=\mathbf{v}^{(j)}(f,\xi),\mathbf{v}^{r}=\mathbf{v}^{(j-1)}(f,\xi)\\
0 & \text{otherwise}
\end{cases},
\]
for all $i\in[n]$ and $\xi\in\{\pm1\}^{n}$. Note that $\mathbf{w}^{(j|r,s)}$
is non-zero only on those $\xi$ such that $\mathbf{v}^{(r)}$and
$\mathbf{v}^{(s)}$ are the members of covers $V_{j}$ and $V_{j-1}$
close in the $\|\cdot\|_{\infty,\mathbf{z}}$-norm for some $f\in\mF$.
Define the set of trees $W_{j}$,
\[
W_{j}:=\{\mathbf{w}^{(j|r,s)}:1\le r\le|V_{j-1}|,1\le s\le|V_{j}|\}.
\]
Then we get 
\begin{equation}
\begin{split}\sup_{f\in\mF}\abs{\sum_{i=1}^{n}\xi_{i}f(\mathbf{z}_{i}(\xi_{1:i-1}))}\le & n\epsilon_{J}+\sup_{f\in\mF}\left(\sum_{j=1}^{J}\abs{\sum_{i=1}^{n}\xi_{i}\left\{ \mathbf{v}_{i}^{(j)}(f,\xi_{1:i-1})-\mathbf{v}_{i}^{(j-1)}(f,\xi_{1:i-1})\right\} }\right).\\
\le & n\epsilon_{J}+\sum_{j=1}^{J}\sup_{\mathbf{w}^{(j)}\in W_{j}}\abs{\sum_{i=1}^{n}\xi_{i}\mathbf{w}_{i}^{(j)}(\xi_{1:i-1})}.
\end{split}
\label{eq:entropy_bound_1}
\end{equation}
Note that for $\mathbf{w}^{(j)}\in W_{j}$, there exists $f\in\mF$
such that
\begin{align*}
\|\mathbf{w}^{(j)}\|_{\infty,\mathbf{z}}\le & \sup_{\mathbf{v}^{r}\in V_{j-1},\mathbf{v}^{s}\in V_{j}}\sup_{\xi\in\{\pm1\}^{n}}\max_{i\in[n]}\abs{\mathbf{v}_{i}^{r}(f,\xi_{1:i-1})-\mathbf{v}_{i}^{s}(f,\xi_{1:i-1})}\\
= & \sup_{\mathbf{v}^{r}\in V_{j-1},\mathbf{v}^{s}\in V_{j}}\sup_{\xi\in\{\pm1\}^{n}}\max_{i\in[n]}|\mathbf{v}_{i}^{r}(f,\xi_{1:i-1})-f(\mathbf{z}_{i}(\xi_{1:i-1}))+f(\mathbf{z}_{i}(\xi_{1:i-1}))-\mathbf{v}_{i}^{s}(f,\xi_{1:i-1})|\\
\le & \epsilon_{j-1}+\epsilon_{j}=3\epsilon_{j}.
\end{align*}
For any measurable set $A$ and $\lambda\in\Real$,
\begin{align*}
 & \Expectation\left[\sup_{\mathbf{w}^{(j)}\in W_{j}}\abs{\sum_{i=1}^{n}\xi_{i}\mathbf{w}_{i}^{(j)}(\xi_{1:i-1})}\Indicator A\right]\\
 & =\Probability\left(A\right)\frac{\Expectation\left[\sup_{\mathbf{w}^{(j)}\in W_{j}}\abs{\sum_{i=1}^{n}\xi_{i}\mathbf{w}_{i}^{(j)}(\xi_{1:i-1})}\Indicator A\right]}{\Probability\left(A\right)}\\
 & =\Probability\left(A\right)\frac{\Expectation\left[\log\left\{ \exp\left(\lambda\sup_{\mathbf{w}^{(j)}\in W_{j}}\abs{\sum_{i=1}^{n}\xi_{i}\mathbf{w}_{i}^{(j)}(\xi_{1:i-1})}\right)\right\} \Indicator A\right]}{\lambda\Probability\left(A\right)}\\
 & \le\frac{\Probability\left(A\right)}{\lambda}\log\left(\frac{\Expectation\left[\exp\left(\lambda\sup_{\mathbf{w}^{(j)}\in W_{j}}\abs{\sum_{i=1}^{n}\xi_{i}\mathbf{w}_{i}^{(j)}(\xi_{1:i-1})}\right)\Indicator A\right]}{\Probability\left(A\right)}\right)\\
 & \le\frac{\Probability\left(A\right)}{\lambda}\log\left(\frac{\Expectation\left[\exp\left(\lambda\sup_{\mathbf{w}^{(j)}\in W_{j}}\abs{\sum_{i=1}^{n}\xi_{i}\mathbf{w}_{i}^{(j)}(\xi_{1:i-1})}\right)\right]}{\Probability\left(A\right)}\right),
\end{align*}
where the first inequality holds by Jensen's inequality (Note that
$\Expectation[\cdot\Indicator A]/\Probability(A)=\Expectation[\cdot|A]$
defines a conditional distribution). Let us write the covering number
$N(\epsilon_{j}):=|V_{j}|$. Because
\begin{align*}
 & \Expectation\left[\exp\left(\lambda\sup_{\mathbf{w}^{(j)}\in W_{j}}\abs{\sum_{i=1}^{n}\xi_{i}\mathbf{w}_{i}^{(j)}(\xi_{1:i-1})}\right)\right]\\
 & \le\Expectation\left[\sum_{\mathbf{w}^{(j)}\in W_{j}}\exp\left(\lambda\abs{\sum_{i=1}^{n}\xi_{i}\mathbf{w}_{i}^{(j)}(\xi_{1:i-1})}\right)\right]\\
 & \le\Expectation\left[\sum_{\mathbf{w}^{(j)}\in W_{j}}\exp\left(\lambda\sum_{i=1}^{n}\xi_{i}\mathbf{w}_{i}^{(j)}(\xi_{1:i-1})\right)+\exp\left(-\lambda\sum_{i=1}^{n}\xi_{i}\mathbf{w}_{i}^{(j)}(\xi_{1:i-1})\right)\right]\\
 & \le|W_{j}|\exp\left(\frac{3\lambda^{2}\epsilon_{j}n}{2}\right)\\
 & \le N(\epsilon_{j})^{2}\exp\left(\frac{3\lambda^{2}\epsilon_{j}^{2}n}{2}\right),
\end{align*}
we obtain
\begin{align*}
\Expectation\left[\sup_{\mathbf{w}^{(j)}\in W_{j}}\abs{\sum_{i=1}^{n}\xi_{i}\mathbf{w}_{i}^{(j)}(\xi_{1:i-1})}\Indicator A\right]\le & \frac{\Probability\left(A\right)}{\lambda}\left(\frac{3\lambda^{2}\epsilon_{j}^{2}n}{2}+\log\frac{N(\epsilon_{j})^{2}}{\Probability(A)}\right)
\end{align*}
Setting $\lambda=\epsilon_{j}^{-1}\sqrt{2\log(N_{j}^{2}/\Probability(A))/(3n)}$,
\begin{align*}
\Expectation\left[\sup_{\mathbf{w}^{(j)}\in W_{j}}\abs{\sum_{i=1}^{n}\xi_{i}\mathbf{w}_{i}^{(j)}(\xi_{1:i-1})}\Indicator A\right]\le & 2\Probability\left(A\right)\epsilon_{j}\sqrt{3n\log\frac{N(\epsilon_{j})}{\sqrt{\Probability(A)}}}
\end{align*}
Summing up over $j\in[J]$,
\begin{align*}
\Expectation\left[\sum_{j=1}^{J}\sup_{\mathbf{w}^{(j)}\in W_{j}}\abs{\sum_{i=1}^{n}\xi_{i}\mathbf{w}_{i}^{(j)}(\xi_{1:i-1})}\Indicator A\right]\le & 2\Probability\left(A\right)\sum_{j=1}^{J}\epsilon_{j}\sqrt{3n\log\frac{N(\epsilon_{j})}{\sqrt{\Probability(A)}}}\\
= & 2\Probability\left(A\right)\sum_{j=1}^{J}\frac{\epsilon_{j}}{\epsilon_{j}-\epsilon_{j+1}}\int_{\epsilon_{j+1}}^{\epsilon_{j}}\sqrt{3n\log\frac{N(\epsilon_{j})}{\sqrt{\Probability(A)}}}d\epsilon\\
= & 4\Probability\left(A\right)\sum_{j=1}^{J}\int_{\epsilon_{j+1}}^{\epsilon_{j}}\sqrt{3n\log\frac{N(\epsilon_{j})}{\sqrt{\Probability(A)}}}d\epsilon.
\end{align*}
Because $N(\epsilon)$ is nonincreasing in $\epsilon$,
\begin{align*}
\Expectation\left[\sum_{j=1}^{J}\sup_{\mathbf{w}^{(j)}\in W_{j}}\abs{\sum_{i=1}^{n}\xi_{i}\mathbf{w}_{i}^{(j)}(\xi_{1:i-1})}\Indicator A\right]\le & 4\Probability\left(A\right)\sum_{j=1}^{J}\int_{\epsilon_{j+1}}^{\epsilon_{j}}\sqrt{3n\log\frac{N(\epsilon)}{\sqrt{\Probability(A)}}}d\epsilon\\
= & 4\Probability\left(A\right)\int_{\epsilon_{J+1}}^{1/2}\sqrt{3n\log\frac{N(\epsilon)}{\sqrt{\Probability(A)}}}d\epsilon.
\end{align*}
From \eqref{eq:entropy_bound_1}, 
\begin{align*}
\Expectation\left[\sup_{f\in\mF}\abs{\sum_{i=1}^{n}\xi_{i}f(\mathbf{z}_{i}(\xi_{1:i-1}))}\Indicator A\right]\le & \Probability\left(A\right)\left(n\epsilon_{J}+4\int_{\epsilon_{J+1}}^{1/2}\sqrt{3n\log\frac{N(\epsilon)}{\sqrt{\Probability(A)}}}d\epsilon\right)\\
= & 2\Probability\left(A\right)\left(n\epsilon_{J+1}+2\int_{\epsilon_{J+1}}^{1/2}\sqrt{3n\log\frac{N(\epsilon)}{\sqrt{\Probability(A)}}}d\epsilon\right).
\end{align*}
For any $a>0$, let $\mE_{t}(a):=\{\sup_{f\in\mF}\abs{\sum_{i=1}^{n}\xi_{i}f(\mathbf{z}_{i}(\xi_{1:i-1}))}>a\}$.
Then, by Markov inequality,
\begin{align*}
\Probability\left(\mE_{t}(a)\right)\le & \frac{1}{a}\Expectation\left[\sup_{f\in\mF}\abs{\sum_{i=1}^{n}\xi_{i}f(\mathbf{z}_{i}(\xi_{1:i-1}))}\Indicator{\mE_{t}(a)}\right]\\
\le & \frac{2}{a}\Probability\left(\mE_{t}(a)\right)\left(n\epsilon_{J+1}+2\int_{\epsilon_{J+1}}^{1/2}\sqrt{3n\log\frac{N(\epsilon)}{\sqrt{\Probability(\mE_{t}(a))}}}d\epsilon\right)
\end{align*}
Canceling out the probability terms,
\begin{align*}
a\le & 2\left(n\epsilon_{J+1}+2\int_{\epsilon_{J+1}}^{1/2}\sqrt{3n\log\frac{N(\epsilon)}{\sqrt{\Probability(\mE_{t}(a))}}}d\epsilon\right)
\end{align*}
 Setting 
\[
a=2\left(n\epsilon_{J+1}+2\int_{\epsilon_{J+1}}^{1/2}\sqrt{3n\log\frac{N(\epsilon)}{\sqrt{\delta}}}d\epsilon\right)
\]
gives
\[
\Probability\left(\sup_{f\in\mF}\abs{\sum_{i=1}^{n}\xi_{i}f(\mathbf{z}_{i}(\xi_{1:i-1}))}>a\right)=\Probability\left(\mE_{t}(a)\right)\le\delta.
\]
Setting suitable $J\in\mathbb{N}$ proves the result. 
\end{proof}

\subsection{Probabilistic Inequalities}

\begin{lem}
\label{lem:exp_bound} (Exponential martingale inequality) If a martingale $(X_{t};t\ge0)$, adapted to filtration $\mF_{t}$, satisfies $\CE{\exp(\lambda X_{t})}{\mF_{t-1}}\le\exp(\lambda^{2}\sigma_{t}^{2}/2)$
for some constant $\sigma_{t}$, for all $t$, then for any $a\ge0$,
\[
\Probability\left(\abs{X_{T}-X_{0}}\ge a\right)\le2\exp\left(-\frac{a^{2}}{2\sum_{t=1}^{T}\sigma_{t}^{2}}\right)
\]
Thus, with probability at least $1-\delta$,
\[
\abs{X_{T}-X_{0}}\le\sqrt{2\sum_{t=1}^{T}\sigma_{t}^{2}\log\frac{2}{\delta}}.
\]
\end{lem}

\begin{lem}
\label{lem:Azuma_Bernstein} (Azuma-Bernstein inequality) Let $\{X_{s}\}_{s\ge1}$
denote the martingale difference adapted to the filtration $\{\mF_{s}\}_{s\ge0}$
such that $\CE{X_{s}}{\mF_{s-1}}=0$. Suppose that $|X_{s}|\le M$
almost surely for $s\ge1$. Then with probability at least $1-\delta$,
\[
\sum_{s=1}^{n}X_{s}\le\frac{2}{3}M\log\frac{1}{\delta}+\sqrt{2\sum_{s=1}^{n}\CE{X_{s}^{2}}{\mF_{s-1}}\log\frac{1}{\delta}}.
\]
\end{lem}

\section{MISSING PROOFS}
In this section, we provide complete proofs omitted in the manuscript.

\subsection{Proof of Proposition~\ref{prop:linearity}}
\begin{proof}
For $h\in[H]$ and $(x,a)\in \mX \times \mA$,
\begin{align*}
[\Probability_{h}V_{h}^{\pi}](x,a):= & \int_{\mX}\Value h{\pi}{x^{\prime}}\Feature xa^{\top}\psi(x^{\prime})dx^{\prime}\\
= & \Feature xa^{\top}\left\{ \int_{\mX}\Value h{\pi}{x^{\prime}}\psi(x^{\prime})dx^{\prime}\right\} .
\end{align*}
Setting $w_{h}^{\pi}:=\int_{\mX}\Value h{\pi}{x^{\prime}}\psi(x^{\prime})dx^{\prime}$
proves the result.
\end{proof}

\subsection{Proof of Theorem~\ref{thm:lower_bound}}
\label{subsec:lower_bound_proof} 
\begin{proof}
Without loss of generality, suppose $s_{\star}$ and $s=s_{\star}/2$
are even. 
Let $e_{i}\in\Real^{s_{\star}}$ denote the $i$-th Euclidean basis. 
We set the action space $\mA:=[s_{\star}]^{s}$ and the state
space $\mX:=\{-1,1\}^{d}\times\{x_{0},x_{g},x_{b}\}$. Since $d>s_{\star}^{2}$,
we can define $\psi(x_{1,1},\ldots,x_{|\mS_{\star}|,2s},x_{|\mS_{\star}|^{2}+1},\ldots,x_{d},x_{0}):=\mathbf{0}$.
Let us write $x_{1:d}:=(x_{1,1},\ldots,x_{|\mS_{\star}|,2s},x_{|\mS_{\star}|^{2}+1},\ldots,x_{d})$.
For $\sigma^{2}\in(0,1]$, given $(i_{1},\ldots,i_{s})\in[s_{\star}]^{s}$
set 
\begin{align*}
\psi(x_{1:d},x_{g}|i_{1},\ldots,i_{s})^{\top} & :=\frac{1}{\sigma s}(\underset{s\times s_{\star}}{\underbrace{x_{i_{1},1}e_{i_{1}}^{\top},\ldots,x_{i_{s},s}e_{i_{s}}^{\top}}},0,\ldots,0),\\
\psi(x_{1:d},x_{b}|i_{1},\ldots,i_{s})^{\top} & :=\frac{1}{\sigma s}(\underset{s\times s_{\star}}{\underbrace{x_{j_{1},1}e_{j_{1}}^{\top},\ldots,x_{j_{s},s}e_{j_{s}}^{\top}}},0,\ldots,0),
\end{align*}
where $j_{v}\in[s_{\star}]\setminus\{i_{v}\}$ for each $v\in[s]$.
Note that $\psi(x)_{\mS_{\star}^{c}}=\mathbf{0}$ for $\mS_{\star}:=\{i_{1},j_{1},s_{\star}+i_{2},s_{\star}+j_{2},\ldots,(s-1)s_{\star}+i_{s},(s-1)s_{\star}+j_{s}\}$.
Here, $x_{1:d}$ are sampled independently from $d$ Bernouill distributions
over $\{\pm1\}$. For each $v\in[s]$, let $z_{v}:=\sum_{u\in[s_{\star}]\setminus\{i_{v}\}}x_{u,v}e_{u}$.
Note that $z_{v}^{\top}e_{i_{v}}=0$ and $z_{v}^{\top}e_{j_{v}}=x_{j_{v}}$.
Further, for each action $a=(a(1),\ldots,a(s))\in\mA$, let $y_{v}(a):=\sum_{u\in[|\mS_{\star}|]\setminus\{i_{v},a(v)\}}x_{u,v}e_{u}+\Indicator{a(v)\neq i_{v}}x_{a(v),v}e_{a(v)}$.
we construct the feature vectors 
\begin{align*}
\Feature{(x_{1:d},x_{0})}A^{\top} & :=\begin{cases}
\frac{\sigma}{2}(\underset{s\times s_{\star}}{\underbrace{x_{i_{1},1}e_{i_{1}}^{\top}+z_{1}^{\top},\ldots,x_{i_{s},s}e_{i_{s}}^{\top}+z_{s}^{\top}}},x_{s_{\star}^{2}+1},\ldots,x_{d}) & a=(i_{1},\ldots,i_{s})\\
\sigma(\underset{s\times s_{\star}}{\underbrace{v_{1}x_{i_{1},1}e_{i_{1}}^{\top}+y_{1}(a)^{\top},\ldots,v_{s}x_{i_{s},s}e_{i_{s}}^{\top}+y_{s}(a)^{\top}}},x_{s_{\star}^{2}+1},\ldots,x_{d}) & a\neq(i_{1},\ldots,i_{s})
\end{cases}\\
\Feature{(x_{1:d},x_{g})}A^{\top} & =\sigma(\underset{s\times s_{\star}}{\underbrace{x_{i_{1},1}e_{i_{1}}^{\top},\ldots,x_{i_{s},s}e_{i_{s}}^{\top}}},x_{s_{\star}^{2}+1},\ldots,x_{d})\\
\Feature{(x_{1:d},x_{b})}A^{\top} & =\sigma\left(z_{1}^{\top}+v_{1}x_{i_{1},1}e_{i_{1}}^{\top},\ldots,z_{s}^{\top}+v_{s}x_{i_{s},s}e_{i_{s}}^{\top},x_{s_{\star}^{2}+1},\ldots,x_{d}\right),
\end{align*}
where $v_{1},\ldots,v_{s}\in\{\pm1\}$ satisfies $v_{1}+\cdots+v_{s}=0$.
The condition $\sigma^{2}\le1$ ensures $\|\phi(x,a)\|_{\infty}\le1$.
Under this construction, the transition probability is 
\begin{align*}
\CP{(x_{1:d}x_{g})}{(x_{1:d},x_{0}),a}=\Feature{(x_{1:d},x_{0})}a^{\top}\psi(x_{1:d},x_{g};i_{1},\ldots,i_{s}) & =\begin{cases}
\frac{1}{2} & a=(i_{1},\ldots,i_{s})\\
0 & a\neq(i_{1},\ldots,i_{s})
\end{cases},\\
\CP{(x_{1:d},x_{b})}{(x_{1:d},x_{0}),a}=\Feature{(x_{1:d},x_{0})}a^{\top}\psi(x_{1:d},x_{b};i_{1},\ldots,i_{s}) & =\begin{cases}
\frac{1}{2} & a=(i_{1},\ldots,i_{s})\\
1 & a\neq(i_{1},\ldots,i_{s})
\end{cases},\\
\CP{(x_{1:d},x_{b})}{(x_{1:d},x_{b}),a}=\Feature{(x_{1:d},x_{b})}a^{\top}\psi(x_{1:d},x_{b};i_{1},\ldots,i_{s}) & =1,\\
\CP{(x_{1:d},x_{g})}{(x_{1:d},x_{g}),a}=\Feature{(x_{1:d},x_{g})}a^{\top}\psi(x_{1:d},x_{g};i_{1},\ldots,i_{s}) & =1.
\end{align*}
The construction fixes $x_{1:d}$ and the good $(x_{g})$ or bad state $(x_{b})$ after the choice of the first step. 
To evaluate the restrictive minimum eigenvalue, 
\begin{align*}
\Sigma^{\pi^{U}} & =\Expectation^{\pi^{U}}\left[\frac{1}{H\abs{\mA}}\sum_{h=1}^{H}\Feature{X_{h}}{a_{h}}\Feature{X_{h}}{a_{h}}^{\top}\right]\\
 & =\Expectation^{\pi^{U}}\left[\frac{1}{H\abs{\mA}}\sum_{h=1}^{H}\sum_{a\in\mA}\Feature{X_{h}}a\Feature{X_{h}}a^{\top}\right]\\
 & \succeq\Expectation^{\pi^{U}}\left[\frac{1}{H\abs{\mA}}\sum_{h=2}^{H}\sum_{a\in\mA}\Feature{X_{h}}a\Feature{X_{h}}a^{\top}\right]\\
 & =\Expectation^{\pi^{U}}\left[\CE{\frac{1}{H\abs{\mA}}\sum_{h=2}^{H}\sum_{a\in\mA}\Feature{X_{h}}a\Feature{X_{h}}a^{\top}}{a_{1}\neq(i_{1},\ldots,i_{s})}\Probability\left(a_{1}\neq(i_{1},\ldots,i_{s})\right)\right]\\
 & =\Expectation^{\pi^{U}}\left[\CE{\frac{1}{H\abs{\mA}}\sum_{h=2}^{H}\sum_{a\in\mA}\Feature{X_{h}}a\Feature{X_{h}}a^{\top}}{X_{2}(d+1)=\cdots=X_{H}(d+1)=x_{b}}\frac{\abs{\mA}-1}{\abs{\mA}}\right]
\end{align*}
Because $x_{1:d}$ are independent random variables such that $\Expectation[x_{i}x_{j}]=0$
and $\Expectation[x_{i}^{2}]=1$, we obtain $\Expectation\phi(X_{h},a)\Feature{X_{h}}a^{\top}=\sigma I_{d}$
for $X_{h}(d+1)=x_{b}$. Thus, 
\begin{align*}
\Sigma^{\pi^{U}}\succeq & \frac{\abs{\mA}-1}{H\abs{\mA}^{2}}\sum_{h=2}^{H}\abs{\mA}\sigma^{2}I_{d}\\
\succeq & \frac{H-1}{H}\frac{\abs{\mA}-1}{\abs{\mA}}\sigma^{2}I_{d}\\
\succeq & \frac{\sigma^{2}}{4}I_{d}.
\end{align*}
Thus we obtain $\sigma_{\min}(i_{1},\ldots,i_{s}):=\sigma_{\min(i_{1},\ldots,i_{s})}(\Sigma^{\pi_{U}},s_{\star})\ge\sigma^{2}/4$.
Define $y:=\min\{5/s_{\star},\sqrt{s_{\star}d/N},1/(\sigma^{2}\sqrt{N})\}\in[0,1]$
and set the rewards for good state $\Reward{(x_{1:d},x_{g})}a=y$
and for the bad states $\Reward{(x_{1:d},x_{b})}a=0$ for all $a\in\mA$.
For the initial state we set $\Reward{(x_{1:d},x_{0})}{(i_{1},\ldots,i_{s})}=y/2$
and $\Reward{(x_{1:d},x_{0})}{(j_{1},\ldots,j_{s})}=y/2$. For $a\neq(i_{1},\ldots,i_{s})$
and $a\neq(j_{1},\ldots,j_{s})$ we set, $\Reward{(x_{1:d},x_{0})}a=0$.
Because the optimal policy gains expected reward $HyN/2$, for any
$(i_{1},\ldots,i_{s})\in[s_{\star}]^{s}$ and any algorithm $\widehat{A}$
which generates the policy $\widehat{\pi}^{(n)}$ that selects $\widehat{a}_{1}^{(n)},\ldots\widehat{a}_{H}^{(n)}$,
the expected regret is 
\begin{align*}
\Expectation_{(i_{1},\ldots,i_{s})}\left[\text{R}(N,\widehat{A})\right]= & \frac{Hy}{2}N-\sum_{n=1}^{N}\Expectation_{(i_{1},\ldots,i_{s})}\left[\Value 1{\widehat{\pi}^{(n)}}{X_{1}}\right].
\end{align*}
By construction of the SMDP, 
\begin{align*}
 & \Expectation_{(i_{1},\ldots,i_{s})}\left[\Value 1{\widehat{\pi}^{(n)}}{X_{1}}\right]\\
 & =\frac{Hy}{2}\Expectation_{(i_{1},\ldots,i_{s})}\left[\Indicator{\widehat{a}_{1}^{(n)}=(i_{1},\ldots,i_{s})}\right]+\frac{y}{2}\Expectation_{(i_{1},\ldots,i_{s})}\left[\Indicator{\widehat{a}_{1}^{(n)}=(j_{1},\ldots,j_{s})}\right]\\
 & \le\frac{Hy}{2}\Expectation_{(i_{1},\ldots,i_{s})}\left[\Indicator{\widehat{a}_{1}^{(n)}=(i_{1},\ldots,i_{s})\cup\widehat{a}_{1}^{(n)}=(j_{1},\ldots,j_{s})}\right].
\end{align*}
Let $M_{(i_{1},\ldots,i_{s})}^{(N)}=\sum_{n=1}^{N}\mathbb{I}(\widehat{a}_{1}^{(n)}=(i_{1},\ldots,i_{s})\cup\widehat{a}_{1}^{(n)}=(j_{1},\ldots,j_{s}))$.
Note that $0\le M_{(i_{1},\ldots,i_{s})}^{(N)}\le N$, almost surely.
For each $v\in[s]$, by Pinsker's inequality, 
\begin{equation}
\begin{split} & \Expectation_{(i_{1},\ldots,i_{s})}\left[M_{(i_{1},\ldots,i_{s})}^{(N)}\right]\\
 & \le\Expectation_{(i_{1},\ldots,i_{v-1}0,i_{v+1}\ldots,i_{s})}\left[M_{(i_{1},\ldots,i_{s})}^{(N)}\right]+N\sqrt{\frac{1}{2}D(\Probability_{(i_{1},\ldots,i_{v-1}0,i_{v+1}\ldots,i_{p})},\Probability_{(i_{1},\ldots,i_{p})})},
\end{split}
\label{eq:lower_pinsker}
\end{equation}
where the distribution of $\Expectation_{(i_{1},\ldots,i_{v-1}0,i_{v+1}\ldots,i_{s})}$
is constructed by modifying 
\begin{align*}
\psi(x_{1:d},x_{g}|i_{1},\ldots,i_{v-1}0,i_{v+1}\ldots,i_{s})^{\top} & :=\frac{1}{(s-1)}(e_{i_{1}}^{\top},\ldots,e_{i_{v-1}}^{\top},0^{\top},e_{i_{v+1}}^{\top},\ldots,e_{i_{s}}^{\top},0^{\top},\ldots,0^{\top})\cdot\frac{1}{2^{d}}\\
\psi(x_{1:d},x_{b}|i_{1},\ldots,i_{v-1}0,i_{v+1}\ldots,i_{s})^{\top} & :=\frac{1}{(s-1)}(e_{j_{1}}^{\top},\ldots,e_{j_{v-1}}^{\top},0^{\top},e_{j_{v+1}}^{\top}\ldots,,e_{j_{s}}^{\top},0^{\top},\ldots,0^{\top})\cdot\frac{1}{2^{d}}
\end{align*}
and for $\Delta\in(0,1/4)$ to be determined later 
\begin{align*}
 & \Feature{(x_{1:d},x_{0})}a^{\top}\\
 & :=\begin{cases}
\left(\frac{1}{2}+\Delta\right)\left(e_{i_{1}}^{\top},\ldots,e_{i_{v-1}}^{\top},0^{\top},e_{i_{v+1}}^{\top},\ldots,e_{i_{s}}^{\top},0^{\top},\ldots,0^{\top}\right)\\
+\left(\frac{1}{2}-\Delta\right)\left(e_{j_{1}}^{\top},\ldots,e_{j_{v-1}}^{\top},0^{\top},e_{j_{v+1}}^{\top}\ldots,,x_{j_{s}}e_{j_{s}}^{\top},0^{\top},\ldots,0^{\top}\right) & a=(i_{1},\ldots,i_{s})\\
\left(e_{j_{1}}^{\top},\ldots,e_{j_{v-1}}^{\top},0^{\top},e_{j_{v+1}}^{\top}\ldots,,x_{j_{s}}e_{j_{s}}^{\top},0^{\top},\ldots,0^{\top}\right) & a\neq(i_{1},\ldots,i_{s})
\end{cases}.
\end{align*}
This construction modifies the distribution when $a=(i_{1},\ldots,i_{s})$
\begin{align*}
\CP{X_{2}(d+1)=x_{g}}{x_{0},a} & =\frac{1}{2}+\Delta,\\
\CP{X_{2}(d+1)=x_{b}}{x_{0},a} & =\frac{1}{2}-\Delta.
\end{align*}
Other feature vectors are constructed as: 
\begin{align*}
\Feature{(x_{1:d},x_{g})}a^{\top} & =\left(e_{i_{1}}^{\top},\ldots,e_{i_{v-1}}^{\top},0^{\top},e_{i_{v+1}}^{\top},\ldots,e_{i_{s}}^{\top},0^{\top},\ldots,0^{\top}\right),\\
\Feature{(x_{1:d},x_{b})}a^{\top} & =\left(e_{j_{1}}^{\top},\ldots,e_{j_{v-1}}^{\top},0^{\top},e_{j_{v+1}}^{\top}\ldots,,x_{j_{s}}e_{j_{s}}^{\top},0^{\top},\ldots,0^{\top}\right),
\end{align*}
for all $a\in\mA$. Then the distribution of $\Probability_{(i_{1},\ldots,i_{s})}$
and $\Probability_{(i_{1},\ldots,i_{v-1}0,i_{v+1}\ldots,i_{s})}$
only differs when the action of the first step $\widehat{a}_{1}^{(n)}=(i_{1},\ldots,i_{s})$
(Note that this problem does not count in hard instances and its RME
can be zero). Let $D(\Probability_{1},\Probability_{2})$ denote the
relative entropy between probability measures $\Probability_{1}$
and $\Probability_{2}$ and $\Probability_{(i_{1},\ldots,i_{s})}(a)$
denote the distribution of states when $a_{1}=a$. By the divergence
decomposition (Lemma 15.1 in \citet{lattimore2020bandit}), 
\begin{align*}
 & D(\Probability_{(i_{1},\ldots,i_{v-1}0,i_{v+1}\ldots,i_{s})},\Probability_{(i_{1},\ldots,i_{s})})\\
 & =\sum_{a\in\mA}\Expectation_{(i_{1},\ldots,i_{v-1}0,i_{v+1}\ldots,i_{s})}\left[\sum_{n=1}^{N}\Indicator{\widehat{a}_{1}^{(n)}=a}\right]D(\Probability_{(i_{1},\ldots,i_{v-1}0,i_{v+1}\ldots,i_{s})}(a),\Probability_{(i_{1},\ldots,i_{s})}(a))\\
 & =\Expectation_{(i_{1},\ldots,i_{v-1}0,i_{v+1}\ldots,i_{s})}\left[\sum_{n=1}^{N}\Indicator{\widehat{a}_{1}^{(n)}=(i_{1},\ldots,i_{s})}\right]D(\Probability_{(i_{1},\ldots,i_{v-1}0,i_{v+1}\ldots,i_{s})}((i_{1},\ldots,i_{s})),\Probability_{(i_{1},\ldots,i_{s})}((i_{1},\ldots,i_{s})))\\
 & =\Expectation_{(i_{1},\ldots,i_{v-1}0,i_{v+1}\ldots,i_{s})}\left[\sum_{n=1}^{N}\mathbb{I}(\widehat{a}_{1}^{(n)}=(i_{1},\ldots,i_{s}))\right]\left(\left(\frac{1}{2}+\Delta\right)\log\left(1+2\Delta\right)+\left(\frac{1}{2}-\Delta\right)\log\left(1-2\Delta\right)\right)\\
 & =\Expectation_{(i_{1},\ldots,i_{v-1}0,i_{v+1}\ldots,i_{s})}\left[\sum_{n=1}^{N}\mathbb{I}(\widehat{a}_{1}^{(n)}=(i_{1},\ldots,i_{s}))\right]\left(\frac{\log\left(1-4\Delta^{2}\right)}{2}+\Delta\log\frac{1+2\Delta}{1-2\Delta}\right)\\
 & \le\Expectation_{(i_{1},\ldots,i_{v-1}0,i_{v+1}\ldots,i_{s})}\left[M_{(i_{1},\ldots,i_{s})}^{(N)}\right]\left(\frac{\log\left(1-4\Delta^{2}\right)}{2}+\Delta\log\frac{1+2\Delta}{1-2\Delta}\right).
\end{align*}
Because $\Delta\le1/4$, we have $\left(\frac{\log\left(1-4\Delta^{2}\right)}{2}+\Delta\log\frac{1+2\Delta}{1-2\Delta}\right)\le4\Delta^{2}$
and 
\[
D(\Probability_{(i_{1},\ldots,i_{v-1}0,i_{v+1}\ldots,i_{s})},\Probability_{(i_{1},\ldots,i_{s})})\le\Expectation_{(i_{1},\ldots,i_{v-1}0,i_{v+1}\ldots,i_{s})}\left[M_{(i_{1},\ldots,i_{s})}^{(N)}\right]\left(4\Delta^{2}\right).
\]
From \eqref{eq:lower_pinsker} 
\begin{align*}
 & \Expectation_{(i_{1},\ldots,i_{s})}\left[M_{(i_{1},\ldots,i_{s})}^{(N)}\right]\\
 & \le\Expectation_{(i_{1},\ldots,i_{v-1}0,i_{v+1}\ldots,i_{s})}\left[M_{(i_{1},\ldots,i_{s})}^{(N)}\right]+N\sqrt{\frac{1}{2}D(\Probability_{(i_{1},\ldots,i_{v-1}0,i_{v+1}\ldots,i_{s})},\Probability_{(i_{1},\ldots,i_{s})})}\\
 & \le\Expectation_{(i_{1},\ldots,i_{v-1}0,i_{v+1}\ldots,i_{s})}\left[M_{(i_{1},\ldots,i_{s})}^{(N)}\right]+N\Delta\sqrt{2\Expectation_{(i_{1},\ldots,i_{v-1}0,i_{v+1}\ldots,i_{s})}\left[M_{(i_{1},\ldots,i_{s})}^{(N)}\right]}
\end{align*}
From the regret, we obtain 
\begin{align*}
\Expectation_{(i_{1},\ldots,i_{s})}\left[\text{R}(N,\widehat{A})\right]= & \frac{Hy}{2}N-\sum_{n=1}^{N}\Expectation_{(i_{1},\ldots,i_{s})}\left[\Value 1{\widehat{\pi}^{(n)}}{X_{1}}\right]\\
\ge & \frac{Hy}{2}N-\frac{Hy}{2}\Expectation_{(i_{1},\ldots,i_{s})}\left[M_{(i_{1},\ldots,i_{s})}^{(N)}\right]\\
\ge & \frac{Hy}{2}\left(N-\Expectation_{(i_{1},\ldots,i_{v-1}0,i_{v+1}\ldots,i_{s})}\left[M_{(i_{1},\ldots,i_{s})}^{(N)}\right]-N\Delta\sqrt{2\Expectation_{(i_{1},\ldots,i_{v-1}0,i_{v+1}\ldots,i_{s})}\left[M_{(i_{1},\ldots,i_{s})}^{(N)}\right]}\right).
\end{align*}
Taking supremum over $(i_{1},\ldots,i_{s})\in[s_{\star}]^{s}$, 
\begin{align*}
\sup_{(i_{1},\ldots,i_{s})}\Expectation_{(i_{1},\ldots,i_{s})}\left[\text{R}(N,\widehat{A})\right]\ge & \frac{1}{s_{\star}^{s}}\sum_{(i_{1},\ldots,i_{s})}\Expectation_{(i_{1},\ldots,i_{s})}\left[\text{R}(N,\widehat{A})\right]\\
= & \frac{1}{s_{\star}^{s-1}}\sum_{v=1}^{s}\sum_{(i_{1},\ldots,i_{v-1},i_{v+1},\ldots,i_{s})}\frac{1}{s_{\star}}\sum_{i_{v}=1}^{s_{\star}}\Expectation_{(i_{1},\ldots,i_{s})}\left[\text{R}(N,\widehat{A})\right].
\end{align*}
Taking the average over $i_{v}\in[s_{\star}]$, 
\begin{align*}
 & \frac{1}{s_{\star}}\sum_{i_{v}=1}^{s_{\star}}\Expectation_{(i_{1},\ldots,i_{s})}\left[\text{R}(N,\widehat{A})\right]\\
 & \ge\frac{Hy}{2}\frac{1}{s_{\star}}\sum_{i_{v}=1}^{s_{\star}}\left(N-\Expectation_{(i_{1},\ldots,i_{v-1}0,i_{v+1}\ldots,i_{s})}\left[M_{(i_{1},\ldots,i_{s})}^{(N)}\right]-N\Delta\sqrt{2\Expectation_{(i_{1},\ldots,i_{v-1}0,i_{v+1}\ldots,i_{s})}\left[M_{(i_{1},\ldots,i_{s})}^{(N)}\right]}\right)\\
 & \ge\frac{Hy}{2}\left(N-\frac{1}{s_{\star}}\Expectation_{(i_{1},\ldots,i_{v-1}0,i_{v+1}\ldots,i_{s})}\left[\sum_{i_{v}=1}^{s_{\star}}M_{(i_{1},\ldots,i_{s})}^{(N)}\right]-\frac{N\Delta}{s_{\star}}\sqrt{2s_{\star}\Expectation_{(i_{1},\ldots,i_{v-1}0,i_{v+1}\ldots,i_{s})}\left[\sum_{i_{v}=1}^{s_{\star}}M_{(i_{1},\ldots,i_{s})}^{(N)}\right]}\right)\\
 & \ge\frac{Hy}{2}\left(N-\frac{N}{s_{\star}}-N\Delta\sqrt{\frac{2N}{s_{\star}}}\right)\\
 & \ge\frac{HyN}{2}\left(\frac{4}{5}-\Delta\sqrt{\frac{2N}{s_{\star}}}\right),
\end{align*}
where the third inequality holds by $s_{\star}\ge5$. Setting $\Delta=(1/5)\sqrt{s_{\star}/N}\in(0,1/4)$
gives 
\[
\frac{1}{s_{\star}}\sum_{i_{v}=1}^{s_{\star}}\Expectation_{(i_{1},\ldots,i_{s})}\left[\text{R}(N,\widehat{A})\right]\ge\frac{HyN}{5}.
\]
Thus, 
\begin{align*}
\sup_{(i_{1},\ldots,i_{s})}\Expectation_{(i_{1},\ldots,i_{s})}\left[\text{R}(N,\widehat{A})\right]\ge & \frac{1}{s_{\star}^{s-1}}\sum_{v=1}^{s}\sum_{(i_{1},\ldots,i_{v-1},i_{v+1},i_{s})}\frac{HyN}{5}\\
= & \frac{s_{\star}HyN}{5}\\
= & H\min\{N,\frac{\sqrt{s_{\star}dN}}{5},\frac{s_{\star}\sqrt{N}}{5\sigma^{2}}\}
\end{align*}
We conclude that there exists $(\tilde{i}_{1},\ldots,\tilde{i}_{s})\in[s_{\star}]^{s}$
such that 
\begin{align*}
\Expectation_{(\tilde{i}_{1},\ldots,\tilde{i}_{s})}\left[\text{R}(N,\widehat{A})\right]\ge & H\min\{N,\frac{\sqrt{s_{\star}dN}}{5},\frac{s_{\star}\sqrt{N}}{5\sigma^{2}}\}\\
\ge & H\min\{N,\frac{\sqrt{s_{\star}dN}}{5},\frac{s_{\star}\sqrt{N}}{20\sigma_{\min}(\tilde{i}_{1},\ldots,\tilde{i}_{s})}\},
\end{align*}
where the last inequality holds by $\sigma_{\min}(\tilde{i}_{1},\ldots,\tilde{i}_{s})\ge\sigma^{2}/4$. 
\end{proof}

\subsection{Proof of Theorem~\ref{thm:est_tail}}
\label{subsec:tail_proof}
\begin{proof}
By Lemma \ref{lem:lasso}, it is sufficient to prove the following
inequalities
\begin{align}
\norm{\frac{1}{nH\abs{\mA}}\sum_{\tau=1}^{n}\sum_{k=1}^{H}\sum_{a\in\mA}\Feature{\State k{\tau}}a\Feature{\State k{\tau}}a^{\top}-\Sigma^{\pi_{U}}}_{\infty} & \le\frac{\sigma_{U}}{32\abs{\mS_{\star}}}\label{eq:sigma_inequality}\\
\norm{\sum_{\tau=1}^{n}\sum_{k=1}^{H}\sum_{a\in\mA}\left\{ \tilde{Y}_{\Estimator{h+1}n}^{(\tau)}(a)-\Feature{\State k{\tau}}a^{\top}\Barw hn\right\} \Feature{\State k{\tau}}a}_{\infty} & \le\lambda_{\text{Est}}^{(n)}\label{eq:error_inequality}
\end{align}
To prove \eqref{eq:sigma_inequality}, let $A:=(\Action 11,\ldots,\Action Hn)$
and $\tilde{A}:=(\PseudoAction 11,\ldots,\PseudoAction Hn)$ denote
a collection of actions selected by the policy of RDRLVI and psuedo-actions
selected by Uniform policy $\pi^{U}$, respectively. Let $X:=(\State 11,\ldots,\State HN)$
and $\tilde{X}_{h}:=(\tilde{x}_{1}^{(1)},\ldots\tilde{x}_{H}^{(N)})$
denote a sample path for states under algorithm policy and the uniform
policy $\pi^{U}$. Let $n_{1}:=\min\{n\in\mathbb{N}:n\ge1024\sigma_{U}^{-4}s_{\star}^{4}H^{2}\log\frac{2d^{2}Hn^{2}}{\delta}\}$
and 
\begin{align*}
\mB(X,A) & :=\bigcup_{n=n_{1}}^{N}\left\{ \norm{\frac{1}{nH\abs{\mA}}\sum_{\tau=1}^{n}\sum_{k=1}^{H}\sum_{a\in\mA}\Feature{\State k{\tau}}a\Feature{\State k{\tau}}a^{\top}-\Sigma^{\pi_{U}}}_{\infty}>\frac{\sigma_{U}}{32s_{\star}}\right\} ,\\
B(\tilde{X},\tilde{A}) & :=\bigcup_{n=n_{1}}^{N}\left\{ \norm{\frac{1}{nH}\sum_{\tau=1}^{n}\sum_{k=1}^{H}\Feature{\tilde{x}_{k}^{(\tau)}}{\PseudoAction k{\tau}}\Feature{\tilde{x}_{k}^{(\tau)}}{\PseudoAction k{\tau}}^{\top}-\Sigma^{\pi_{U}}}_{\infty}>\frac{\sigma_{U}}{32s_{\star}}\right\} .
\end{align*}
Note the algorithm restricts the event on $A=\tilde{A}$. For $Z:=(z_{1}^{(1)},\ldots,z_{H}^{(n)})\in\mX^{HN}$,
\begin{align*}
 & \Probability\left(\mB(X,A)\cap\left\{ A=\tilde{A}\right\} \right)\\
 & =\int_{\mX^{NH}\times\mA^{NH}\times\mA^{NH}}\Indicator{\mB(X,A)}\Indicator{A=\tilde{A}}d\Probability(X,A,\tilde{A})\\
 & =\int_{\mX^{N}}\int_{\mX^{(H-1)N}}\int_{\mA^{NH}}\int_{\mA^{NH}}\Indicator{\mB(X,A)}\Indicator{A=\tilde{A}}\prod_{n=1}^{N}\prod_{h=2}^{H+1}\Feature{z_{h-1}^{(n)}}{\Action{h-1}n}^{\top}\psi\left(x_{h}^{(n)}\right)d\Probability(A)d\Probability(\tilde{A})d\Probability(Z)\\
 & =\int_{\mX^{N}}\int_{\mX^{(H-1)N}}\int_{\mA^{NH}}\int_{\mA^{NH}}\Indicator{\mB(X,\tilde{A})}\Indicator{A=\tilde{A}}\prod_{n=1}^{N}\prod_{h=2}^{H+1}\Feature{z_{h-1}^{(n)}}{\PseudoAction{h-1}{\tau}}^{\top}\psi\left(x_{h}^{(n)}\right)d\Probability(A)d\Probability(\tilde{A})d\Probability(Z).
\end{align*}
Because the term
\[
\prod_{n=1}^{N}\prod_{h=2}^{H+1}\Feature{z_{h-1}^{(n)}}{\PseudoAction{h-1}{\tau}}^{\top}\psi\left(x_{h}^{(n)}\right)
\]
is the density function for $\tilde{X}$. we obtain, 
\begin{align*}
 & \Probability\left(\mB(X,A)\cap\left\{ A=\tilde{A}\right\} \right)\\
 & =\int_{\mX^{N}}\int_{\mX^{(H-1)N}}\int_{\mA^{NH}}\int_{\mA^{NH}}\Indicator{\mB(X,\tilde{A})}\Indicator{A=\tilde{A}}\prod_{n=1}^{N}\prod_{h=2}^{H+1}\Feature{z_{h-1}^{(n)}}{\PseudoAction{h-1}{\tau}}^{\top}\psi\left(x_{h}^{(n)}\right)d\Probability(A)d\Probability(\tilde{A})d\Probability(Z)\\
 & \le\int_{\mX^{N}}\int_{\mX^{(H-1)N}}\int_{\mA^{NH}}\int_{\mA^{NH}}\Indicator{\mB(X,\tilde{A})}\prod_{n=1}^{N}\prod_{h=2}^{H+1}\Feature{z_{h-1}^{(n)}}{\PseudoAction{h-1}{\tau}}^{\top}\psi\left(x_{h}^{(n)}\right)d\Probability(A)d\Probability(\tilde{A})d\Probability(Z)\\
 & =\Probability\left(\mB(\tilde{X},\tilde{A})\right).
\end{align*}
For $i,j\in[d]$, let
\[
v_{ij}^{(\tau)}:=\frac{1}{H}\sum_{k=1}^{H}\Feature{\tilde{x}_{k}^{(\tau)}}{\tilde{a}_{k}^{(\tau)}}(i)\Feature{\tilde{x}_{k}^{(\tau)}}{\tilde{a}_{k}^{(\tau)}}^{\top}(j)-\Sigma_{ij}^{\pi_{U}}.
\]
Then $\Expectation\left[v_{ij}^{(\tau)}\right]=0$ and $|v_{ij}^{(\tau)}|\le1$.
Applying Lemma \ref{lem:exp_bound}, with probability at least $1-2\delta/(dn)^{2}$,
\[
\abs{\sum_{\tau=1}^{n}v_{ij}^{(\tau)}}\le\sqrt{2n\log\frac{d^{2}n^{2}}{\delta}}.
\]
Thus, with probability at least $1-2\delta/n^{2}$,
\[
\norm{\frac{1}{nH}\sum_{\tau=1}^{n}\sum_{k=1}^{H}\Feature{\tilde{x}_{k}^{(\tau)}}{\tilde{a}_{k}^{(\tau)}}\Feature{\tilde{x}_{k}^{(\tau)}}{\tilde{a}_{k}^{(\tau)}}^{\top}-\Sigma^{\pi_{U}}}_{\infty}\le\sqrt{\frac{2}{n}\log\frac{d^{2}n^{2}}{\delta}}.
\]
For all $n\ge n_{1}$, we have $n\ge2^{11}\sigma_{U}^{-2}s_{\star}^{2}\log\frac{d^{2}n^{2}}{\delta}$
and
\[
\norm{\frac{1}{nH}\sum_{\tau=1}^{n}\sum_{k=1}^{H}\Feature{\tilde{x}_{k}^{(\tau)}}{\tilde{a}_{k}^{(\tau)}}\Feature{\tilde{x}_{k}^{(\tau)}}{\tilde{a}_{k}^{(\tau)}}^{\top}-\Sigma^{\pi_{U}}}_{\infty}\le\frac{\sigma_{U}}{32s_{\star}}.
\]
Therefore, we obtain,
\[
\Probability\left(\mB(X,A)\cap\left\{ A=\tilde{A}\right\} \right)\le\Probability\left(\mB(\tilde{X},\tilde{A})\right)\le\sum_{n=n_{1}}^{N}\frac{\delta}{n^{2}}\le\delta,
\]
which proves the inequality \eqref{eq:sigma_inequality} for all $n\ge n_{1}$.
Similarly we can prove
\[
\Probability\left(\bigcup_{n=n_{1}}^{N}\left\{ \norm{\frac{1}{nH}\sum_{\tau=1}^{n}\sum_{k=1}^{H}\Feature{\State k{\tau}}{\Action k{\tau}}\Feature{\State k{\tau}}{\Action k{\tau}}^{\top}-\Sigma^{\pi_{U}}}_{\infty}\ge\frac{\sigma_{U}}{32s_{\star}}\right\} \cap\left\{ A=\tilde{A}\right\} \right)\le\Probability\left(\mB(\tilde{X},\tilde{A})\right)\le\delta,
\]
and with probability at least $1-2\delta$,
\begin{equation}
\norm{\frac{1}{nH}\sum_{\tau=1}^{n}\sum_{k=1}^{H}\Feature{\State k{\tau}}{\Action k{\tau}}\Feature{\State k{\tau}}{\Action k{\tau}}^{\top}-\Sigma^{\pi_{U}}}_{\infty}\le\frac{\sigma_{U}}{32s_{\star}},\label{eq:impute_sigma_infty}
\end{equation}
for all $n\ge n_{1}$.

To prove the inequality \eqref{eq:error_inequality}, recall that
for $h\in[H]$ and $n\in[N]$, 
\[
\Barw hn:=\int_{\mX}\Pi_{[0,H]}\left(\max_{a^{\prime}\in\mA}\widehat{Q}_{\Estimator{h+1}n}(x,a^{\prime})\right)\psi(x)dx.
\]
Define 
\begin{equation}
\eta_{w,k}^{(\tau)}(a):=\Pi_{[0,H]}\left(\max_{a^{\prime}\in\mA}\widehat{Q}_{w}(X_{k+1}^{(\tau)}(a),a^{\prime})\right)-\Expectation_{X\sim\CP{\cdot}{\State k{\tau},a}}\left[\Pi_{[0,H]}\left(\max_{a^{\prime}\in\mA}\widehat{Q}_{w}(X,a^{\prime})\right)\right].\label{eq:eta_definition}
\end{equation}
Note that
\begin{align*}
 & \widehat{Y}_{\Estimator{h+1}n,k}^{(\tau)}(\State k{\tau},a)-\Feature{\State k{\tau}}a^{\top}\Barw hn\\
 & =\Pi_{[0,H]}\left(\max_{a^{\prime}\in\mA}\widehat{Q}_{\Estimator{h+1}n}(X_{k+1}^{(\tau)}(a),a^{\prime})\right)-\Feature{\State k{\tau}}a^{\top}\Barw hn\\
 & =\Pi_{[0,H]}\left(\max_{a^{\prime}\in\mA}\widehat{Q}_{\Estimator{h+1}n}(X_{k+1}^{(\tau)}(a),a^{\prime})\right)-\Feature{\State k{\tau}}a^{\top}\left\{ \int_{\mX}\Pi_{[0,H]}\left(\max_{a^{\prime}\in\mA}\widehat{Q}_{\Estimator{h+1}n}(x,a^{\prime})\right)\psi(x)dx\right\} \\
 & =\Pi_{[0,H]}\left(\max_{a^{\prime}\in\mA}\widehat{Q}_{\Estimator{h+1}n}(X_{k+1}^{(\tau)}(a),a^{\prime})\right)-\Expectation_{X\sim\CP{\cdot}{\State k{\tau},a}}\left[\Pi_{[0,H]}\left(\max_{a^{\prime}\in\mA}\widehat{Q}_{\Estimator{h+1}n}(X,a^{\prime})\right)\right]\\
 & :=\eta_{\Estimator{h+1}n,k}^{(\tau)}(\PseudoAction k{\tau})
\end{align*}
and
\begin{align*}
\tilde{\eta}_{\Estimator{h+1}n,k}^{(\tau)}(a):= & \tilde{Y}_{\Estimator{h+1}n}^{(\tau)}(a)-\Feature{\State k{\tau}}a^{\top}\Barw hn\\
= & \frac{\Indicator{\PseudoAction k{\tau}=a}}{\abs{\mA}^{-1}}\eta_{\Estimator{h+1}n,k}^{(\tau)}(a)+\left(1-\frac{\Indicator{\PseudoAction k{\tau}=a}}{\abs{\mA}^{-1}}\right)\Feature{\State k{\tau}}a^{\top}\left(\Impute hn-\Barw hn\right)\\
= & \abs{\mA}\eta_{\Estimator{h+1}n,k}^{(\tau)}(\PseudoAction k{\tau})+\left(1-\frac{\Indicator{\PseudoAction k{\tau}=a}}{\abs{\mA}^{-1}}\right)\Feature{\State k{\tau}}a^{\top}\left(\Impute hn-\Barw hn\right)
\end{align*}
Then the inequality \eqref{eq:error_inequality} becomes
\[
\norm{\sum_{\tau=1}^{n}\sum_{k=1}^{H}\sum_{a\in\mA}\tilde{\eta}_{\Estimator{h+1}n,k}^{(\tau)}(a)\Feature{\State k{\tau}}a}_{\infty}\le\lambda_{\text{Est}}^{(n)}.
\]
The left hand side is decomposed as
\begin{align*}
 & \norm{\sum_{\tau=1}^{n}\sum_{k=1}^{H}\sum_{a\in\mA}\tilde{\eta}_{\Estimator{h+1}n,k}^{(\tau)}(a)\Feature{\State k{\tau}}a}_{\infty}\\
 & \le\abs{\mA}\norm{\sum_{\tau=1}^{n}\sum_{k=1}^{H}\eta_{\Estimator{h+1}{\tau},k}^{(\tau)}(\PseudoAction k{\tau})\Feature{\State k{\tau}}{\PseudoAction k{\tau}}}_{\infty}\\
 &\;+\norm{\sum_{\tau=1}^{n}\sum_{k=1}^{H}\sum_{a\in\mA}\left(1-\frac{\Indicator{\PseudoAction k{\tau}=a}}{\abs{\mA}^{-1}}\right)\Feature{\State k{\tau}}a^{\top}\left(\Impute hn-\Barw hn\right)\Feature{\State k{\tau}}a}_{\infty}\\
 & \le\abs{\mA}\norm{\sum_{\tau=1}^{n}\sum_{k=1}^{H}\eta_{\Estimator{h+1}{\tau},k}^{(\tau)}(\PseudoAction k{\tau})\Feature{\State k{\tau}}{\PseudoAction k{\tau}}}_{\infty}+\norm{\sum_{\tau=1}^{n}\sum_{k=1}^{H}\sum_{a\in\mA}\left(1-\frac{\Indicator{\PseudoAction k{\tau}=a}}{\abs{\mA}^{-1}}\right)\Feature{\State k{\tau}}a}_{\infty}\norm{\Impute hn-\Barw hn}_{1},
\end{align*}
where the last inequality involves $\|\Feature xa\|_{\infty}\le1$.
Since 
\[
\CE{\sum_{a\in\mA}\left(1-\frac{\Indicator{\PseudoAction k{\tau}=a}}{\abs{\mA}^{-1}}\right)\Feature{\State k{\tau}}a}{\State k{\tau}}=0,
\]
we can use Lemma \ref{lem:exp_bound} to obtain with probability at
least $1-2\delta/(Hn^{2})$
\[
\norm{\sum_{\tau=1}^{n}\sum_{k=1}^{H}\sum_{a\in\mA}\left(1-\frac{\Indicator{\PseudoAction k{\tau}=a}}{\abs{\mA}^{-1}}\right)\Feature{\State k{\tau}}a}_{\infty}\le\abs{\mA}\sqrt{nH\log\frac{dHn^{2}}{\delta}}.
\]
Thus it is sufficient to prove
\begin{equation}
\abs{\mA}\norm{\sum_{\tau=1}^{n}\sum_{k=1}^{H}\eta_{\Estimator{h+1}n,k}^{(\tau)}(\PseudoAction k{\tau})\Feature{\State k{\tau}}{\PseudoAction k{\tau}}}_{\infty}+\abs{\mA}\sqrt{nH\log\frac{dHn^{2}}{\delta}}\norm{\Impute hn-\Barw hn}_{1}\le\lambda_{\text{Est}}^{(n)}.\label{eq:error_lambda}
\end{equation}
 We prove \eqref{eq:error_lambda} by inductive arguments. For step
$H$, we have $\Estimator{H+1}n=\mathbf{0}$ and $\widehat{Q}_{H+1}^{\Estimator{H+1}n}(x,a)=0$
for all $(x,a)\in\mX\times\mA$. This implies $\eta_{\Estimator{H+1}n,k}^{(\tau)}(a)=0$
for all $a\in\mA$, and the inequality
\begin{align*}
 & \norm{\sum_{\tau=1}^{n}\sum_{k=1}^{H}\left\{ \widehat{Y}_{\Estimator{H+1}n}^{(\tau)}(\State k{\tau},\Action k{\tau})-\Feature{\State k{\tau}}{\Action k{\tau}}^{\top}\Barw hn\right\} \Feature{\State k{\tau}}{\Action k{\tau}}}_{\infty}\\
 & =\norm{\sum_{\tau=1}^{n}\sum_{k=1}^{H}\eta_{\Estimator{H+1}n}(\Action k{\tau})\Feature{\State k{\tau}}{\Action k{\tau}}}_{\infty}\le\lambda_{\text{Im}}^{(n)}
\end{align*}
holds. By Lemma \ref{lem:lasso} and \eqref{eq:impute_sigma_infty},
we obtain
\begin{align*}
\norm{\Impute Hn-\Barw Hn}_{1}\le & \frac{8\lambda_{\text{Im}}^{(n)}s_{\star}}{Hn\sigma_{U}}=\frac{64s_{\star}}{\sigma_{U}\sqrt{n}}\sqrt{\log\frac{dHn^{2}}{\delta}},
\end{align*}
which implies
\begin{align*}
 & \abs{\mA}\norm{\sum_{\tau=1}^{n}\sum_{k=1}^{H}\eta_{\Estimator{H+1}n,k}^{(\tau)}(\PseudoAction k{\tau})\Feature{\State k{\tau}}{\PseudoAction k{\tau}}}_{\infty}+\abs{\mA}\sqrt{nH\log\frac{dHn^{2}}{\delta}}\norm{\Impute hn-\Barw hn}_{1}\\
 & =\frac{64\abs{\mA}s_{\star}\sqrt{H}}{\sigma_{U}}\log\frac{dHn^{2}}{\delta}\\
 & \le9\abs{\mA}H\sqrt{n\log\frac{dHn^{2}}{\delta}}\\
 & =\lambda_{\text{Est}}^{(n)},
\end{align*}
where the last inequality holds because $n\ge2^{12}s_{\star}^{2}\sigma_{U}^{-2}H^{-1}\log(dHn^{2}/\delta)$
for $n\ge n_{1}$. Suppose \eqref{eq:error_lambda} holds for steps
$H,\ldots,h+1$. 
Then by Lemma \ref{lem:lasso} and \eqref{eq:sigma_inequality},
\[
\max_{h^{\prime}\ge h+1}\norm{\Estimator{h^{\prime}}n-\Barw{h^{\prime}}n}_{1}\le\frac{8\lambda_{\text{Est}}^{(n)}s_{\star}}{nH\abs{\mA}\sigma_{U}}=\frac{72s_{\star}}{\sqrt{n}\sigma_{U}}\sqrt{\log\frac{dHn^{2}}{\delta}}.
\]
We decompose
\begin{align*}
 & \norm{\sum_{\tau=1}^{n}\sum_{k=1}^{H}\eta_{\Estimator{h+1}n,k}^{(\tau)}(\PseudoAction k{\tau})\Feature{\State k{\tau}}{\PseudoAction k{\tau}}}_{\infty}\\
 & \le\norm{\sum_{\tau=1}^{n}\sum_{k=1}^{H}\left\{ \eta_{\Estimator{h+1}n,k}^{(\tau)}(\PseudoAction k{\tau})-\eta_{w_{h+1}^{\star},k}^{(\tau)}(\PseudoAction k{\tau})\right\} \Feature{\State k{\tau}}{\PseudoAction k{\tau}}}_{\infty}+\norm{\sum_{\tau=1}^{n}\sum_{k=1}^{H}\eta_{w_{h+1}^{\star},k}^{(\tau)}(\PseudoAction k{\tau})\Feature{\State k{\tau}}{\PseudoAction k{\tau}}}_{\infty}
\end{align*}
Because $\Estimator{h+1}n$ depends on the data, we take supremum
over $\mW_{h+1}:=\{w\in\Real^{d}:\|w-w_{h+1}^{\star}\|_{1}\le\sqrt{H}/(52\log\frac{Hdn^{2}}{\delta}\sqrt{\log2d})\}$.
To prove $\Estimator{h+1}n\in\mW_{h+1}$, we observe 
\begin{align*}
\norm{\Estimator{h+1}n-w_{h+1}^{\star}}_{1}\le & \norm{\Estimator{h+1}n-\Barw{h+1}n}_{1}+\norm{\Barw{h+1}n-w_{h+1}^{\star}}_{1}\\
\le & \norm{\Estimator{h+1}n-\Barw{h+1}n}_{1}+\sqrt{s_{\star}}\norm{\Barw{h+1}n-w_{h+1}^{\star}}_{2}\\
\le & \norm{\Estimator{h+1}n-\Barw{h+1}n}_{1}+\sqrt{\frac{2s_{\star}}{\sigma_{U}}}\norm{\Barw{h+1}n-w_{h+1}^{\star}}_{\frac{1}{nH}\sum_{\tau=1}^{n}\sum_{k=1}^{H}\Feature{\tilde{x}_{k}^{(\tau)}}{\tilde{a}_{k}^{(\tau)}}\Feature{\tilde{x}_{k}^{(\tau)}}{\tilde{a}_{k}^{(\tau)}}^{\top}},
\end{align*}
where the last inequality holds by Lemma \ref{lem:eigenvalue_bound}
and \eqref{eq:impute_sigma_infty}. By definition of $\Barw{h+1}n$,
\begin{align*}
 & \norm{\Barw{h+1}n-w_{h+1}^{\star}}_{\frac{1}{nH}\sum_{\tau=1}^{n}\sum_{k=1}^{H}\Feature{\tilde{x}_{k}^{(\tau)}}{\tilde{a}_{k}^{(\tau)}}\Feature{\tilde{x}_{k}^{(\tau)}}{\tilde{a}_{k}^{(\tau)}}^{\top}}^{2}\\
 & \le\max_{(x,a)\in\mX\times\mA}\abs{\Feature xa^{\top}(\Barw{h+1}n-w_{h+1}^{\star})}\\
 & \le\max_{(x,a)\in\mX\times\mA}\abs{\int\left\{ \Pi_{[0,H]}\left(\max_{a^{\prime}\in\mA}\widehat{Q}_{\Estimator{h+2}n}(x^{\prime},a)\right)-\Value{h+1}{\star}{x^{\prime}}\right\} \Feature xa^{\top}\psi(x^{\prime})dx^{\prime}}\\
 & \le\max_{x^{\prime}\in\mX}\abs{\Pi_{[0,H]}\left(\max_{a^{\prime}\in\mA}\widehat{Q}_{\Estimator{h+2}n}(x^{\prime},a)\right)-\max_{a^{\prime}\in\mA}\AV{h+1}{\star}{x^{\prime}}a}\int\Feature xa^{\top}\psi(x^{\prime})dx^{\prime}\\
 & =\max_{x^{\prime}\in\mX}\abs{\Pi_{[0,H]}\left(\max_{a^{\prime}\in\mA}\widehat{Q}_{\Estimator{h+2}n}(x^{\prime},a)\right)-\max_{a^{\prime}\in\mA}\AV{h+1}{\star}{x^{\prime}}a}.
\end{align*}
By definition of $\AV{h+1}{\star}{x^{\prime}}a=\Reward{x^{\prime}}a+\Feature{x^{\prime}}a^{\top}w_{h+2}^{\star}=\widehat{Q}_{w_{h+2}^{\star}}(x^{\prime},a)\in[0,H]$,
we obtain 
\begin{align*}
 & \norm{\Barw{h+1}n-w_{h+1}^{\star}}_{\frac{1}{nH}\sum_{\tau=1}^{n}\sum_{k=1}^{H}\Feature{\tilde{x}_{k}^{(\tau)}}{\tilde{a}_{k}^{(\tau)}}\Feature{\tilde{x}_{k}^{(\tau)}}{\tilde{a}_{k}^{(\tau)}}^{\top}}^{2}\\
 & \le\max_{x^{\prime}\in\mX}\abs{\Pi_{[0,H]}\left(\max_{a^{\prime}\in\mA}\widehat{Q}_{\Estimator{h+2}n}(x^{\prime},a)\right)-\max_{a^{\prime}\in\mA}\widehat{Q}_{w_{h+2}^{\star}}(x^{\prime},a)}\\
 & =\max_{x^{\prime}\in\mX}\abs{\Pi_{[0,H]}\left(\max_{a^{\prime}\in\mA}\widehat{Q}_{\Estimator{h+2}n}(x^{\prime},a)\right)-\Pi_{[0,H]}\left(\max_{a^{\prime}\in\mA}\widehat{Q}_{w_{h+2}^{\star}}(x^{\prime},a)\right)}\\
 & \le\max_{x^{\prime}\in\mX}\max_{a^{\prime}\in\mA}\abs{\widehat{Q}_{\Estimator{h+2}n}(x^{\prime},a)-\max_{a^{\prime}\in\mA}\widehat{Q}_{w_{h+2}^{\star}}(x^{\prime},a)}\\
 & \le\max_{(x,a)\in\mX\times\mA}\abs{\Feature xa^{\top}\left(\Estimator{h+2}n-w_{h+2}^{\star}\right)}.
\end{align*}
Therefore,
\begin{align*}
\norm{\Estimator{h+1}n-w_{h+1}^{\star}}_{1}\le & \norm{\Estimator{h+1}n-\Barw{h+1}n}_{1}+\sqrt{\frac{2s_{\star}}{\sigma_{U}}\max_{(x,a)\in\mX\times\mA}\abs{\Feature xa^{\top}\left(\Estimator{h+2}n-w_{h+2}^{\star}\right)}}\\
\le & \norm{\Estimator{h+1}n-\Barw{h+1}n}_{1}+\sqrt{\frac{2s_{\star}}{\sigma_{U}}\left(\norm{\Estimator{h+2}n-\Barw{h+2}n}_{1}+\max_{(x,a)\in\mX\times\mA}\abs{\Feature xa^{\top}\left(\Barw{h+2}n-w_{h+2}^{\star}\right)}\right)}
\end{align*}
Applying the inequality recursively,
\[
\norm{\Estimator{h+1}n-w_{h+1}^{\star}}_{1}\le\norm{\Estimator{h+1}n-\Barw{h+1}n}_{1}+\sqrt{\frac{2s_{\star}}{\sigma_{U}}\sum_{h^{\prime}=h+2}^{H}\norm{\Estimator{h^{\prime}}n-\Barw{h^{\prime}}n}_{1}}.
\]
By inductive assumption,
\begin{align*}
\norm{\Estimator{h+1}n-w_{h+1}^{\star}}_{1}\le & \frac{72s_{\star}}{\sqrt{n}\sigma_{U}}\sqrt{\log\frac{dHn^{2}}{\delta}}+\sqrt{\frac{144Hs_{\star}^{2}}{\sqrt{n}\sigma_{U}^{2}}\sqrt{\log\frac{dHn^{2}}{\delta}}}\\
\le & \frac{\sqrt{H}}{52\log\frac{Hdn^{2}}{\delta}\sqrt{\log2d}}
\end{align*}
where the last inequality holds by $n\ge Cs_{\star}^{4}\sigma_{U}^{-4}H^{-2}\log^{5}(dHn^{2}/\delta)\log^{2}(2d)$
for some absolute constant $C:=(144)^{2}\cdot(52)^{4}+8\cdot(72)^{2}(52)^{2}$
for $n\ge n_{1}$. Thus, we obtain $\Estimator{h+1}n\in\mW_{h+1}$,
and
\begin{align*}
 & \norm{\sum_{\tau=1}^{n}\sum_{k=1}^{H}\eta_{\Estimator{h+1}n,k}^{(\tau)}(\PseudoAction k{\tau})\Feature{\State k{\tau}}{\PseudoAction k{\tau}}}_{\infty}\\
 & \le\sup_{w\in\mW_{h+1}}\norm{\sum_{\tau=1}^{n}\sum_{k=1}^{H}\left\{ \eta_{w,k}^{(\tau)}(\PseudoAction k{\tau})-\eta_{w_{h+1}^{\star},k}^{(\tau)}(\PseudoAction k{\tau})\right\} \Feature{\State k{\tau}}{\PseudoAction k{\tau}}}_{\infty}+\norm{\sum_{\tau=1}^{n}\sum_{k=1}^{H}\eta_{w_{h+1}^{\star},k}^{(\tau)}(\PseudoAction k{\tau})\Feature{\State k{\tau}}{\PseudoAction k{\tau}}}_{\infty}
\end{align*}
By Lemma~\ref{lem:sup_bound}, with probability at least $1-\delta/(Hn^{2})$
\begin{equation}
\sup_{w\in\mW_{h+1}}\norm{\sum_{\tau=1}^{n}\sum_{k=1}^{H}\left\{ \eta_{w,k}^{(\tau)}(\PseudoAction k{\tau})-\eta_{w_{h+1}^{\star},k}^{(\tau)}(\PseudoAction k{\tau})\right\} \Feature{\State k{\tau}}{\PseudoAction k{\tau}}}_{\infty}\le3H\sqrt{n\log\frac{dHn^{2}}{\delta}}\label{eq:sup_bound}
\end{equation}
Note that $\|\eta_{h}^{w}(\State u{\tau},\Action u{\tau})\Feature{\State u{\tau}}{\Action u{\tau}}\|_{\infty}\le H$.
By Lemma \ref{lem:Azuma_Bernstein}, with probability at least $1-2\delta/(Hn^{2})$,
\begin{align*}
 & \norm{\sum_{\tau=1}^{n}\sum_{k=1}^{H}\eta_{w_{h+1}^{\star},k}^{(\tau)}(\PseudoAction k{\tau})\Feature{\State k{\tau}}{\PseudoAction k{\tau}}}_{\infty}\\
 & \le\frac{2H}{3}\log\frac{dHn^{2}}{\delta}+\sqrt{2\sum_{\tau=1}^{n}\sum_{k=1}^{H}\CE{\left(\eta_{w_{h+1}^{\star},k}^{(\tau)}(\PseudoAction k{\tau})\right)^{2}}{\mH_{k}^{(\tau)}}\log\frac{dHn^{2}}{\delta}},
\end{align*}
where $\mH_{k}^{(\tau)}$ is a sigma algebra generated by $\{\State{h^{\prime}}u,\Action{h^{\prime}}u\}_{u\in[\tau-1],h^{\prime}\in[H]}\cup\{\State{h^{\prime}}{\tau},\Action{h^{\prime}}{\tau}\}_{h^{\prime}\in[k]}$.
Note that
\begin{align*}
\Pi_{[0,H]}\left(\max_{a^{\prime}\in\mA}\widehat{Q}_{w_{h+1}^{\star}}(x,a^{\prime})\right)= & \Pi_{[0,H]}\left(\max_{a\in\mA}\left\{ \Reward xa+\Feature xa^{\top}w_{h+1}^{\star}\right\} \right)\\
= & \Pi_{[0,H]}\left(\max_{a\in\mA}\left\{ \Reward xa+\left[\Probability V_{h+1}^{\star}\right](x,a)\right\} \right)\\
= & \Pi_{[0,H]}\left(\max_{a\in\mA}\AV h{\star}xa\right)\\
= & \Pi_{[0,H]}\left(\Value h{\star}x\right)\\
= & \Value h{\star}x.
\end{align*}
By definition \eqref{eq:eta_definition},
\begin{align*}
\eta_{w_{h+1}^{\star},k}^{(\tau)}(\PseudoAction k{\tau})= & \Pi_{[0,H]}\left(\max_{a^{\prime}\in\mA}\widehat{Q}_{w_{h+1}^{\star}}(X_{k+1}^{(\tau)}(\PseudoAction k{\tau}),a^{\prime})\right)-\Expectation_{X\sim\CP{\cdot}{\State k{\tau},\PseudoAction k{\tau}}}\left[\Pi_{[0,H]}\left(\max_{a^{\prime}\in\mA}\widehat{Q}_{w_{h+1}^{\star}}(X,a^{\prime})\right)\right]\\
= & \Value h{\star}{X_{k+1}^{(\tau)}(\PseudoAction k{\tau})}-[\Probability V_{h}^{\star}](\State k{\tau},\PseudoAction k{\tau}).
\end{align*}
Applying Lemma \ref{lem:exp_bound}, with probability at least $1-2\delta/(Hn^{2})$,
\begin{align*}
\sum_{\tau=1}^{n}\sum_{k=1}^{H}\CE{\eta_{w_{h+1}^{\star},k}^{(\tau)}(\PseudoAction k{\tau})^{2}}{\mH_{k}^{(\tau)}} & \le\sum_{\tau=1}^{n}\Expectation\left[\sum_{k=1}^{H}\CE{\eta_{w_{h+1}^{\star},k}^{(\tau)}(\PseudoAction k{\tau})^{2}}{\mH_{k}^{(\tau)}}\right]+H^{2}\sqrt{2n\log\frac{Hn^{2}}{\delta}}.
\end{align*}
Using the variance bound (Lemma \ref{lem:Var_bound})
we get
\begin{align*}
\sum_{\tau=1}^{n}\sum_{k=1}^{H}\CE{\eta_{w_{h+1}^{\star},k}^{(\tau)}(\PseudoAction k{\tau})^{2}}{\mH_{k}^{(\tau)}}\le & 5n(H^{2}+H)+H^{3}\sqrt{2n\log\frac{Hn^{2}}{\delta}}\\
\le & 11nH^{2},
\end{align*}
where the last inequality holds by $n\ge n_{1}\ge2H^{2}\log(Hn^{2}/\delta)$.
Thus, 
\[
\norm{\sum_{\tau=1}^{n}\sum_{k=1}^{H}\eta_{w_{h+1}^{\star},k}^{(\tau)}(\PseudoAction k{\tau})\Feature{\State k{\tau}}{\PseudoAction k{\tau}}}_{\infty}\le\frac{2H}{3}\log\frac{dHn^{2}}{\delta}+H\sqrt{22n\log\frac{dHn^{2}}{\delta}}\le5H\sqrt{n\log\frac{dHn^{2}}{\delta}},
\]
where the last inequality holds by $n\ge n_{1}\ge(100/9)\log(dHn^{2}/\delta)$.
Thus, we obtain
\[
\norm{\sum_{\tau=1}^{n}\sum_{k=1}^{H}\eta_{\Estimator{h+1}n,k}^{(\tau)}(\PseudoAction k{\tau})\phi_{k}^{(\tau)}}_{\infty}\le8H\sqrt{n\log\frac{dHn^{2}}{\delta}}
\]
which implies
\begin{align*}
 & \abs{\mA}\norm{\sum_{\tau=1}^{n}\sum_{k=1}^{H}\eta_{\Estimator{h+1}n,k}^{(\tau)}(\PseudoAction k{\tau})\Feature{\State k{\tau}}{\PseudoAction k{\tau}}}_{\infty}+\abs{\mA}\sqrt{nH\log\frac{dHn^{2}}{\delta}}\norm{\Impute hn-\Barw hn}_{1}\\
 & \le8\abs{\mA}H\sqrt{n\log\frac{dHn^{2}}{\delta}}+\abs{\mA}\sqrt{nH\log\frac{dHn^{2}}{\delta}}\norm{\Impute hn-\Barw hn}_{1}.
\end{align*}
By using similar argument, we obtain with probability at least $1-5\delta/(Hn^{2})$,
\begin{equation}
\norm{\sum_{\tau=1}^{n}\sum_{k=1}^{H}\eta_{\Estimator{h+1}n,k}^{(\tau)}(\Action k{\tau})\phi_{k}^{(\tau)}}_{\infty}\le\lambda_{\text{Im}}^{(n)},
\label{eq:Impute_lambda}
\end{equation}
Using Lemma~\ref{lem:lasso},
\[
\norm{\Impute hn-\Barw hn}_{1}\le\frac{8\lambda_{\text{Im}}^{(n)}s_{\star}}{\sigma_{U}Hn}=\frac{64s_{\star}}{\sigma_{U}\sqrt{n}}\sqrt{\log\frac{dHn^{2}}{\delta}},
\]
which implies
\begin{align*}
 & \abs{\mA}\norm{\sum_{\tau=1}^{n}\sum_{k=1}^{H}\eta_{\Estimator{h+1}n,k}^{(\tau)}(\PseudoAction k{\tau})\Feature{\State k{\tau}}{\PseudoAction k{\tau}}}_{\infty}+\abs{\mA}\sqrt{nH\log\frac{dHn^{2}}{\delta}}\norm{\Impute hn-\Barw hn}_{1}\\
 & \le8\abs{\mA}H\sqrt{n\log\frac{dHn^{2}}{\delta}}+\frac{64s_{\star}\sqrt{H}}{\sigma_{U}}\log\frac{dHn^{2}}{\delta}\\
 & \le9\abs{\mA}H\sqrt{n\log\frac{dHn^{2}}{\delta}},
\end{align*}
where the last inequality holds by $n\ge n_{1}$. Therefore we conclude
\[
\norm{\sum_{\tau=1}^{n}\sum_{k=1}^{H}\sum_{a\in\mA}\tilde{\eta}_{\Estimator{h+1}n,k}^{(\tau)}(a)\Feature{\State k{\tau}}a}_{\infty}\le9\abs{\mA}H\sqrt{n\log\frac{dHn^{2}}{\delta}}=\lambda_{\text{Est}}^{(n)}
\]
\end{proof}    

\subsection{Proof of Lemma~\ref{lem:sup_bound}}

%\begin{lem}
%(Restatement of Lemma~\ref{lem:sup_bound}) Suppose $n^{3}\ge16e^{2}$
%and let $\eta_{w,k}^{(\tau)}(a)$ be a random variable defined in
%\eqref{eq:eta_definition}. For $\rho>0$, define 
%\[
%\mW_{h+1}(\rho):=\left\{ w\in\Real^{d}:\|w-w_{h+1}^{\star}\|_{1}\le\rho\right\} .
%\]
%Let $\Action 1{\tau},\ldots,\Action H{\tau}$ denote the selected
%actions by policy $\pi^{(\tau)}$. Then for any policy $\pi^{(\tau)}$,
%with probability at least $1-\delta/(Hn^{2})$, 
%\[
%\sup_{w\in\mW_{h+1}(\rho)}\norm{\sum_{\tau=1}^{n}\sum_{k=1}^{H}\left\{ \eta_{w,k}^{(\tau)}(\Action k{\tau})-\eta_{w_{h+1}^{\star},k}^{(\tau)}(\Action k{\tau})\right\} \phi_{k}^{(\tau)}}_{\infty}\le\rho\sqrt{2nH\log2d}\left(8+\frac{256\sqrt{3}}{3}\log^{3/2}\frac{Hdn^{2}}{\delta}\right)
%\]
%\end{lem}

\begin{proof}
Fix the policies $\pi^{(1)},\ldots,\pi^{(n)}$ and set $Z_{k}^{(\tau)}:=(\State k{\tau},\Action k{\tau},\State{k+1}{\tau},u_{k}^{(\tau)})$
and $\mathcal{Z}:=(\mX\times\mA)^{2}$, where $\{u_{k}^{(\tau)}\}_{k\in[H],\tau\in[n]}$
are the IID Uniform random variables over $\mA$. Let $\mH_{k}^{(\tau)}$
denote the sigma algebra generated by $\{\State vu,\Action vu\}_{v\in[H],u\in[\tau-1]}\cup\{\State v{\tau},\Action v{\tau}\}_{v\in[k]}$
with $\mH_{0}^{(\tau)}:=\mH_{H}^{(\tau-1)}$. For $i\in[d]$, 
\begin{align*}
f_{w,i}(x_{1},a,x_{2},u) & :=\rho^{-1}\left\{ \Pi_{[0,H]}\left(\max_{a^{\prime}\in\mA}\widehat{Q}_{w}(x_{2},a^{\prime})\right)-\Pi_{[0,H]}\left(\max_{a^{\prime}\in\mA}\widehat{Q}_{w_{h+1}^{\star}}(x_{2},a^{\prime})\right)\right\} \Feature{x_{1}}a(i)
\end{align*}
Then with the function class $\mF_{i}:=\{f_{w,i}:w\in\mW_{h+1}\}$,
\[
\sup_{w\in\mW_{h+1}}\norm{\sum_{\tau=1}^{n}\sum_{k=1}^{H}\left\{ \eta_{w,k}^{(\tau)}-\eta_{w_{h+1}^{\star},k}^{(\tau)}\right\} \phi_{k}^{(\tau)}}_{\infty}=\rho\max_{i\in[d]}\sup_{f_{i}\in\mF_{i}}\abs{\sum_{\tau=1}^{n}\sum_{k=1}^{H}f_{i}(\State k{\tau},\Action k{\tau},\State{k+1}{\tau},u_{k}^{(\tau)})}.
\]
Note that for any $f\in\mF_{i}$, there exists $w\in\mW_{h+1}$ such
that 
\begin{align*}
\max_{(x_{1},a,x_{2},u)\in\mathcal{Z}}\abs{f(x_{1},a,x_{2},u)}\le & \rho^{-1}\abs{\Pi_{[0,H]}\left(\max_{a^{\prime}\in\mA}\widehat{Q}_{w}(x_{2},a^{\prime})\right)-\Pi_{[0,H]}\left(\max_{a^{\prime}\in\mA}\widehat{Q}_{w_{h+1}^{\star}}(x_{2},a^{\prime})\right)}\\
\le & \rho^{-1}\max_{a^{\prime}\in\mA}\abs{\widehat{Q}_{h}^{w}(x,a^{\prime})-\widehat{Q}_{h}^{w_{h+1}^{\star}}(x,a^{\prime})}\\
\le & \rho^{-1}\norm{w-w_{h+1}^{\star}}_{1}\\
\le & 1.
\end{align*}
Let $\mathbf{z}:=(\mathbf{z}_{1}^{(1)},\ldots,\mathbf{z}_{H}^{(n)})$
denote a sequence of binary tree such that $\mathbf{z}_{k}^{(\tau)}:\{\pm1\}^{\tau H+k}\to\mathcal{Z}$
and $\xi:=(\xi_{1}^{(1)},\ldots,\xi_{H}^{(n)})$ denote a sequnce
of IID Bernoulli random variables with $\Probability(\xi_{1}^{(1)}=-1)=\Probability(\xi_{1}^{(1)}=1)=1/2$.
By Lemma \ref{lem:prob_bound}, for any $x>0$, 
\[
\Probability\left(\max_{i\in[d]}\sup_{f\in\mF_{i}}\abs{\sum_{\tau=1}^{n}\sum_{k=1}^{H}f(\State k{\tau},\Action k{\tau},\State{k+1}{\tau},u_{k}^{(\tau)})}>x\right)\le4\sum_{i=1}^{d}\sup_{\mathbf{z}}\CP{\sup_{f\in\mF_{i}}\abs{\sum_{\tau=1}^{n}\sum_{k=1}^{H}\xi_{k}^{(\tau)}f(\mathbf{z}_{k}^{(\tau)}(\xi))}>\frac{x}{4}}{\mathbf{z}}.
\]
By Lemma~\ref{lem:entropy_bound}, setting 
\[
x=8\sup_{\mathbf{z}}\inf_{\alpha>0}\left\{ nH\alpha+2\int_{\alpha}^{1/2}\sqrt{3nH\log\frac{N(\epsilon,\mF,\|\cdot\|_{\infty,\mathbf{z}})\sqrt{n^{2}Hd}}{\sqrt{\delta}}}d\epsilon\right\} 
\]
we obtain 
\[
\Probability\left(\max_{i\in[d]}\sup_{f\in\mF_{i}}\abs{\sum_{\tau=1}^{n}\sum_{k=1}^{H}f(\State k{\tau},\Action k{\tau},\State{k+1}{\tau},u_{k}^{(\tau)})}>x\right)\le\frac{\delta}{n^{2}Hd}.
\]
To find an upper bound for $x$, define a function $g_{w,i}:\mathcal{Z}\to\Real$
by 
\begin{align*}
g_{w,i}(x_{1},a,x_{2},u) & :=L^{-1}\Feature{x_{2}}u^{\top}w\Feature{x_{1}}a(i),
\end{align*}
and a function class $\mathcal{G}_{i}:=\{g_{w,i}-g_{w_{h+1}^{\star},i}:w\in\mW_{h+1}\}$.
Given $\epsilon>0$ and a binary tree $\mathbf{z}:=(\mathbf{z}_{1}^{(1)},\ldots,\mathbf{z}_{H}^{(n)})=((\State 11,\Action 11,\State 21,u_{1}^{(1)}),\ldots(\State Hn,\Action Hn,\State{H+1}n,u_{H}^{(n)}))$,
for any $f\in\mF_{i}$, there exists $g_{\tilde{w},i}$ in the $\epsilon$-cover
of $\mathcal{G}_{i}$ such that 
\begin{align*}
 & \max_{\tau,k}\bigg|f_{w,i}(\mathbf{z}_{k}^{(\tau)})-\Pi_{[0,H]}\left(\max_{a^{\prime}\in\mA}\widehat{Q}_{h}^{w_{h+1}^{\star}}(\State{k+1}{\tau},a^{\prime})\right)\Feature{\State k{\tau}}{\Action k{\tau}}(i)\\
 & \quad-\Pi_{[0,H]}\left(\max_{u\in\mA}r(u,\State{k+1}{\tau})+g_{\tilde{w},i}(\State k{\tau},\Action k{\tau},\State{k+1}{\tau},u)\right)\Feature{\State k{\tau}}{\Action k{\tau}}(i)\bigg|\\
 & \le\max_{\tau,k}\bigg|\Pi_{[0,H]}\left(\max_{u\in\mA}r(u,\State{k+1}{\tau})+g_{w,i}(\State k{\tau},\Action k{\tau},\State{k+1}{\tau},u)\right)-\Pi_{[0,H]}\left(\max_{u\in\mA}r(u,\State{k+1}{\tau})+g_{\tilde{w},i}(\State k{\tau},\Action k{\tau},\State{k+1}{\tau},u)\right)\bigg|\\
 & \le\max_{\tau,k}\max_{u\in\mA}\abs{g_{w,i}(\State k{\tau},\Action k{\tau},\State{k+1}{\tau},u)-g_{\tilde{w},i}(\State k{\tau},\Action k{\tau},\State{k+1}{\tau},u)}\\
 & \le\max_{\tau,k}\abs{g_{w,i}(\mathbf{z}_{k}^{(\tau)})-g_{w_{h+1}^{\star},i}(\mathbf{z}_{k}^{(\tau)})+g_{w_{h+1}^{\star},i}(\mathbf{z}_{k}^{(\tau)})-g_{\tilde{w},i}(\mathbf{z}_{k}^{(\tau)})}\\
 & \le\epsilon.
\end{align*}
Thus ,$N(\epsilon,\mF_{i},\|\cdot\|_{\infty,\mathbf{z}})\le N(\epsilon,\mathcal{G}_{i},\|\cdot\|_{\infty,\mathbf{z}})$
and 
\[
2\int_{\alpha}^{1/2}\sqrt{3nH\log\frac{N(\epsilon,\mF_{i},\|\cdot\|_{\infty,\mathbf{z}})\sqrt{n^{2}Hd}}{\sqrt{\delta}}}d\epsilon\le2\int_{\alpha}^{1/2}\sqrt{3nH\log\frac{N(\epsilon,\mathcal{G}_{i},\|\cdot\|_{\infty,\mathbf{z}})\sqrt{n^{2}Hd}}{\sqrt{\delta}}}d\epsilon.
\]
Define the sequential rademacher complexity, 
\[
R_{H}^{(n)}(\mathcal{G}_{i}):=\sup_{\mathbf{z}}\Expectation\left[\sup_{g\in\mathcal{G}_{i}}\sum_{\tau=1}^{n}\sum_{k=1}^{H}\xi_{k}^{(\tau)}g(\mathbf{z}_{k}^{(\tau)}(\xi))\right].
\]
Note that the Rachmechar complexity is bounded as 
\begin{align*}
R_{H}^{(n)}(\mathcal{G}_{i})= & \sup_{\mathbf{z}}\Expectation\left[\sup_{g\in\mathcal{G}_{i}}\sum_{\tau=1}^{n}\sum_{k=1}^{H}\xi_{k}^{(\tau)}g(\mathbf{z}_{k}^{(\tau)}(\xi))\right]\\
= & \rho^{-1}\sup_{\mathbf{z}}\Expectation\left[\sup_{w\in\mW_{h+1}}\sum_{\tau=1}^{n}\sum_{k=1}^{H}\xi_{k}^{(\tau)}\Feature{\State k{\tau}}{u_{k}^{(\tau)}}^{\top}\left(w-w_{h+1}^{\star}\right)\Feature{\State k{\tau}}{\Action k{\tau}}(i)\right]\\
\le & \rho^{-1}\sup_{\mathbf{z}}\sup_{w\in\mW_{h+1}}\norm{w-w_{h+1}^{\star}}_{1}\Expectation\left[\norm{\sum_{\tau=1}^{n}\sum_{k=1}^{H}\xi_{k}^{(\tau)}\Feature{\State k{\tau}}{u_{k}^{(\tau)}}\Feature{\State k{\tau}}{\Action k{\tau}}(i)}_{\infty}\right]\\
\le & \sup_{\mathbf{z}}\Expectation\left[\norm{\sum_{\tau=1}^{n}\sum_{k=1}^{H}\xi_{k}^{(\tau)}\Feature{\State k{\tau}}{u_{k}^{(\tau)}}\Feature{\State k{\tau}}{\Action k{\tau}}(i)}_{\infty}\right].
\end{align*}
By Jensen's inequality, for any $\lambda>0$, 
\begin{align*}
 & \Expectation\left[\norm{\sum_{\tau=1}^{n}\sum_{k=1}^{H}\xi_{k}^{(\tau)}\Feature{\State k{\tau}}{u_{k}^{(\tau)}}\Feature{\State k{\tau}}{\Action k{\tau}}(i)}_{\infty}\right]\\
 & \le\frac{1}{\lambda}\log\Expectation\exp\left(\lambda\norm{\sum_{\tau=1}^{n}\sum_{k=1}^{H}\xi_{k}^{(\tau)}\Feature{\State k{\tau}}{u_{k}^{(\tau)}}\Feature{\State k{\tau}}{\Action k{\tau}}(i)}_{\infty}\right)\\
 & \le\frac{1}{\lambda}\log\sum_{j\in[d]}\Expectation\exp\left(\lambda\abs{\sum_{\tau=1}^{n}\sum_{k=1}^{H}\xi_{k}^{(\tau)}\Feature{\State k{\tau}}{u_{k}^{(\tau)}}(j)\Feature{\State k{\tau}}{\Action k{\tau}}(i)}\right)\\
 & \le\frac{1}{\lambda}\log\bigg\{\sum_{j\in[d]}\Expectation\exp\left(\lambda\sum_{\tau=1}^{n}\sum_{k=1}^{H}\xi_{k}^{(\tau)}\Feature{\State k{\tau}}{u_{k}^{(\tau)}}(j)\Feature{\State k{\tau}}{\Action k{\tau}}(i)\right)\\
 & +\Expectation\exp\left(-\lambda\sum_{\tau=1}^{n}\sum_{k=1}^{H}\xi_{k}^{(\tau)}\Feature{\State k{\tau}}{u_{k}^{(\tau)}}(j)\Feature{\State k{\tau}}{\Action k{\tau}}(i)\right)\bigg\}\\
 & \le\frac{1}{\lambda}\log2d\exp\left(\frac{\lambda^{2}nH}{2}\right)\\
 & =\frac{\log2d}{\lambda}+\frac{\lambda nH}{2}.
\end{align*}
where the last inequality uses $\|\Feature xa\|_{\infty}\le1$. Minimizing
over $\lambda>0$ gives 
\[
\Expectation\left[\norm{\sum_{\tau=1}^{n}\sum_{k=1}^{H}\xi_{k}^{(\tau)}\Feature{\State k{\tau}}{u_{k}^{(\tau)}}\Feature{\State k{\tau}}{\Action k{\tau}}(i)}_{\infty}\right]\le\sqrt{\frac{nH\log2d}{2}}.
\]
Therefore we obtain, 
\[
R_{H}^{(n)}(\mathcal{G}_{i})\le\sqrt{\frac{nH\log2d}{2}}.
\]
Setting $\alpha=\sqrt{\frac{2\log2d}{nH}}$, 
\begin{align*}
 & \sup_{\mathbf{z}}\inf_{\alpha>0}\left\{ nH\alpha+2\int_{\alpha}^{1/2}\sqrt{3nH\log\frac{N(\epsilon,\mathcal{G}_{i},\|\cdot\|_{\infty,\mathbf{z}})\sqrt{n^{2}Hd}}{\sqrt{\delta}}}d\epsilon\right\} \\
 & \le\sqrt{2nH\log2d}+2\sup_{\mathbf{z}}\int_{\sqrt{\frac{2\log2d}{nH}}}^{1/2}\sqrt{3nH\log\frac{N(\epsilon,\mathcal{G}_{i},\|\cdot\|_{\infty,\mathbf{z}})\sqrt{n^{2}Hd}}{\sqrt{\delta}}}d\epsilon.
\end{align*}
By Corollary 1 and Lemma 2 in \citet{rakhlin2015sequential}, whenever
$\epsilon\ge\sqrt{\frac{2\log2d}{nH}}\ge2n^{-1}H^{-1}R_{H}^{(n)}(\mathcal{G}_{i})$,
\[
\log N(\epsilon,\mathcal{G},\|\cdot\|_{\infty,\mathbf{z}})\le\frac{32}{nH\epsilon^{2}}R_{H}^{(n)}(\mathcal{G}_{i})^{2}\log\frac{2enH}{\epsilon}.
\]
Thus, 
\begin{align*}
 & \int_{\sqrt{\frac{2\log2d}{nH}}}^{1/2}\sqrt{3nH\log\frac{N(\epsilon,\mathcal{G}_{i},\|\cdot\|_{\infty,\mathbf{z}})\sqrt{n^{2}Hd}}{\sqrt{\delta}}}d\epsilon\\
 & \le4\sqrt{6}R_{H}^{(n)}(\mathcal{G}_{i})\int_{\sqrt{\frac{2\log2d}{nH}}}^{1/2}\frac{1}{\epsilon}\sqrt{-\log\epsilon+\log\frac{2enH\sqrt{n^{2}Hd}}{\sqrt{\delta}}}d\epsilon\\
 & =4\sqrt{6}R_{H}^{(n)}(\mathcal{G}_{i})\left[-\frac{2}{3}\left(-\log\epsilon+\log\frac{2enH\sqrt{n^{2}Hd}}{\sqrt{\delta}}\right)^{3/2}\right]_{\sqrt{\frac{2\log2d}{nH}}}^{1/2}\\
 & \le\frac{8\sqrt{6}}{3}R_{H}^{(n)}(\mathcal{G}_{i})\left(\log\sqrt{\frac{nH}{2\log2d}}+\log\frac{2enH\sqrt{n^{2}Hd}}{\sqrt{\delta}}\right)^{3/2}.
\end{align*}
Now we obtain 
\begin{align*}
x\le & 8\sqrt{2nH\log2d}+16\frac{8\sqrt{6}}{3}R_{H}^{(n)}(\mathcal{G}_{i})\log^{3/2}\frac{2enH^{2}\sqrt{n^{3}d}}{\sqrt{2\delta\log2d}}\\
\le & 8\sqrt{2nH\log2d}+\frac{128\sqrt{3}}{3}\sqrt{nH\log2d}\log^{3/2}\frac{4enH^{2}\sqrt{n^{3}d}}{\sqrt{2\delta\log2d}}\\
\le & 8\sqrt{2nH\log2d}+\frac{256\sqrt{3}}{3}\sqrt{2nH\log2d}\log^{3/2}\frac{2\sqrt{e}Hdn^{\frac{5}{4}}}{\delta}\\
\le & 8\sqrt{2nH\log2d}+\frac{256\sqrt{3}}{3}\sqrt{2nH\log2d}\log^{3/2}\frac{Hdn^{2}}{\delta},
\end{align*}
the last inequality holds by $n^{3}\ge16e^{2}$. Thus we conclude
with probability at least $1-\delta/(Hn^{2})$ 
\begin{align*}
\sup_{w\in\mW_{h+1}(\rho)}\norm{\sum_{\tau=1}^{n}\sum_{k=1}^{H}\left\{ \eta_{w,k}^{(\tau)}-\eta_{w_{h+1}^{\star},k}^{(\tau)}\right\} \phi_{k}^{(\tau)}}_{\infty}\le & \rho\left(8\sqrt{2nH\log2d}+\frac{256\sqrt{3}}{3}\sqrt{2nH\log2d}\log^{3/2}\frac{Hdn^{2}}{\delta}\right)\\
\le & \rho\sqrt{2nH\log2d}\left(8+\frac{256\sqrt{3}}{3}\log^{3/2}\frac{Hdn^{2}}{\delta}\right)
\end{align*}
\end{proof}

\subsection{Proof of Lemma~\ref{lem:Var_bound}}

%\begin{lem}
%(Restatement of Lemma~\ref{lem:Var_bound}) 
%\label{lem:Var_bound}
%For each $\tau\ge1$, let $\mH^{(\tau)}$ denote the sigma algebra
%generated by $\{\State us,\Action us\}_{s\in[\tau-1],u\in[H]}$, where $\Action 1{\tau},\ldots,\Action H{\tau}$ denote a sequence of
%actions selected by a policy $\pi^{(\tau)}$. 
%Then, for any policy $\pi^{(\tau)}$ and $h\in[H]$, the sum of variance of the value function is bounded by 
%\[
%\CE{\sum_{k=1}^{H}\left\{ \Value h{\star}{\State{k+1}{\tau}}-\left[\Probability V_{h}^{\star}\right](\State k{\tau},\Action k{\tau})\right\} ^{2}}{\mH^{(\tau)}}\le5H^{2}+5H.
%\]
%\end{lem}

\begin{proof}
For each $k\in[H]$, the definition of action value function $\AV{h-1}{\star}xa$
gives,
\begin{align*}
\left\{ \Value h{\star}{\State{k+1}{\tau}}-\left[\Probability V_{h}^{\star}\right](\State k{\tau},\Action k{\tau})\right\} ^{2}= & \left\{ \Value h{\star}{\State{k+1}{\tau}}-\AV{h-1}{\star}{\State k{\tau}}{\Action k{\tau}}+\Reward{\State k{\tau}}{\Action k{\tau}}\right\} ^{2}\\
\le & \frac{5}{4}\left\{ \Value h{\star}{\State{k+1}{\tau}}-\AV{h-1}{\star}{\State k{\tau}}{\Action k{\tau}}\right\} ^{2}+5\Reward{\State{h^{\prime}-1}{\tau}}{\Action{h^{\prime}-1}{\tau}}^{2}
\end{align*}
where the second inequality holds by $(a+b)^{2}\le\frac{5}{4}a^{2}+5b^{2}$
for $a,b\in\Real$. Because the reward function is bounded by $1$,
\[
\left\{ \Value h{\star}{\State{k+1}{\tau}}-\left[\Probability V_{h}^{\star}\right](\State k{\tau},\Action k{\tau})\right\} ^{2}\le\frac{5}{4}\left\{ \Value h{\star}{\State{k+1}{\tau}}-\AV{h-1}{\star}{\State k{\tau}}{\Action k{\tau}}\right\} ^{2}+5.
\]
For $k\in[H]$, let $\mH_{k}^{(\tau)}$denote the sigma algebra generated
by $\{\State u{\tau},\Action u{\tau}\}_{u=1,\ldots,k}\cup\{\State us,\Action us\}_{s\in[\tau-1],u\in[H]}$.
Taking conditional expectation on both sides,
\[
\CE{\left\{ \Value h{\star}{\State{k+1}{\tau}}-\left[\Probability V_{h}^{\star}\right](\State k{\tau},\Action k{\tau})\right\} ^{2}}{\mH_{k}^{(\tau)}}\le\frac{5}{4}\CE{\left\{ \Value h{\star}{\State{k+1}{\tau}}-\AV{h-1}{\star}{\State k{\tau}}{\Action k{\tau}}\right\} ^{2}}{\mH_{k}^{(\tau)}}+5.
\]
In the first term, 
\begin{align*}
 & \CE{\left\{ \Value h{\star}{\State{k+1}{\tau}}-\AV{h-1}{\star}{\State k{\tau}}{\Action k{\tau}}\right\} ^{2}}{\mH_{k}^{(\tau)}}\\
 & =\CE{\Value h{\star}{\State{k+1}{\tau}}^{2}}{\mH_{k}^{(\tau)}}-2\AV{h-1}{\star}{\State k{\tau}}{\Action k{\tau}}\left[\Probability V_{h}^{\star}\right](\State k{\tau},\Action k{\tau})+\AV{h-1}{\star}{\State k{\tau}}{\Action k{\tau}}^{2}.
\end{align*}
Note that for any $\Action k{\tau}\in\mA$, we have 
\begin{align*}
\left[\Probability V_{h}^{\star}\right](\State k{\tau},\Action k{\tau})= & \AV{h-1}{\star}{\State k{\tau}}{\Action k{\tau}}-\Reward{\State k{\tau}}{\Action k{\tau}}\\
\le & \AV{h-1}{\star}{\State k{\tau}}{\Action k{\tau}}.
\end{align*}
Because the function $f(x)=-2xb+x^{2}$ is nondecreasing for $x\ge b$,
\begin{align*}
 & \CE{\Value h{\star}{\State{k+1}{\tau}}^{2}}{\mH_{k}^{(\tau)}}-2\AV{h-1}{\star}{\State k{\tau}}{\Action k{\tau}}\left[\Probability V_{h}^{\star}\right](\State k{\tau},\Action k{\tau})+\AV{h-1}{\star}{\State k{\tau}}{\Action k{\tau}}^{2}\\
 & \le\CE{\Value h{\star}{\State{k+1}{\tau}}^{2}}{\mH_{k}^{(\tau)}}-2\max_{a\in\mA}\AV{h-1}{\star}{\State k{\tau}}a\left[\Probability V_{h}^{\star}\right](\State k{\tau},\Action k{\tau})+\max_{a\in\mA}\AV{h-1}{\star}{\State k{\tau}}a^{2}\\
 & =\CE{\Value h{\star}{\State{k+1}{\tau}}^{2}}{\mH_{k}^{(\tau)}}-2\Value{h-1}{\star}{\State k{\tau}}\left[\Probability V_{h}^{\star}\right](\State k{\tau},\Action k{\tau})+\Value{h-1}{\star}{\State k{\tau}}^{2}\\
 & =\CE{\left\{ \Value h{\star}{\State{k+1}{\tau}}-\Value{h-1}{\star}{\State k{\tau}}\right\} ^{2}}{\mH_{k}^{(\tau)}},
\end{align*}
Summing up over $k\in[H]$, 
\[
\sum_{k=1}^{H}\CE{\left\{ \Value h{\star}{\State{k+1}{\tau}}-\left[\Probability V_{h}^{\star}\right](\State k{\tau},\Action k{\tau})\right\} ^{2}}{\mH_{k}^{(\tau)}}\le\frac{5}{4}\sum_{k=1}^{H}\CE{\left\{ \Value h{\star}{\State{k+1}{\tau}}-\Value{h-1}{\star}{\State k{\tau}}\right\} ^{2}}{\mH_{k}^{(\tau)}}+5H.
\]
Note that $\left[\Probability V_{h}^{\star}\right](\State k{\tau},\Action k{\tau})\le\AV{h-1}{\star}{\State k{\tau}}{\Action k{\tau}}\le\max_{a^{\prime}\in\mA}\AV{h-1}{\star}{\State k{\tau}}{a^{\prime}}=\Value{h-1}{\star}{\State k{\tau}}$
for any $k\in[H]$. Thus, for any $k_{1}\neq k_{2}$, the cross-product
terms,
\[
\CE{\Value h{\star}{\State{k_{1}+1}{\tau}}-\Value{h-1}{\star}{\State{k_{1}}{\tau}}}{\mH_{k_{1}}^{(\tau)}}\CE{\Value h{\star}{\State{k_{2}+1}{\tau}}-\Value{h-1}{\star}{\State{k_{2}}{\tau}}}{\mH_{k_{2}}^{(\tau)}}\ge0,
\]
which implies
\[
\sum_{k=1}^{H}\CE{\left\{ \Value h{\star}{\State{k+1}{\tau}}-\Value{h-1}{\star}{\State k{\tau}}\right\} ^{2}}{\mH_{k}^{(\tau)}}\le\left\{ \sum_{k=1}^{H}\CE{\Value h{\star}{\State{k+1}{\tau}}-\Value{h-1}{\star}{\State k{\tau}}}{\mH_{k}^{(\tau)}}\right\} ^{2}
\]
Taking conditional expectation on both sides,
\begin{align*}
 & \CE{\sum_{k=1}^{H}\CE{\left\{ \Value h{\star}{\State{k+1}{\tau}}-\Value{h-1}{\star}{\State k{\tau}}\right\} ^{2}}{\mH_{k}^{(\tau)}}}{\mH^{(\tau)}}\\
 & \le\CE{\left\{ \sum_{k=1}^{H}\CE{\Value h{\star}{\State{k+1}{\tau}}-\Value{h-1}{\star}{\State k{\tau}}}{\mH_{k}^{(\tau)}}\right\} ^{2}}{\mH^{(\tau)}}\\
 & =\CE{\left\{ \sum_{k=1}^{H-1}\CE{\Value h{\star}{\State{k+1}{\tau}}-\Value{h-1}{\star}{\State k{\tau}}}{\mH_{k}^{(\tau)}}+\CE{\Value h{\star}{\State{H+1}{\tau}}-\Value{h-1}{\star}{\State H{\tau}}}{\mH_{H}^{(\tau)}}\right\} ^{2}}{\mH^{(\tau)}}\\
 & \le\CE{\CE{\left\{ \sum_{k=1}^{H-1}\CE{\Value h{\star}{\State{k+1}{\tau}}-\Value{h-1}{\star}{\State k{\tau}}}{\mH_{k}^{(\tau)}}+\Value h{\star}{\State{H+1}{\tau}}-\Value{h-1}{\star}{\State H{\tau}}\right\} ^{2}}{\mH_{H}^{(\tau)}}}{\mH^{(\tau)}}\\
 & =\CE{\left\{ \sum_{k=1}^{H-1}\CE{\Value h{\star}{\State{k+1}{\tau}}-\Value{h-1}{\star}{\State k{\tau}}}{\mH_{k}^{(\tau)}}+\Value h{\star}{\State{H+1}{\tau}}-\Value{h-1}{\star}{\State H{\tau}}\right\} ^{2}}{\mH^{(\tau)}},
\end{align*}
where the second inequality holds by Jensen's inequality. Applying
the inequality recursively,
\[
\CE{\sum_{k=1}^{H}\CE{\left\{ \Value h{\star}{\State{k+1}{\tau}}-\Value{h-1}{\star}{\State k{\tau}}\right\} ^{2}}{\mH_{k}^{(\tau)}}}{\mH^{(\tau)}}\le\CE{\left\{ \sum_{k=1}^{H}\Value h{\star}{\State{k+1}{\tau}}-\Value{h-1}{\star}{\State k{\tau}}\right\} ^{2}}{\mH^{(\tau)}}.
\]
There we obtain
\begin{align*}
\left\{ \sum_{k=1}^{H}\Value h{\star}{\State{k+1}{\tau}}-\Value{h-1}{\star}{\State k{\tau}}\right\} ^{2}= & \left\{ \sum_{k=1}^{H}\Value h{\star}{\State{k+1}{\tau}}-\Value h{\star}{\State k{\tau}}+\Value h{\star}{\State k{\tau}}-\Value{h-1}{\star}{\State k{\tau}}\right\} ^{2}\\
= & \left\{ \Value h{\star}{\State{H+1}{\tau}}-\Value h{\star}{\State 1{\tau}}+\sum_{k=1}^{H}\Value h{\star}{\State k{\tau}}-\Value{h-1}{\star}{\State k{\tau}}\right\} ^{2}.
\end{align*}
Note that $\abs{\Value h{\star}{\State{H+1}{\tau}}-\Value h{\star}{\State 1{\tau}}}\le H$.
Because $\Value h{\star}x:=\sup_{\pi}\Value h{\star}x$ for all $x\in\mX$
and $h\in[H]$, 
\begin{align*}
\abs{\sum_{k=1}^{H}\Value h{\star}{\State k{\tau}}-\Value{h-1}{\star}{\State k{\tau}}}\le & \sum_{k=1}^{H}\abs{\Value h{\star}{\State k{\tau}}-\Value{h-1}{\star}{\State k{\tau}}}\\
\le & \sum_{k=1}^{H}\sup_{x\in\mX}\abs{\Value h{\star}x-\Value{h-1}{\star}x}\\
= & \sum_{k=1}^{H}\sup_{x\in\mX}\abs{\sup_{\pi}\Value h{\pi}x-\sup_{\pi}\Value{h-1}{\pi}x}\\
\le & \sum_{k=1}^{H}\sup_{x\in\mX}\sup_{\pi}\abs{\Value h{\pi}x-\Value{h-1}{\pi}x}\\
= & \sum_{k=1}^{H}\sup_{x\in\mX}\max_{a\in\mA}\abs{\Reward xa}\\
\le & H.
\end{align*}
Thus, we obtain
\[
\left\{ \sum_{k=1}^{H}\Value h{\star}{\State{k+1}{\tau}}-\Value{h-1}{\star}{\State k{\tau}}\right\} ^{2}\le4H^{2}.
\]
Gathering the inequalities proves the variance bound. 
\end{proof}

\subsection{Proof of Lemma~\ref{lem:regret_decomposition}}
\begin{proof}
By definition of the regret,
\begin{align*}
\Expectation\left[\text{R}(N,\widehat{A})\right]:= & \Expectation\left[\sum_{\tau=1}^{N}\Value 1{\star}{\State 1{\tau}}-V_{1}^{\widehat{\pi}^{(\tau)}}(\State 1{\tau})\right]
\end{align*}
For any $\tau\in[N]$ and $h\in[H]$, define $\mS_{h}^{(\tau)}:=\{\Action h{\tau}=\arg\max_{a\in\mA}\Reward{\State h{\tau}}a+\Feature{\State h{\tau}}a^{\top}\Estimator{h+1}{\tau-1})\}$.
By construction of the algorithm we have $\Probability(\mS_{h}^{(\tau)})=(1-\tau^{-1/2})^{1/H}$.For
$\tau\in[N]$ and $h\in[H]$, define $\widehat{a}_{H}^{(\tau)}(x):=\arg\max_{a\in\mA}\Reward xa+\Feature xa^{\top}\Estimator{h+1}{\tau-1})$.
Because $\Estimator{H+1}{\tau-1}=\mathbf{0}$, for any $x\in\mX$,
\begin{align*}
V_{H}^{\widehat{\pi}^{(\tau)}}(x) & =\Expectation^{\widehat{\pi}}\left[\Reward x{a_{H}}\right]\\
 & \ge\left(1-\tau^{-1/2}\right)^{1/H}\Reward x{\widehat{a}_{H}^{(\tau)}}\\
 & =\left(1-\tau^{-1/2}\right)^{1/H}\Pi_{[0,H]}(\Reward x{\widehat{a}_{H}^{(\tau)}}+\Feature x{\widehat{a}_{H}^{(\tau)}}^{\top}\Estimator{H+1}{\tau-1})\\
 & =\left(1-\tau^{-1/2}\right)^{1/H}\max_{a\in\mA}\Pi_{[0,H]}(\Reward xa+\Feature xa^{\top}\Estimator{H+1}{\tau-1})\\
 & =\left(1-\tau^{-1/2}\right)^{1/H}\max_{a\in\mA}\Pi_{[0,H]}\left(\widehat{Q}_{\Estimator{H+1}{\tau-1}}(x,a)\right),
\end{align*}
where the first inequality holds because the reward function is nonnegative.
For step $H-1$, 
\begin{align*}
V_{H-1}^{\widehat{\pi}^{(\tau)}}(x)= & \Expectation\left[\AV{H-1}{\widehat{\pi}}x{a_{H-1}}\right]\\
\ge & \left(1-\tau^{-1/2}\right)^{1/H}\Pi_{[0,H]}\left(\AV{H-1}{\widehat{\pi}}x{\widehat{a}_{H-1}^{(\tau)}(x)}\right)\\
= & \left(1-\tau^{-1/2}\right)^{1/H}\Pi_{[0,H]}\left(\Reward x{\widehat{a}_{H-1}^{(\tau)}(x)}+[\Probability V_{H}^{\widehat{\pi}}](x,\widehat{a}_{H-1}^{(\tau)}(x))\right)\\
\ge & \left(1-\tau^{-1/2}\right)^{2/H}\Pi_{[0,H]}\left(\Reward x{\widehat{a}_{H-1}^{(\tau)}(x)}+\int_{\mX}\max_{a\in\mA}\Pi_{[0,H]}\left(\widehat{Q}_{\Estimator{H+1}{\tau-1}}(x,a)\right)\Feature x{\widehat{a}_{H-1}^{(\tau)}(x)}^{\top}\psi(x^{\prime})dx^{\prime}\right)
\end{align*}
By definition of $\Barw H{\tau-1}$, 
\begin{align*}
V_{H-1}^{\widehat{\pi}^{(\tau)}}(x)\ge & \left(1-\tau^{-1/2}\right)^{2/H}\Pi_{[0,H]}\left(\Reward x{\widehat{a}_{H-1}^{(\tau)}(x)}+\Feature x{\widehat{a}_{H-1}^{(\tau)}(x)}^{\top}\Barw H{\tau-1}\right)\\
\ge & \left(1-\tau^{-1/2}\right)^{2/H}\Pi_{[0,H]}\left(\Reward x{\widehat{a}_{H-1}^{(\tau)}(x)}+\Feature x{\widehat{a}_{H-1}^{(\tau)}(x)}^{\top}\Estimator H{\tau-1}\right)-\norm{\Barw H{\tau-1}-\Estimator H{\tau-1}}_{1}\\
= & \left(1-\tau^{-1/2}\right)^{2/H}\max_{a\in\mA}\Pi_{[0,H]}\left(\widehat{Q}_{\Estimator H{\tau-1}}(x,a)\right)-\norm{\Barw H{\tau-1}-\Estimator H{\tau-1}}_{1}.
\end{align*}
Recursively, for step $1$, 
\[
V_{1}^{\widehat{\pi}^{(\tau)}}(x)\ge\left(1-\tau^{-1/2}\right)\max_{a\in\mA}\Pi_{[0,H]}\left(\widehat{Q}_{\Estimator 2{\tau-1}}(x,a)\right)-\sum_{h=2}^{H}\norm{\Estimator 2{\tau-1}-\Barw 2{\tau-1}}_{1}.
\]
For the optimal value function, 
\begin{align*}
V_{1}^{\star}(x)= & \max_{a\in\mA}\AV 1{\star}xa\\
= & \max_{a\in\mA}\left\{ \Reward xa+[\Probability V_{2}^{\star}](x,a)\right\} \\
= & \max_{a\in\mA}\Pi_{[0,H]}\left(\Reward xa+\Feature xa^{\top}w_{2}^{\star}\right)\\
\le & \max_{a\in\mA}\Pi_{[0,H]}\left\{ \Reward xa+\Feature xa^{\top}\Estimator 2{\tau-1}\right\} +\max_{(x,a)\in\mX\times\mA}\abs{\Feature xa^{\top}\left(w_{2}^{\star}-\Estimator 2{\tau-1}\right)}\\
= & \max_{a\in\mA}\Pi_{[0,H]}\left(\widehat{Q}_{\Estimator 2{\tau-1}}(x,a)\right)+\max_{(x,a)\in\mX\times\mA}\abs{\Feature xa^{\top}\left(w_{2}^{\star}-\Estimator 2{\tau-1}\right)}.
\end{align*}
By definition of $\Barw 2{\tau-1}$, 
\begin{align*}
 & \max_{(x,a)\in\mX\times\mA}\abs{\Feature xa^{\top}\left(w_{2}^{\star}-\Estimator 2{\tau-1}\right)}\\
 & \le\max_{(x,a)\in\mX\times\mA}\abs{\Feature xa^{\top}\left(w_{2}^{\star}-\Barw 2{\tau-1}\right)}+\norm{\Estimator 2{\tau-1}-\Barw 2{\tau-1}}_{1}\\
 & =\max_{(x,a)\in\mX\times\mA}\abs{\Feature xa^{\top}\left\{ \int_{\mX}\left\{ \max_{a\in\mA}\AV 2{\star}{x^{\prime}}a-\max_{a\in\mA}\Pi_{[0,H]}\left(\widehat{Q}_{\Estimator 3{\tau-1}}(x^{\prime},a)\right)\right\} \psi(x^{\prime})dx^{\prime}\right\} }+\norm{\Estimator 2{\tau-1}-\Barw 2{\tau-1}}_{1}
\end{align*}
Because $\int\Feature xa^{\top}\psi(x^{\prime})dx^{\prime}=\int\Probability(x^{\prime}|x,a)dx^{\prime}=1$
for all $(x,a)\in\mX\times\mA$, 
\begin{align*}
 & \max_{(x,a)\in\mX\times\mA}\abs{\Feature xa^{\top}\left(w_{2}^{\star}-\Estimator 2{\tau-1}\right)}\\
 & \le\max_{x\in\mX}\abs{\max_{a\in\mA}\AV 2{\star}{x^{\prime}}a-\max_{a\in\mA}\Pi_{[0,H]}\left(\widehat{Q}_{\Estimator 3{\tau-1}}(x^{\prime},a)\right)}+\norm{\Estimator 2{\tau-1}-\Barw 2{\tau-1}}_{1}\\
 & =\max_{x\in\mX}\abs{\max_{a\in\mA}\Pi_{[0,H]}\left(\AV 2{\star}xa\right)-\max_{a\in\mA}\Pi_{[0,H]}\left(\widehat{Q}_{\Estimator 3{\tau-1}}(x,a)\right)}+\norm{\Estimator 2{\tau-1}-\Barw 2{\tau-1}}_{1}\\
 & \le\max_{(x,a)\in\mX\times\mA}\abs{\Feature xa^{\top}\left(w_{3}^{\star}-\Estimator 3{\tau-1}\right)}+\norm{\Estimator 2{\tau-1}-\Barw 2{\tau-1}}_{1}.
\end{align*}
Recursively, we obtain
\begin{align*}
V_{1}^{\star}(x)\le & \max_{a\in\mA}\Pi_{[0,H]}\left(\widehat{Q}_{\Estimator 2{\tau-1}}(x,a)\right)+\max_{(x,a)\in\mX\times\mA}\abs{\Feature xa^{\top}\left(w_{H+1}^{\star}-\Estimator{H+1}{\tau-1}\right)}+\sum_{h=2}^{H}\norm{\Estimator h{\tau-1}-\Barw h{\tau-1}}_{1}\\
= & \max_{a\in\mA}\Pi_{[0,H]}\left(\widehat{Q}_{\Estimator 2{\tau-1}}(x,a)\right)+\sum_{h=2}^{H}\norm{\Estimator h{\tau-1}-\Barw h{\tau-1}}_{1},
\end{align*}
where the inequality holds by $w_{H+1}^{\star}=\Estimator{H+1}{\tau-1}=\mathbf{0}$.
Therefore
\begin{align*}
\Value 1{\star}{\State 1{\tau}}-V_{1}^{\widehat{\pi}^{(\tau)}}(\State 1{\tau})\le & \frac{1}{\sqrt{\tau}}\max_{a\in\mA}\Pi_{[0,H]}\left(\widehat{Q}_{\Estimator 2{\tau-1}}(\State 1{\tau},a)\right)+2\sum_{h=2}^{H}\norm{\Estimator h{\tau-1}-\Barw h{\tau-1}}_{1}\\
\le & \frac{H}{\sqrt{\tau}}+2\sum_{h=2}^{H}\norm{\Estimator h{\tau-1}-\Barw h{\tau-1}}_{1}.
\end{align*}
Summing over $\tau\in[N]$ proves the result.
\end{proof}

\subsection{Proof of Theorem~\ref{thm:upper_bound}}
By Lemma B.4 in~\citet{kim2023pareto},
$N\ge\frac{Cs_{\star}^{4}H^{2}\log^{2}(2d)}{\sigma_{U}^{4}}\log^{5}\frac{2ds_{\star}^{4/5}H^{2/5}}{e\sigma^{4/5}\sqrt{\delta}}$, implies 
\[
N\ge C\sigma_{U}^{-4}s_{\star}^{4}H^{2}\log^{5}(dHN^{2}/\delta)\log^{2}(2d).
\]
By the regret decomposition (Lemma~\ref{lem:regret_decomposition}), with $N_1=C\sigma_{U}^{-4}s_{\star}^{4}H^{2}\log^{5}(dHN^{2}/\delta)\log^{2}(2d)$,
\[
\mathrm{R}(N,\widehat{A}_{\algo}) \le 2H(\sqrt{N}+C\sigma_{U}^{-4}s_{\star}^{4}H^{2}\log^{5}(dHN^{2}/\delta)\log^{2}(2d)) + 2\sum_{n=N_1}^{N-1}\sum_{h=2}^{H}\|\Estimator h{n}-\Barw h{n}\|_{1}.
\]
Applying the tail inequality (Theorem~\ref{thm:est_tail}) for the estimator proves the regret bound.

%\subsection{Proof of Lemma~\ref{lem:RME_inequality}}

\end{document}